\documentclass[11pt]{article}

\usepackage{arxiv}

\usepackage{amsmath,amsfonts,latexsym,fullpage,graphicx,vmargin, mathrsfs}
\usepackage{pstricks}
\usepackage{comment}
\usepackage[inline]{enumitem}
\usepackage{cleveref}

\usepackage{xcolor}
\usepackage{url}

\usepackage{amssymb}
\usepackage{mathtools}
\usepackage{amsmath,amsfonts}

\usepackage{subcaption}
\usepackage{scalerel}
\usepackage{url}
\usepackage{graphicx}
\usepackage{booktabs}
\usepackage{cleveref}
\usepackage{float}
\floatstyle{plaintop}
\restylefloat{table}

\usepackage{multirow}

\DeclareMathOperator{\pers}{Pers}
\DeclareMathOperator{\relu}{ReLU}
\DeclareMathOperator{\supp}{supp}

\DeclareMathOperator{\POT}{POT}
\DeclareMathOperator{\OT}{OT}
\DeclareMathOperator{\Lip}{Lip}

\newcommand{\aug}{\mathrm{aug}}

\newcommand{\PL}{\mathrm{PL}}

\newtheorem{prop}{Proposition}
\newtheorem{defi}{Definition}
\newtheorem{rmk}{Remark}

\newtheorem{lem}{Lemma}
\newtheorem{cor}{Corollary}

\newcommand{\R}{\mathbb{R}}
\newcommand{\N}{\mathbb{N}}
\newcommand{\Ss}{\mathbb{S}}

\newcommand{\virgolette}[1]{``#1''}

\firstpageno{1}

\begin{document}

\title{Persistence Spheres: a Bi-continuous Linear Representation of Measures for Partial Optimal Transport}

\author{Matteo Pegoraro\thanks{Institute of Computing,
       Faculty of Informatics,
       Universitá della Svizzera Italiana}}

%

\maketitle

\begin{abstract}
We improve and extend persistence spheres, introduced in~\cite{pegoraro2025persistence}.
Persistence spheres map an integrable measure $\mu$ on the upper half-plane, including persistence diagrams (PDs) as counting measures, to a function $S(\mu)\in C(\mathbb{S}^2)$, and the map is stable with respect to 1-Wasserstein partial transport distance $\mathrm{POT}_1$. Moreover, to the best of our knowledge, persistence spheres are the first explicit representation used in topological machine learning for which continuity of the inverse on the image is established at every compactly supported target. Recent bounded-cardinality bi-Lipschitz embedding results in partial transport spaces, despite being powerful, are not given by the kind of explicit summary map considered here. Our construction is rooted in convex geometry: for positive measures, the defining ReLU integral is the support function of the lift zonoid. Building on~\cite{pegoraro2025persistence}, we refine the definition to better match the $\mathrm{POT}_1$ deletion mechanism, encoding partial transport via a signed diagonal augmentation. In particular, for integrable $\mu$, the uniform norm between $S(0)$ and $S(\mu)$ depends only on the persistence of $\mu$, without any need of ad-hoc re-weightings, reflecting optimal transport to the diagonal at persistence cost. This yields a parameter-free representation at the level of measures (up to numerical discretization), while accommodating future extensions where $\mu$ is a smoothed measure derived from PDs (e.g., persistence intensity functions~\citep{wu2024estimation}). Across clustering, regression, and classification tasks involving functional data, time series, graphs, meshes, and point clouds, the updated persistence spheres are competitive and often improve upon persistence images, persistence landscapes, persistence splines, and sliced Wasserstein kernel baselines.
\end{abstract}


\section{Introduction}

Topological Data Analysis (TDA) provides geometric descriptors designed to capture qualitative structure in data while being robust to noise and, in many contexts, insensitive to parametrization. Its flagship tool, persistent homology, tracks homological features across a scale parameter, recording when connected components, loops, and higher-dimensional cavities are created and later filled in. The resulting summaries are typically encoded as persistence diagrams (PDs) or barcodes, which have become standard objects for exploratory analysis and for building topologically informed learning pipelines \citep{edelsbrunner2010computational, oudot2015persistence}.

\paragraph{Distances and the non-linear geometry of diagrams.}
From a statistical viewpoint, PDs are most naturally compared by Wasserstein-type distances defined through partial optimal transport (POT), where unmatched mass is optimally sent to the diagonal \citep{divol2021understanding}. These metrics are central to stability results, but they endow the space of diagrams with a strongly non-linear geometry. As a consequence, even elementary operations, such as averaging, regression, or principal component analysis, do not admit straightforward analogues. For instance, averages can be formulated as Wasserstein barycenters \citep{mileyko2011probability}, which can be costly to approximate and may be non-unique, complicating both computation and interpretation.

\paragraph{Topological machine learning via vectorizations.}
A large body of work therefore focuses on mapping PDs into linear spaces where classical statistical tools apply. This line of research underpins topological machine learning \citep{papamarkou2024position}, where vectorized topological features and differentiable topological objectives are used in predictive tasks and representation learning \citep{moor2020topological, waylandmapping}. Existing approaches include explicit embeddings and feature maps, as well as kernel methods \citep{reininghaus2015stable, kusano2018kernel, carriere2017sliced}. Among explicit embeddings, one finds constructions based on descriptive statistics \citep{asaad2022persistent}, algebraic or tropical coordinates \citep{kalivsnik2019tropical, monod2019tropical, di2015comparing}, and functional representations such as landscapes, images, and related summaries \citep{bubenik2015statistical, adams2017persistence, biscio2019accumulated, dong2024persistence, gotovac2025topological}, alongside other geometric encodings \citep{mitra2024geometric}. To the best of our knowledge, however, none of these explicit vectorizations comes with some continuity statement for the inverse map. By contrast, for kernel-induced Hilbert embeddings, the sliced Wasserstein construction yields such a bi-continuity statement on bounded-cardinality subclasses, as we recall below.

At the same time, strong impossibility results show that one cannot hope for a globally faithful linearization of the Wasserstein geometry of PDs. Following the synthesis in \citet{mitra2024geometric}, the picture can be summarized as follows:
(i) the space of diagrams with finitely many points does not admit a bi-Lipschitz embedding into any Hilbert space \citep{mitra2021space};
(ii) even restricting to at most $n$ points, no bi-Lipschitz embedding into any finite-dimensional Euclidean space exists \citep{carriere2019metric};
(iii) on such bounded-cardinality classes, a bi-continuous embedding into a Hilbert space can be achieved via the sliced Wasserstein construction \citep{carriere2017sliced};
(iv) on at most $n$ points one can also obtain coarse embeddings into Hilbert spaces \citep{mitra2021space, mitra2024geometric};
and (v) more recently, bi-Lipschitz embeddings into a Hilbert space have been obtained for diagrams with at most $n$ points \citep{bate2024bi}.
Notably, among these results, \citet{carriere2017sliced, mitra2024geometric} provide explicit geometric constructions.

More precisely, \citet{carriere2017sliced} prove that the sliced Wasserstein distance itself is bi-Lipschitz equivalent to $\POT_1$ under a uniform bound on the number of points, while for the RKHS distance induced by the sliced Wasserstein kernel the comparison with $\POT_1$ is given through continuous control functions on the same bounded-cardinality classes.

\paragraph{Persistence spheres and bi-continuity \citep{pegoraro2025persistence}.}
In \citet{pegoraro2025persistence} we introduced persistence spheres, a functional representation sending a persistence diagram to a real-valued function on $\Ss^2$, with two complementary guarantees: Lipschitz stability with respect to $\POT_1$ and continuity of the inverse on the image. In this sense, persistence spheres provide a form of geometric faithfulness which, to the best of our knowledge, is not available for other explicit representations of persistence diagrams currently used in topological machine learning. The construction is rooted in the lift-zonoid representation of a measure: for a positive integrable measure $\mu$, the map
\[
v \longmapsto \int \relu(\langle v,(1,p)\rangle)\,d\mu(p)
\]
is the support function of the lift zonoid $Z_\mu$ \citep{koshevoy1998lift,hendrych2022note}.

\paragraph{Comparison with \citet{pegoraro2025persistence}.}
Relative to \citet{pegoraro2025persistence}, we strengthen the framework in four directions:
\begin{enumerate}
\item \textbf{POT coherence without re-weighting.} The refined definition matches the $\POT_1$ geometry of deletions: the uniform norm between $S(0)$ and $S(\mu)$ depends only on the persistence of $\mu$, reflecting that unmatched mass is optimally sent to the diagonal at cost equal to its persistence (see \Cref{rmk:pot_to_zero}). As a consequence, the representation can be defined directly at the level of measures, without introducing reweighting schemes or tuning parameters (beyond numerical choices such as the discretization grid used to evaluate the functions).
\item \textbf{A measure-theoretic formulation.} We work with integrable measures on the upper half-plane endowed with $\POT_1$ \citep{divol2021understanding}, covering classical PDs (as discrete measures) and extending seamlessly to smooth summaries such as persistence intensity functions \citep{wu2024estimation}. This unified setting is geared toward statistical modeling of persistence objects, allowing one to treat discrete diagrams and smoothed/estimated representations within a single geometry.
\item \textbf{A comparative analysis of topological summaries.} Beyond introducing the revised persistence-sphere map, we compare how several standard summaries, including persistence landscapes, persistence images, persistence splines, and sliced Wasserstein constructions, deform the underlying $\POT_1$ geometry. This highlights the specific geometric biases induced by different vectorizations and helps explain part of their empirical behavior in the supervised and unsupervised experiments.
\item \textbf{Improved empirical performance.} Repeating the full experimental suite of \citet{pegoraro2025persistence} with the refined definition yields generally improved results across the same case studies, while avoiding any ad hoc reweighting of persistence pairs.
\end{enumerate}

\paragraph{Outline.}
After this introduction, \Cref{sec:convex} recalls the convex-analytic background on support functions and lift zonoids, and \Cref{sec:POT} develops the measure-theoretic setting, including integrable measures, partial optimal transport, and its comparison with ordinary optimal transport on cross-augmented measures. In \Cref{sec:lift_zonoids} we review the classical lift-zonoid transform and its continuity properties, and in \Cref{sec:pers_spheres} we introduce the new signed lift-zonoid formulation of persistence spheres, prove injectivity, and analyze basic qualitative phenomena. The core bi-continuity theory is established in \Cref{sec:uniform}: we first prove uniform stability with respect to $\POT_1$, then derive inverse continuity, including a local H\"older-type inverse bound on fixed compact sets. We then pass to Hilbert-valued formulations in \Cref{sec:hilbert}, where we obtain $L^2(\Ss^2)$ versions of the main continuity results and compare our regime with related Hilbert embeddings from the literature. The empirical part evaluates the revised construction in unsupervised simulations (\Cref{sec:unsupervised}) and on supervised benchmarks (\Cref{sec:supervised}), before \Cref{sec:discussion} concludes. The appendices collect the Sobolev and ReLU details used in the local inverse estimate (\Cref{app:sobolev_relu}), place signed augmentation among other linear summaries (\Cref{sec:augment_discussion}), and illustrate geometric biases induced by common vectorizations (\Cref{sec:deforming_geometry}).

\section{Convex Sets and Support Functions}
\label{sec:convex}

We briefly review the notation and concepts from convex analysis and geometry that will be used throughout.
Standard references include \cite{rockafellar1997convex, salinetti_convex}.

\begin{defi}\label{def:minkowski}
Given convex sets $A,B\subset \R^3$ and a scalar $\lambda\ge 0$, define
\[
A\oplus B \coloneqq \{a+b:\ a\in A,\ b\in B\},
\qquad
\lambda A \coloneqq \{\lambda a:\ a\in A\}.
\]
\end{defi}

\begin{defi}\label{def:support}
Given a compact convex set $A\subset \R^3$, its support function is
\[
h_A:\R^3\to\R,
\qquad
h_A(v)\coloneqq \max_{a\in A}\langle v,a\rangle.
\]
\end{defi}

The restriction $h_A|_{\Ss^{2}}$ determines $h_A$ (hence $A$), and Minkowski operations are linearized by support
functions:
\[
h_{\lambda_1 A\oplus \lambda_2 B}=\lambda_1 h_A+\lambda_2 h_B
\qquad (\lambda_1,\lambda_2\ge 0).
\]

\begin{defi}\label{def:hausdorff}
Given compact subsets $A,B\subset Z$, with $(Z,d_Z)$ a metric space, their Hausdorff distance is
\[
d_H(A,B)\coloneqq
\max\Big\{\sup_{a\in A} d_Z(a,B),\ \sup_{b\in B} d_Z(b,A)\Big\}.
\]
\end{defi}

\begin{prop}\label{prop:convex}
Let $A,B\subset \R^3$ be nonempty compact convex sets. Then
\[
d_H(A,B)=\|h_A-h_B\|_{L^\infty(\Ss^2)}
=\sup_{v\in \Ss^{2}} |h_A(v)-h_B(v)|,
\]
where $d_H$ here denotes the Hausdorff distance induced by the Euclidean norm on $\R^3$.
In particular, the map
\[
A\longmapsto h_A|_{\Ss^2}
\]
is injective.
\end{prop}

\section{Integrable Measures and Partial Optimal Transport}
\label{sec:POT}

For $r\ge 0$, denote
\[
B_r \coloneqq \{p\in \R^2:\ \|p\|_2\le r\},
\qquad
B_r^c \coloneqq \R^2\setminus B_r.
\]

We will work with integrable measures and uniformly integrable sequences; see, e.g., \cite{hendrych2022note}.

\begin{defi}\label{def:integrable}
Let $Z\subset \R^2$ be a Borel set. A positive Borel measure $\mu$ on $Z$ is called
integrable if it is finite and
\[
\int_Z \|p\|_2\,d\mu(p)<\infty.
\]

If $Z=X$, we denote by $\mathcal M$ the set of integrable positive Borel measures on $X$.

A sequence of integrable measures $\{\mu_n\}_{n\in\N}$ on $Z$ is uniformly integrable if
\[
\lim_{r\to\infty}\ \sup_{n\in\N}\ \int_{Z\cap B_r^c}\|p\|_2\,d\mu_n(p)=0.
\]
\end{defi}

Since $\R^2$ is a locally compact Polish space, every finite Borel measure on $\R^2$ is a Radon measure. In particular,
every integrable measure (Definition~\ref{def:integrable}) is Radon, hence the results of \cite{divol2021understanding}
apply throughout.

To compare measures we use weak and vague convergence; see, e.g., \cite{kallenberg1997foundations}.

\begin{defi}\label{def:weak_vague}
Let $Z\subset\R^2$ be a Borel set, endowed with the relative topology. A sequence of integrable measures $\{\mu_n\}_{n\in\N}$ on $Z$ converges weakly to $\mu$, written $\mu_n\xrightarrow{w}\mu$, if
\[
\int_Z f\,d\mu_n \to \int_Z f\,d\mu
\quad\text{for every continuous bounded } f:Z\to\R.
\]
It converges vaguely to $\mu$, written $\mu_n\xrightarrow{v}\mu$, if
\[
\int_Z f\,d\mu_n \to \int_Z f\,d\mu
\quad\text{for every continuous compactly supported } f:Z\to\R.
\]
\end{defi}

\subsection{Measures and persistence}
We adopt a measure-theoretic perspective on persistence diagrams (PDs), following \cite{divol2021understanding}. Set
\[
X \coloneqq \R^2_{x<y}=\{(x,y)\in\R^2:\ x<y\},
\qquad
\Delta \coloneqq \{(x,y)\in\R^2:\ x=y\},
\qquad
\overline X \coloneqq X\cup\Delta.
\]
For $p=(x,y)\in\R^2$, let $\pi_\Delta(p)$ be its projection onto $\Delta$ in the $\infty$-norm:
\[
\pi_\Delta(p)\in \arg\min_{z\in\Delta}\|p-z\|_{\infty},
\qquad\text{so that}\qquad
\pi_\Delta(x,y)=\Big(d(p),\,d(p)\Big),
\]
with $d(p)\coloneqq \tfrac{x+y}{2}$.
Moreover set
\[
\pers(p)\coloneqq\|p-\pi_\Delta(p)\|_{\infty}=\tfrac{y-x}{2}\qquad (p\in X).
\]

\begin{defi}\label{def:pers_mass}
Let $\mu$ be a measure on $X$ and let $Z\subset X$ be measurable. We define
\[\textstyle
\pers_Z(\mu)\coloneqq \int_Z \pers(p) d\mu(p)= \int_Z \|p-\pi_\Delta(p)\|_{\infty}\, d\mu(p)
=
\frac12\int_Z (y-x)\, d\mu((x,y)).
\]
When $Z=X$, we simply write $\pers(\mu)$.
\end{defi}

Note that, for $p=(x,y)\in X$,
\[
\pers(p)=\frac{y-x}{2}\ \le\ \frac{|y|+|x|}{2}\ =\ \frac{\|p\|_1}{2}\ \le\ \frac{\sqrt2}{2}\,\|p\|_2.
\]
Therefore, for any positive Borel measure $\mu$ on $X$,
\[
\pers(\mu)=\int_X \pers(p)\,d\mu(p)\ \le\ \frac{\sqrt2}{2}\int_X \|p\|_2\,d\mu(p).
\]
So, for every $\mu\in\mathcal{M}$, we have \(\pers(\mu)<\infty\).

\subsection{Optimal Transport and Partial optimal transport.}

If $T,S$ are measurable spaces, $f:T\to S$ is measurable, and $\eta$ is a (Borel) measure on $T$, the pushforward
$f_\#\eta$ is the measure on $S$ defined by
\[
(f_\#\eta)(A)\coloneqq \eta\big(f^{-1}(A)\big)\qquad\text{for every measurable }A\subset S.
\]
For every nonnegative measurable function $g:S\to[0,+\infty]$, one has the identity
\begin{equation}\label{eq:pushforward_identity}
\int_S g(s)\, d(f_\#\eta)(s)=\int_T g\big(f(t)\big)\, d\eta(t).
\end{equation}
For two (positive) measures $\mu,\nu$ on the same measurable space, we write
\[
\mu \le \nu
\quad\Longleftrightarrow\quad
\mu(A)\le \nu(A) \qquad\text{for every measurable }A\subset S.
\]

\smallskip
\noindent\textbf{Partial Optimal Transport.}
Let $\pi_1,\pi_2$ be the canonical projections
\[
\pi_1: X\times X\to X,\quad (u,v)\mapsto u,
\qquad
\pi_2: X\times X\to X,\quad (u,v)\mapsto v.
\]
 A partial transport plan between $\mu,\nu\in\mathcal{M}$
is a Borel measure $\gamma$ on $X\times X$ such that
\[
(\pi_1)_\#\gamma \le \mu,
\qquad
(\pi_2)_\#\gamma \le \nu.
\]
Its cost is
\begin{align*}
\mathcal{C}(\gamma;\mu,\nu)
\coloneqq\;&
\int_{X\times X} \|p-q\|_\infty\, d\gamma(p,q)\\
&\;+\;
\int_{X} \|p-\pi_\Delta(p)\|_\infty\,
d\big(\mu-(\pi_1)_\#\gamma\big)(p)\\
&\;+\;
\int_{X} \|q-\pi_\Delta(q)\|_\infty\,
d\big(\nu-(\pi_2)_\#\gamma\big)(q).
\end{align*}

\begin{defi}\label{def:POT1}
For $\mu,\nu\in\mathcal{M}$ we define the partial optimal transport $1$-Wasserstein distance
\[
\POT_1(\mu,\nu)\coloneqq \inf_{\gamma}\ \mathcal{C}(\gamma;\mu,\nu),
\]
where the infimum ranges over all partial transport plans $\gamma$ between $\mu$ and $\nu$.
\end{defi}

We recall a key convergence criterion for $\POT_1$ from \cite{divol2021understanding}.

\begin{theorem}[Convergence for $\POT_1$]\label{thm:POT_conv}
Let $\mu\in\mathcal{M}$ and let $\{\mu_n\}_{n\in\N}\subset\mathcal{M}$. Then
\[
\POT_1(\mu_n,\mu)\to 0
\quad\Longleftrightarrow\quad
\mu_n\xrightarrow{v}\mu
\ \text{ and }\ 
\pers(\mu_n)\to \pers(\mu).
\]
\end{theorem}

\smallskip
\noindent\textbf{Optimal Transport.} 
Let $\tilde\mu,\tilde\nu$ be positive Borel measures on $\overline X$ such that
$\tilde\mu(\overline X)=\tilde\nu(\overline X)$. Let $\mathrm{pr}_1,\mathrm{pr}_2$ be the canonical projections
\[
\mathrm{pr}_1:\overline X\times\overline X\to\overline X,\quad (u,v)\mapsto u,
\qquad
\mathrm{pr}_2:\overline X\times\overline X\to\overline X,\quad (u,v)\mapsto v.
\]
We denote by $\Pi(\tilde\mu,\tilde\nu)$ the set of couplings of $\tilde\mu$ and $\tilde\nu$, i.e.,
the set of positive Borel measures $\Gamma$ on $\overline X\times\overline X$ such that
\[
(\mathrm{pr}_1)_\#\Gamma=\tilde\mu,
\qquad
(\mathrm{pr}_2)_\#\Gamma=\tilde\nu.
\]
\begin{defi}\label{def:OT1}
The $1$-Wasserstein distance between $\tilde\mu, \tilde\nu$ positive Borel measures on $\overline X$ is
\[
\OT_1(\tilde\mu,\tilde\nu)
\coloneqq
\inf_{\Gamma\in\Pi(\tilde\mu,\tilde\nu)}
\int_{\overline X\times\overline X}\|u-v\|_\infty\, d\Gamma(u,v).
\]
\end{defi}

\smallskip
\noindent\textbf{Partial Optimal Transport vs Optimal Transport.}

A useful way to view $\POT_1$ is relating it to a standard $\OT_1$ problem after enlarging the space by the diagonal.
In $\POT_1$, any portion of mass that is left unmatched is ``killed'' by sending it to the diagonal, paying the persistence cost $\|p-\pi_\Delta(p)\|_\infty=\pers(p)$.
This can be replicated by an ordinary transport plan if we add to each measure enough mass on the diagonal to serve as a sink (or source) for deletions.
Concretely, if $\mu$ must delete some mass while matching to $\nu$, we compensate by adding to $\nu$ the corresponding diagonal projection $(\pi_\Delta)_\#\mu$, so that all mass is transported in the augmented space $\overline X=X\cup\Delta$.
The resulting augmented measures have equal total mass and can be compared by the usual $1$-Wasserstein distance on $\overline X$.

\begin{defi}\label{def:cross_aug}
Let $\mu,\nu\in\mathcal{M}$. The cross-augmented measure of $\mu$ by $\nu$ is the positive Borel measure on $\overline X$
defined by
\begin{equation}\label{eq:cross_aug_def}
\mu\oplus_\Delta \nu \;\coloneqq\; \mu + (\pi_\Delta)_\#\nu .
\end{equation}
\end{defi}

\begin{prop}\label{prop:cross_aug}
If $\mu,\nu\in\mathcal{M}$, then $\mu\oplus_\Delta \nu$ is an integrable measure on $\overline X$.
\end{prop}

\begin{proof}
Using \eqref{eq:pushforward_identity} with $f=\pi_\Delta:X\to\Delta$ and $g(z)=\|z\|_2$, the inequality $\|\pi_\Delta(p)\|_2\le \|p\|_2$ yields
\[
\int_\Delta \|z\|_2\, d(\pi_\Delta)_\#\nu(z)
=\int_X \|\pi_\Delta(p)\|_2\, d\nu(p)
\le \int_X \|p\|_2\, d\nu(p) < \infty.
\]
Therefore,
\[
\int_{\overline X} \|u\|_2\, d(\mu\oplus_\Delta \nu)(u)
=
\int_X \|p\|_2\, d\mu(p)
+\int_\Delta \|z\|_2\, d(\pi_\Delta)_\#\nu(z)
<\infty,
\]
so $\mu\oplus_\Delta \nu$ is integrable on $\overline X$.
\end{proof}

The cross-augmentation construction yields an equivalence (up to universal constants) between $\POT_1$ on $X$ and $\OT_1$ on $\overline X$.
This is crucial because it allows us to treat partial matching and deletions using the classical optimal transport formalism:
in particular we can invoke Kantorovich--Rubinstein duality and the large body of tools developed for $\OT_1$.
Many of the forward estimates in this paper ultimately rely on rewriting differences of Lipschitz sphere functions as integrals against cross-augmented positive measures, so that $\OT_1$ duality becomes available.

\begin{prop}\label{prop:POT_vs_OT_aug}
For all $\mu,\nu\in\mathcal{M}$,
\[
\POT_1(\mu,\nu)\ \le\ \OT_1(\mu\oplus_\Delta \nu,\ \nu\oplus_\Delta \mu)
\ \le\ 2\,\POT_1(\mu,\nu).
\]
\end{prop}

\begin{proof}
\noindent\textbf{First inequality.}
Fix a coupling $\Gamma\in\Pi(\mu\oplus_\Delta \nu,\ \nu\oplus_\Delta \mu)$ and decompose it by restriction to the four regions
$X\times X$, $X\times\Delta$, $\Delta\times X$, and $\Delta\times\Delta$, namely
\[
\Gamma_{AB}\coloneqq \mathbf{1}_{A\times B}\,\Gamma,
\qquad A,B\in\{X,\Delta\},
\qquad\text{so that}\qquad
\Gamma=\Gamma_{XX}+\Gamma_{X\Delta}+\Gamma_{\Delta X}+\Gamma_{\Delta\Delta}.
\]
Set $\gamma\coloneqq \Gamma_{XX}$. Since the first marginal of $\Gamma$ equals $\mu\oplus_\Delta \nu$ and the second equals
$\nu\oplus_\Delta \mu$, we have
\[
(\pi_1)_\#\gamma=(\pi_1)_\#\Gamma_{XX}\le (\mu\oplus_\Delta \nu)|_{X}=\mu,
\qquad
(\pi_2)_\#\gamma=(\pi_2)_\#\Gamma_{XX}\le (\nu\oplus_\Delta \mu)|_{X}=\nu,
\]
so $\gamma$ is a partial transport plan between $\mu$ and $\nu$. Moreover, by comparing the $X$--parts of the first and second marginals,
\[
\mu-(\pi_1)_\#\gamma = (\pi_1)_\#\Gamma_{X\Delta},
\qquad
\nu-(\pi_2)_\#\gamma = (\pi_2)_\#\Gamma_{\Delta X}.
\]
Hence, by definition of pushforward,
\[
\int_X \|p-\pi_\Delta(p)\|_\infty\,
d\bigl(\mu-(\pi_1)_\#\gamma\bigr)(p)
=
\int_{X\times\Delta} \|p-\pi_\Delta(p)\|_\infty\, d\Gamma_{X\Delta}(p,z).
\]
Since for every $(p,z)\in X\times\Delta$ one has
\[
\|p-\pi_\Delta(p)\|_\infty
=
\inf_{w\in\Delta}\|p-w\|_\infty
\le \|p-z\|_\infty,
\]
it follows that
\[
\int_X \|p-\pi_\Delta(p)\|_\infty\,
d\bigl(\mu-(\pi_1)_\#\gamma\bigr)(p)
\le
\int_{X\times\Delta}\|p-z\|_\infty\, d\Gamma_{X\Delta}(p,z).
\]
The same argument gives
\[
\int_X \|q-\pi_\Delta(q)\|_\infty\,
d\bigl(\nu-(\pi_2)_\#\gamma\bigr)(q)
\le
\int_{\Delta\times X}\|z-q\|_\infty\, d\Gamma_{\Delta X}(z,q).
\]
Thus, we obtain
\begin{align*}
\mathcal{C}(\gamma;\mu,\nu)
\le\;&
\int_{X\times X}\|p-q\|_\infty\, d\Gamma_{XX}(p,q)
+\int_{X\times\Delta}\|p-z\|_\infty\, d\Gamma_{X\Delta}(p,z)
+\int_{\Delta\times X}\|z-q\|_\infty\, d\Gamma_{\Delta X}(z,q)\\
\le\;&
\int_{\overline X\times\overline X}\|u-v\|_\infty\, d\Gamma(u,v),
\end{align*}
where we drop the nonnegative contribution of $\Gamma_{\Delta\Delta}$. Taking the infimum over all couplings $\Gamma$ yields
\[
\POT_1(\mu,\nu)\le \OT_1(\mu\oplus_\Delta \nu,\ \nu\oplus_\Delta \mu).
\]

\smallskip
\noindent\textbf{Second inequality.}
Let $\gamma$ be any partial transport plan between $\mu$ and $\nu$, and set
\[
\mu'\coloneqq (\pi_1)_\#\gamma,\quad
\nu'\coloneqq (\pi_2)_\#\gamma,\quad
\xi\coloneqq \mu-\mu',\quad
\eta\coloneqq \nu-\nu'.
\]
Define a coupling $\Gamma\in\Pi(\mu\oplus_\Delta \nu,\ \nu\oplus_\Delta \mu)$ by
\[
\Gamma \coloneqq
\gamma
+(\mathrm{Id},\pi_\Delta)_\#\xi
+(\pi_\Delta,\mathrm{Id})_\#\eta
+(\pi_\Delta\circ\pi_2,\ \pi_\Delta\circ\pi_1)_\#\gamma,
\]
where $(\mathrm{Id},\pi_\Delta):X\to X\times\Delta$, $(\pi_\Delta,\mathrm{Id}):X\to\Delta\times X$ and
$(\pi_\Delta\circ\pi_2,\pi_\Delta\circ\pi_1):X\times X\to\Delta\times\Delta$ are the indicated maps.

We compute the first marginal term by term. Using linearity of pushforward,
\[
(\mathrm{pr}_1)_\#\Gamma
=
(\mathrm{pr}_1)_\#\gamma
+(\mathrm{pr}_1)_\#(\mathrm{Id},\pi_\Delta)_\#\xi
+(\mathrm{pr}_1)_\#(\pi_\Delta,\mathrm{Id})_\#\eta
+(\mathrm{pr}_1)_\#(\pi_\Delta\circ\pi_2,\ \pi_\Delta\circ\pi_1)_\#\gamma.
\]
Now:
\begin{itemize}
\item Since $\mu'=(\pi_1)_\#\gamma$ and $\mathrm{pr}_1=\pi_1$ on $X\times X$, we have
\[
(\mathrm{pr}_1)_\#\gamma = \mu'.
\]
\item For the second term, $\mathrm{pr}_1\circ(\mathrm{Id},\pi_\Delta)=\mathrm{Id}$ on $X$, hence
\[
(\mathrm{pr}_1)_\#(\mathrm{Id},\pi_\Delta)_\#\xi = \xi.
\]
\item For the third term, $\mathrm{pr}_1\circ(\pi_\Delta,\mathrm{Id})=\pi_\Delta$ on $X$, hence
\[
(\mathrm{pr}_1)_\#(\pi_\Delta,\mathrm{Id})_\#\eta = (\pi_\Delta)_\#\eta.
\]
\item For the last term, $\mathrm{pr}_1\circ(\pi_\Delta\circ\pi_2,\pi_\Delta\circ\pi_1)=\pi_\Delta\circ\pi_2$ on $X\times X$,
so
\[
(\mathrm{pr}_1)_\#(\pi_\Delta\circ\pi_2,\ \pi_\Delta\circ\pi_1)_\#\gamma
=
(\pi_\Delta\circ\pi_2)_\#\gamma
=
(\pi_\Delta)_\#(\pi_2)_\#\gamma
=
(\pi_\Delta)_\#\nu'.
\]
\end{itemize}
Collecting terms gives
\[
(\mathrm{pr}_1)_\#\Gamma
=
\mu' + \xi + (\pi_\Delta)_\#\nu' + (\pi_\Delta)_\#\eta
=
(\mu'+\xi) + (\pi_\Delta)_\#(\nu'+\eta)
=
\mu + (\pi_\Delta)_\#\nu
=
\mu\oplus_\Delta\nu,
\]
where we used $\mu=\mu'+\xi$ and $\nu=\nu'+\eta$ by definition of $\xi,\eta$.

Similarly $(\mathrm{pr}_2)_\#\Gamma=\nu\oplus_\Delta \mu$, hence $\Gamma$ is indeed a coupling of the cross-augmented measures.

For the transport cost, the first three terms contribute exactly
\[
\int_{X\times X}\|p-q\|_\infty\, d\gamma(p,q)
\;+\;
\int_{X}\|p-\pi_\Delta(p)\|_\infty\, d\xi(p)
\;+\;
\int_{X}\|q-\pi_\Delta(q)\|_\infty\, d\eta(q).
\]
For the last term, integrated on $\Delta\times \Delta$, we use that $\pi_\Delta$ is $1$-Lipschitz for $\|\cdot\|_\infty$, hence for $(p,q)\in X\times X$,
\[
\|\pi_\Delta(q)-\pi_\Delta(p)\|_\infty \le \|q-p\|_\infty.
\]
Therefore,
\[
\int_{\Delta\times\Delta}\|u-v\|_\infty\, d\big[(\pi_\Delta\circ\pi_2,\ \pi_\Delta\circ\pi_1)_\#\gamma\big](u,v)
\le
\int_{X\times X}\|p-q\|_\infty\, d\gamma(p,q).
\]
Combining the bounds yields
\[
\int_{\overline X\times\overline X}\|u-v\|_\infty\, d\Gamma(u,v)
\le
2\,\mathcal{C}(\gamma;\mu,\nu).
\]
Taking the infimum over $\Gamma$ on the left-hand side and then the infimum over partial plans $\gamma$ on the right-hand side
gives
\[
\OT_1(\mu\oplus_\Delta \nu,\ \nu\oplus_\Delta \mu)\le 2\,\POT_1(\mu,\nu).
\]
\end{proof}

\section{Lift Zonoids}\label{sec:lift_zonoids}

Our starting point is the classical lift-zonoid construction
\citep{koshevoy1998lift,hendrych2022note}. For an integrable positive measure
$\mu$ on $\R^2$, the map
\[
v\longmapsto \int_{\R^2}\relu(\langle v,(1,p)\rangle)\,d\mu(p)
\]
is the support function of a compact convex set $Z_\mu\subset\R^3$, called the
lift zonoid of $\mu$. This correspondence is one of the main geometric
ingredients underlying the original persistence-sphere construction of
\citet{pegoraro2025persistence}. In the present paper we will eventually need a
linear extension of this ReLU integral to deal with diagonal augmentation, but
we begin with the classical positive setting, where the convex-geometric picture
is most transparent.

For $p=(x,y)\in\R^2$ we write $(1,p)\coloneqq (1,x,y)\in\R^3$ and
$\relu(t)\coloneqq \max\{0,t\}$.

\begin{defi}[Lift zonoid]\label{def:lift_zonoid}
Let $\mu$ be an integrable measure on $\R^2$. The lift zonoid of $\mu$ is
the unique compact convex set $Z_\mu\subset\R^3$ whose support function satisfies
\[
h_{Z_\mu}(v)
=
\int_{\R^2}\relu\bigl(\langle v,(1,p)\rangle\bigr)\,d\mu(p),
\qquad v\in\R^3.
\]
\end{defi}

A key example, already exploited in \citet{pegoraro2025persistence}, is obtained
when $\mu$ is a discrete measure. In that case the lift-zonoid construction
becomes completely explicit: each atom contributes a line segment in $\R^3$, and
the lift zonoid is obtained by Minkowski summation of these segments. This gives
a concrete geometric interpretation of the ReLU integral and is useful both for
intuition and for later comparisons with persistence-diagram summaries. See also \Cref{fig:lift_zonoid}.

\begin{example}\label{ex:lift_zonoid_discrete}
Let $\mu=\sum_{i=1}^n c_i\,\delta_{p_i}$ with $p_i\in\R^2$ and $c_i>0$. Then
\[
Z_\mu=\bigoplus_{i=1}^n c_i\,[0,(1,p_i)]\subset\R^3,
\]
and its support function is
\[
h_{Z_\mu}(v)=\sum_{i=1}^n c_i\,\relu\bigl(\langle v,(1,p_i)\rangle\bigr).
\]
\end{example}

The lift-zonoid correspondence is especially powerful: the map
$\mu\mapsto Z_\mu$ is injective and enjoys sharp continuity properties. In
particular, Hausdorff convergence of lift zonoids is equivalent to weak
convergence together with uniform integrability of the underlying measures. This
``bi-continuity'' principle is one of the central mechanisms in
\citet{pegoraro2025persistence}, since it allows one to pass back and forth
between measures, support functions, and convex bodies.

\begin{prop}[Convergence of Lift Zonoids]\label{prop:zonoids_conv}
Let $\mu$ be an integrable measure on $\R^2$ and let $\{\mu_n\}_{n\in\N}$ be a
sequence of integrable measures. Then
\[
d_H(Z_{\mu_n},Z_{\mu})\to 0
\quad\Longleftrightarrow\quad
\mu_n\xrightarrow{w}\mu
\ \text{ and }\ 
\{\mu_n\}_{n\in\N}\ \text{is uniformly integrable}.
\]
\end{prop}

\begin{figure}[t]
    \centering
    \includegraphics[width = 0.6\textwidth]{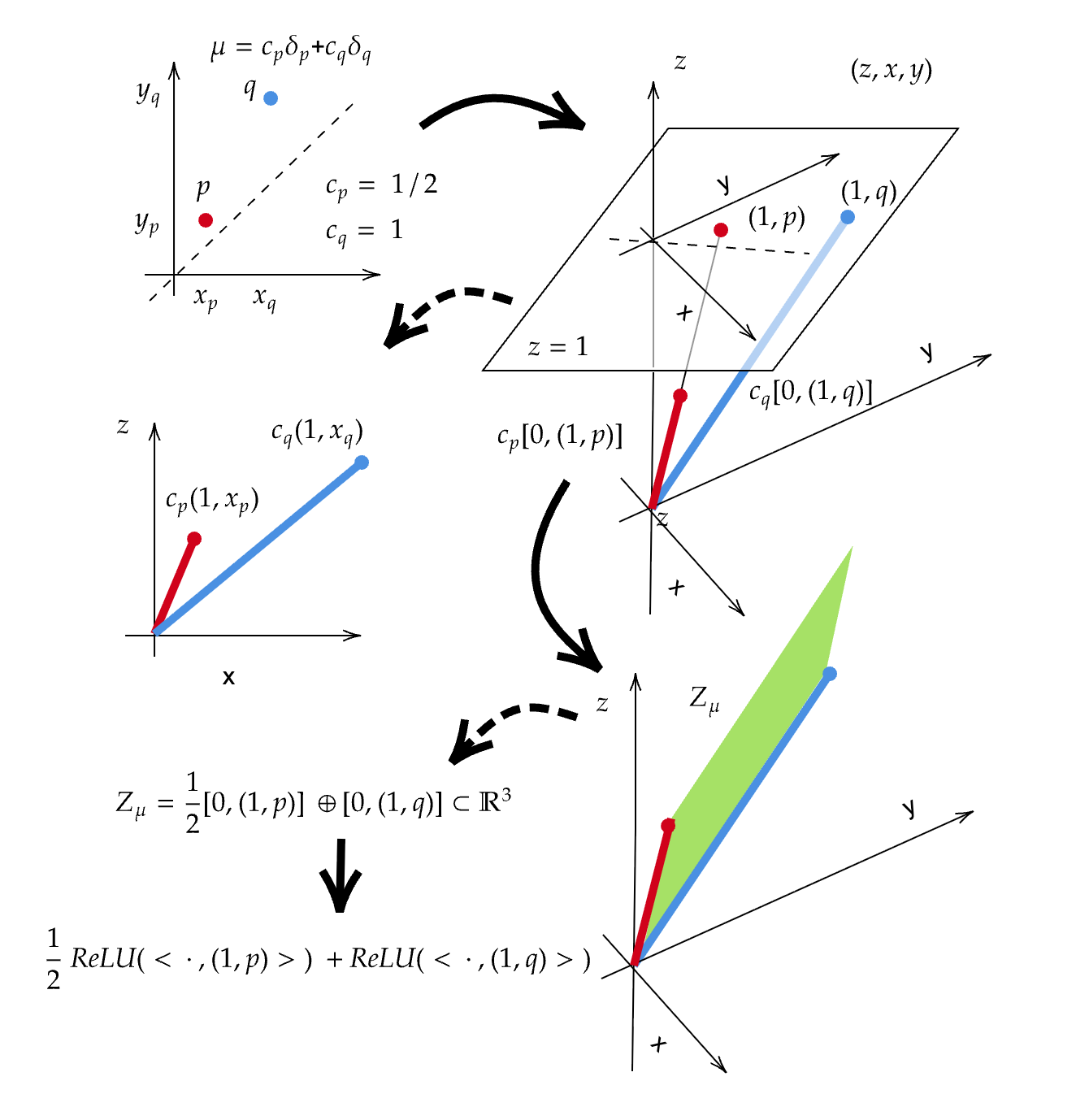}
    \caption{Example of the lift zonoid construction for a discrete measure $\mu=\sum_i c_i\delta_{p_i}$.
    Each atom contributes the segment $c_i[0,(1,p_i)]\subset\R^3$, and $Z_\mu$ is obtained as their Minkowski sum.}
    \label{fig:lift_zonoid}
\end{figure}

The construction developed in the next section is inspired by this classical
picture, but adapted to the partial-transport geometry of persistence diagrams.
There, the natural object is no longer a positive lift zonoid itself, but a
linear ReLU transform applied to the signed augmentation that encodes transport
to the diagonal.

\section{Persistence Spheres via Signed Lift-Zonoid Transforms}\label{sec:pers_spheres}

We now introduce the linear transform underlying the new persistence-sphere
construction. Motivated by \Cref{sec:lift_zonoids}, we keep the same ReLU
integrand as in the lift-zonoid representation, but use it in a purely linear
way so as to accommodate the diagonal augmentation naturally associated with
$\POT_1$.

\begin{defi}[Signed Lift-Zonoid Transform]\label{def:SLZT}
Let $\sigma$ be a finite signed Borel measure on $\R^2$ such that
\[
\int_{\R^2}\|(1,p)\|_2\, d|\sigma|(p)<\infty,
\]
where $|\sigma|$ denotes the total variation measure. The signed
lift-zonoid transform of $\sigma$ is the function
\[
\Lambda(\sigma):\R^3\to\R,
\qquad
\Lambda(\sigma)(v)\coloneqq
\int_{\R^2}\relu\bigl(\langle v,(1,p)\rangle\bigr)\, d\sigma(p).
\]
\end{defi}

\begin{prop}\label{prop:SLZT_linearity}
For $\sigma_1,\sigma_2$ satisfying \Cref{def:SLZT} and
$\lambda_1,\lambda_2\in\R$,
\[
\Lambda(\lambda_1\sigma_1+\lambda_2\sigma_2)
=\lambda_1\,\Lambda(\sigma_1)+\lambda_2\,\Lambda(\sigma_2).
\]
\end{prop}

For positive integrable measures, this reduces exactly to the support-function
representation of \Cref{def:lift_zonoid}. Thus the signed transform should be
viewed as a linear extension of the classical lift-zonoid formula rather than as
a different construction.

We define persistence spheres from measures on the upper half-plane $X$ by
augmenting each measure through its diagonal projection $\pi_\Delta$.

\begin{defi}[Augmented Measure]\label{def:augmented_measure}
Let $\mu$ be a positive Borel measure on $X$. We define its augmented
measure as the signed measure on $\overline X$
\[
\mu^{\aug}\coloneqq \mu-(\pi_\Delta)_\#\mu .
\]
\end{defi}

\begin{prop}\label{prop:SLZT_well_defined}
Let $\mu\in\mathcal{M}$. Then $\mu^{\aug}$ satisfies the integrability condition
of \Cref{def:SLZT}, and therefore $\Lambda(\mu^{\aug})$ is well defined.
\end{prop}

\begin{proof}
Using \eqref{eq:pushforward_identity} with $f=\pi_\Delta:X\to\Delta$ and
$g(z)=\|(1,z)\|_2$, the inequality $\|\pi_\Delta(p)\|_2\le \|p\|_2$ gives
\[
\int_{\Delta}\|(1,z)\|_2\, d(\pi_\Delta)_\#\mu(z)
=
\int_X \|(1,\pi_\Delta(p))\|_2\, d\mu(p)
\le
\int_X \|(1,p)\|_2\, d\mu(p)<\infty.
\]
Hence
\[
\int_{\overline X}\|(1,u)\|_2\, d|\mu^{\aug}|(u)
\le
\int_X \|(1,p)\|_2\, d\mu(p)
+
\int_\Delta \|(1,z)\|_2\, d(\pi_\Delta)_\#\mu(z)
<\infty,
\]
so $\mu^{\aug}$ satisfies \Cref{def:SLZT}.
\end{proof}

\begin{defi}[Persistence sphere]\label{def:PSph}
For $\mu\in\mathcal{M}$, the persistence sphere of $\mu$ is the
restriction to $\Ss^2$ of the signed lift-zonoid transform of its augmentation:
\[
S(\mu):\Ss^2\to\R,
\qquad
S(\mu)\coloneqq \bigl(\Lambda(\mu^{\aug})\bigr)\big|_{\Ss^2}.
\]
\end{defi}

We now relate differences of persistence spheres to cross-augmentation. Given
$\mu,\nu\in\mathcal{M}$, recall $\mu\oplus_\Delta\nu=\mu+(\pi_\Delta)_\#\nu$
from \eqref{eq:cross_aug_def}. A direct computation yields the identity of
signed measures on $\overline X$,
\begin{equation}\label{eq:aug_cross_identity}
\mu^{\aug}-\nu^{\aug}
=
(\mu\oplus_\Delta\nu)-(\nu\oplus_\Delta\mu).
\end{equation}
By linearity of $\Lambda$ (\Cref{prop:SLZT_linearity}), restricting to $\Ss^2$
gives
\begin{equation}\label{eq:S_diff_cross_aug}
S(\mu)-S(\nu)
=
\Bigl(\Lambda(\mu\oplus_\Delta\nu)-\Lambda(\nu\oplus_\Delta\mu)\Bigr)\Big|_{\Ss^2}.
\end{equation}
Consequently, for any norm $\|\cdot\|$ on a function space over $\Ss^2$
(e.g.\ $L^p(\Ss^2)$ or $L^\infty(\Ss^2)$),
\begin{equation}\label{eq:S_norm_cross_aug}
\|S(\mu)-S(\nu)\|
=
\left\|
\Bigl(\Lambda(\mu\oplus_\Delta\nu)-\Lambda(\nu\oplus_\Delta\mu)\Bigr)\Big|_{\Ss^2}
\right\|.
\end{equation}

Note that the measures $\mu\oplus_\Delta\nu$ and $\nu\oplus_\Delta\mu$ are
positive by construction. Moreover, they are integrable on $\overline X$ by
\Cref{prop:cross_aug}. Thus, although persistence spheres are defined through an
augmented signed measure, their differences can be expressed in terms of
positive cross-augmented measures, which will allow us to recover injectivity
and continuity properties by comparison with the classical lift-zonoid setting.

The persistence-sphere construction here differs substantially from the one in
\citet{pegoraro2025persistence}. There, one first reweights the diagram measure
through a persistence-dependent scheme with additional dependence on $\|p\|_2$,
and then applies the lift-zonoid transform, so that the resulting function is
directly the support function of a lift zonoid. In contrast, we avoid
reweighting altogether and encode deletions through diagonal augmentation,
yielding a parameter-free definition at the level of measures (up to numerical
evaluation choices).

As a first step, we obtain injectivity of the novel representation.

\begin{prop}\label{prop:PS_injective}
Let $\mu,\nu\in\mathcal{M}$. If $S(\mu)=S(\nu)$, then $\mu=\nu$.
\end{prop}

\begin{proof}
By \Cref{eq:S_norm_cross_aug}, $S(\mu)=S(\nu)$ implies $\Lambda(\mu\oplus_\Delta\nu)=\Lambda(\nu\oplus_\Delta\mu)$, and injectivity of the lift-zonoid transform on integrable measures
\citep{koshevoy1998lift,hendrych2022note}, in turn, gives  $\mu\oplus_\Delta\nu=\nu\oplus_\Delta\mu$ as positive measures on $\overline X$.
Thus, $\mu-\nu = (\pi_\Delta)_\#(\nu-\mu)$. As the measures on the two sides of the equality have disjoint supports, they must be zero for every measurable set.
\end{proof}

\subsection{A Convenient Change of Coordinates}

With the persistence-sphere map now in place, we next introduce a change of coordinates adapted to the structure of $\POT$. This reformulation is not essential for the definition itself, but it simplifies the expressions appearing below and clarifies the geometric role of persistence.

We want to separate diagonal and off--diagonal variability.
For $v=(v_0,v_1,v_2)\in\Ss^2$ define
\begin{equation}\label{eq:st_def}
s(v):=v_1+v_2,
\qquad
t(v):=v_2-v_1,
\end{equation}
which represent, respectively, the components of $(v_1,v_2)$ along $(1,1)$ (diagonal direction) and $(-1,1)$
(off--diagonal direction). For $p=(x,y)\in X$ recall that
\begin{equation}\label{eq:dpers_def}
d(p)=\frac{x+y}{2},
\qquad
\pers(p)=\frac{y-x}{2},
\end{equation}
so that $d$ is the coordinate along the diagonal and $\pers$ is the distance to the diagonal (persistence).
With this notation, the ReLU arguments appearing in the definition of persistence spheres decompose as
\begin{equation}\label{eq:innerprod_decomp_st}
\langle v,(1,p)\rangle
=
v_0+s(v)\,d(p)+t(v)\,\pers(p),
\qquad
\langle v,(1,\pi_\Delta(p))\rangle
=
v_0+s(v)\,d(p).
\end{equation}
In particular, the pointwise integrand in the definition of $S$ is
\begin{equation}\label{eq:Sv_integrand}
\phi_v(p)\coloneqq
\relu\!\big(v_0+d(p)\,s(v)+\pers(p)\,t(v)\big)
-\relu\!\big(v_0+d(p)\,s(v)\big).
\end{equation}

\Cref{eq:innerprod_decomp_st} will be used repeatedly in the upcoming proofs, since it makes explicit how
directions $v$ probe diagonal location ($s(v)$) and how they probe persistence ($t(v)$).
It also makes transparent an important qualitative feature of persistence spheres: under large one-sided
translations along the diagonal, the representation eventually forgets the exact $d$-location and retains only
the persistence contribution, as the two $d$-contributions delete each other.

Lastly we recall that $\relu$ is $1$-Lipschitz:

\begin{equation}\label{eq:relu_1lip}
|\relu(a)-\relu(b)|\le |a-b|\qquad\forall a,b\in\R.
\end{equation}

\begin{prop}[Diagonal drifts]\label{prop:diag_drift_limits}
Let $\{\mu_n\}\subset\mathcal M$ be a sequence of integrable measures and assume there exists $M<\infty$ such that
\begin{equation}\label{eq:supp_pers_uniform_prop}
\sup_n\ \sup_{p\in\supp(\mu_n)} \pers(p)\le M.
\end{equation}

\smallskip
\noindent\textbf{(A$_+$) Drift to $+\infty$ along the diagonal.}
Assume
\begin{equation}\label{eq:diag_drift_plus_prop}
\forall R>0,\qquad \mu_n\!\big(\{d\le R\}\big)\longrightarrow 0.
\end{equation}
Then for every fixed $v\in\Ss^2$ with $s(v)\neq 0$,
\begin{equation}\label{eq:diag_drift_plus_asympt_prop}
S(\mu_n)(v)\;-\;\mathbf 1_{\{s(v)>0\}}\,t(v)\,\pers(\mu_n)\longrightarrow 0.
\end{equation}
In particular, if $\pers(\mu_n)\to P\in[0,\infty)$, then for every such $v$,
\[
S(\mu_n)(v)\longrightarrow
\begin{cases}
0, & s(v)<0,\\[1mm]
t(v)\,P, & s(v)>0.
\end{cases}
\]

\smallskip
\noindent\textbf{(A$_-$) Drift to $-\infty$ along the diagonal.}
Assume
\begin{equation}\label{eq:diag_drift_minus_prop}
\forall R>0,\qquad \mu_n\!\big(\{d\ge -R\}\big)\longrightarrow 0.
\end{equation}
Then for every fixed $v\in\Ss^2$ with $s(v)\neq 0$,
\begin{equation}\label{eq:diag_drift_minus_asympt_prop}
S(\mu_n)(v)\;-\;\mathbf 1_{\{s(v)<0\}}\,t(v)\,\pers(\mu_n)\longrightarrow 0.
\end{equation}
In particular, if $\pers(\mu_n)\to P\in[0,\infty)$, then for every such $v$,
\[
S(\mu_n)(v)\longrightarrow
\begin{cases}
t(v)\,P, & s(v)<0,\\[1mm]
0, & s(v)>0.
\end{cases}
\]
\end{prop}

\begin{proof}
We prove (A$_+$); the proof of (A$_-$) is analogous.

Fix $v\in\Ss^2$ with $s(v)\neq 0$, and abbreviate
\[
s:=s(v),\qquad t:=t(v).
\]
We choose
\[
R_v:=\frac{|v_0|+|t|M+1}{|s|}
\]
so that, whenever $d(p)\ge R_v$, one has
\[
|s|\,d(p)\ge |v_0|+|t|M+1.
\]

Let now $p\in X$ satisfy $d(p)\ge R_v$ and suppose $s>0$.

Recall that \Cref{eq:innerprod_decomp_st},
\[
\langle v,(1,\pi_\Delta(p))\rangle
=
v_0+s\,d(p),
\qquad
\langle v,(1,p)\rangle
=
v_0+s\,d(p)+t\,\pers(p).
\]
Since $d(p)\ge R_v$, we have
\[
s\,d(p)\ge sR_v=|v_0|+|t|M+1.
\]
Therefore
\[
\langle v,(1,\pi_\Delta(p))\rangle
=
v_0+s\,d(p)
\ge v_0+sR_v
\ge
|t|M+1
>0,
\]
and, using also $\pers(p)\le M$,
\[
\langle v,(1,p)\rangle
=
v_0+s\,d(p)+t\,\pers(p)
\ge
v_0+sR_v-|t|\,\pers(p)
\ge
v_0+sR_v-|t|M
\ge 
1
>0.
\]

Hence both ReLU arguments are positive and therefore
\[
\phi_v(p)
=
\langle v,(1,p)\rangle-\langle v,(1,\pi_\Delta(p))\rangle
=
t\,\pers(p).
\]

If instead $s<0$, then again using \eqref{eq:supp_pers_uniform_prop},
\[
\langle v,(1,\pi_\Delta(p))\rangle
v_0+s,d(p)
\le v_0+sR_v
v_0-|v_0|-|t|M-1
\le -|t|M-1
<0,
\]
and
\[
\langle v,(1,p)\rangle
=
v_0+s\,d(p)+t\,\pers(p)
\le v_0+s R_v + |t|M\le-1 <0.
\]
Hence both ReLU arguments are negative and therefore
\[
\phi_v(p)=0.
\]

Thus, for every $p$ with $d(p)\ge R_v$,
\[
\phi_v(p)=\mathbf 1_{\{s>0\}}\,t\,\pers(p).
\]
It follows that
\[
S(\mu_n)(v)-\mathbf 1_{\{s>0\}}\,t\,\pers(\mu_n)
=
\int_{\{d<R_v\}}\Big(\phi_v(p)-\mathbf 1_{\{s>0\}}\,t\,\pers(p)\Big)\,d\mu_n(p).
\]
Using \eqref{eq:Sv_integrand} and the $1$-Lipschitz property of $\relu$,
\[
|\phi_v(p)|\le |t|\,\pers(p)\le |t|M,
\]
hence
\[
\Big|\phi_v(p)-\mathbf 1_{\{s>0\}}\,t\,\pers(p)\Big|
\le 2|t|M
\qquad\text{for all }p\in X.
\]
Therefore
\[
\big|S(\mu_n)(v)-\mathbf 1_{\{s>0\}}\,t\,\pers(\mu_n)\big|
\le
2|t|M\ \mu_n(\{d<R_v\}).
\]
Since \eqref{eq:diag_drift_plus_prop} holds for every $R>0$, in particular for $R=R_v$, the right-hand side tends to $0$.
This proves \eqref{eq:diag_drift_plus_asympt_prop}. The final convergence statement follows immediately if
$\pers(\mu_n)\to P$.

The proof of (A$_-$) is identical, exchanging the roles of the two tails.
\end{proof}

Moreover, for every $\mu$, the persistence sphere $S(\mu)$ satisfies a sign constraint on the region $\{t(v)<0\}=\{v_1>v_2\}$.

\begin{lem}[A general sign constraint]\label{lem:sign_constraint_t_nonpos}
Let $\mu\in\mathcal M$. Then for every $v\in\Ss^2$ such that $t(v)\le 0$,
\[
S(\mu)(v)\le 0.
\]
\end{lem}

\begin{proof}
By \Cref{eq:Sv_integrand}, if $t(v)\le 0$ then for every $p\in X$ the first ReLU argument is bounded above by
the second one. Hence
\[
\phi_v(p)\le 0\qquad\forall p\in X,
\]
and therefore
\[
S(\mu)(v)=\int_X \phi_v(p)\,d\mu(p)\le 0.
\]
\end{proof}

\Cref{prop:diag_drift_limits} makes precise a flattening phenomenon for persistence spheres under one-sided
diagonal drift. Under a uniform bound on persistence, the sign of the ReLU arguments in
\Cref{eq:innerprod_decomp_st} stabilizes for every fixed direction with $s(v)\neq 0$, and the sphere values
eventually depend only on the total persistence. In this regime, differences purely in the diagonal coordinate
$d$ are asymptotically erased: once mass is translated far enough along $(1,1)$ in one direction, the
representation retains only whether the chosen direction detects that tail, and, if so, through the scalar
factor $t(v)\pers(\mu_n)$. The counting-measure case is made explicit in
\Cref{cor:counting_shift_pointwise_limit}.

Note that this is not an isolated pathology of persistence spheres. Any use of a linear representation
together with the linear operations of its target space necessarily alters some aspects of the original
partial-transport geometry. For instance, barycenters in the representation space---that is, arithmetic
averages---generally do not correspond to Wasserstein barycenters of the underlying measures. What is specific to persistence
spheres is the particularly transparent form that this effect takes under diagonal drift, which can be described
explicitly in terms of the coordinates $s(v)$ and $t(v)$. We return to this broader theme in
\Cref{sec:deforming_geometry}, where persistence spheres are compared with other standard summaries from this
viewpoint.

By contrast, \Cref{lem:sign_constraint_t_nonpos} is a general one-sided fact: directions with $t(v)\le 0$
always yield a nonpositive contribution, independently of any asymptotic regime.

When $P>0$, the pointwise limits on the two open hemispheres given by
\Cref{prop:diag_drift_limits} cannot be the restrictions of a continuous function on $\Ss^2$. For drift to
$+\infty$, they equal $0$ on $\{s(v)<0\}$ and $t(v)P$ on $\{s(v)>0\}$; for drift to $-\infty$, the two roles are
reversed. At every point of the great circle $\{s(v)=0\}$ with $t(v)\neq0$, the limits from the two hemispheres
therefore disagree. Consequently, a sequence satisfying either drift regime with $\pers(\mu_n)\to P>0$ cannot
have persistence spheres converging uniformly on $\Ss^2$, since a uniform limit of continuous functions is
continuous. When $P=0$, this discontinuity argument does not apply. This observation helps explain why large
one-sided diagonal drift must be excluded, or quantitatively controlled, in the uniform inverse-continuity
arguments developed later.

\begin{cor}\label{cor:counting_shift_pointwise_limit}
Let $\{p_1,\dots,p_N\}$ be a finite set and $\mu=\sum_{i=1}^N\delta_{p_i}$.
Define $\mu_k:=\sum_{i=1}^N\delta_{p_i+(k,k)}$.
Then for every $v\in\Ss^2$ with $s(v)\neq 0$,
\[
\lim_{k\to\infty} S(\mu_k)(v)
=
\begin{cases}
0, & s(v)<0,\\[1mm]
t(v)\,\pers(\mu), & s(v)>0.
\end{cases}
\]
\end{cor}

The next lemma quantifies, in a particularly simple one-point setting, how quickly the diagonal-coordinate
variability is washed out by large translations along $(1,1)$. More precisely, if two one-point diagrams
differ only by a fixed offset in the $d$-direction and are both shifted by $(k,k)$, then the corresponding
persistence spheres become eventually identical in each fixed direction $v$, and their uniform discrepancy
decays at worst like $O(1/k)$. In this sense, the lemma provides a quantitative companion to the pointwise
asymptotics of \Cref{prop:diag_drift_limits,cor:counting_shift_pointwise_limit}.

\begin{lem}\label{lem:uniform_decay}
Fix $p\in X$ and set $d:=d(p)$ and $P:=\pers(p)>0$. Let $h>0$ be fixed and, for $k\ge 1$, define
\[
\mu_k:=\delta_{p+(k,k)},
\qquad
\nu_k:=\delta_{p+(k+h,k+h)}.
\]
Then for every $k\ge 2(|d|+h)$,
\begin{equation}\label{eq:large_k_decay}
\big\|S(\mu_k)-S(\nu_k)\big\|_\infty \ \le\ \frac{2h\,\big(1+\sqrt2\,P\big)}{k}.
\end{equation}
\end{lem}

\begin{proof}
Fix $v\in\Ss^2$ and abbreviate $s:=s(v)$ and $t:=t(v)$. Set
\[
K:=t\,P,\qquad a_k:=v_0+s(d+k),\qquad F_K(a):=\relu(a+K)-\relu(a),
\]
so that
\[
S(\mu_k)(v)=F_K(a_k),\qquad S(\nu_k)(v)=F_K(a_k+hs),\qquad \Gamma_k(v)=F_K(a_k)-F_K(a_k+hs).
\]
Since $F_K$ is $1$--Lipschitz,
\begin{equation}\label{eq:Delta_lip_two}
|\Gamma_k(v)|\le h|s|.
\end{equation}

For \eqref{eq:large_k_decay}, note that $F_K$ has breakpoints at $a=0$ and $a=-K$, hence it is constant (with value either $0$ or $K$) outside the
interval
\[
J_K:=\big[\min\{-K,0\},\,\max\{-K,0\}\big]\subset[-|K|,|K|].
\]
Therefore, if $\Gamma_k(v)\neq 0$, the segment $[a_k,a_k+hs]$ intersects $J_K$, so there exists
$\lambda\in[0,1]$ such that
\[
z:=a_k+\lambda hs\in J_K\subset[-|K|,|K|].
\]
Hence
\[
|a_k|
\le |a_k-z|+|z|
\le \lambda h|s|+|K|
\le h|s|+|K|,
\]
and so
\begin{equation}\label{eq:ak_near_kink_two}
|a_k|\le |K|+h|s|.
\end{equation}
On the other hand,
\[
|a_k|=|v_0+s(d+k)|\ge |s|\,|d+k|-|v_0|\ge |s|\,(k-|d|)-1.
\]
Combining with \eqref{eq:ak_near_kink_two} gives, whenever $\Gamma_k(v)\neq 0$,
\[
|s|\,(k-|d|)-1\le |K|+h|s|
\qquad\Longrightarrow\qquad
(k-|d|-h)|s|\le 1+|K|.
\]
If $k\ge 2(|d|+h)$, then $k-|d|-h\ge k/2$, so
\[
|s|\le \frac{2(1+|K|)}{k}.
\]
Plugging into \eqref{eq:Delta_lip_two} yields
\[
|\Gamma_k(v)|\le h|s|\le \frac{2h(1+|K|)}{k}.
\]
Finally, $|K|=|t|P\le \sqrt2\,P$ since $|t(v)|\le \|(v_1,v_2)\|_1\le \sqrt2$, so
\[
|\Gamma_k(v)|\le \frac{2h(1+\sqrt2 P)}{k}.
\]
\end{proof}

\section{Uniform Convergence Results}
\label{sec:uniform}

As suggested by \Cref{prop:convex}, the most natural metric to employ to compare persistence spheres is the sup norm. Thus we start obtaining bi-continuity results relying on such norm, even if it does not meet our final goal of embedding measures in a Hilbert space. 

\subsection{Stability}

In this section we establish the forward continuity direction (stability) of persistence spheres: closeness in \(\POT_1\) implies uniform closeness of the associated sphere functions. The key point is that $S(\mu)$ is built from a ReLU transform, so the difference $S(\mu)-S(\nu)$ can be rewritten as the action of a fixed Lipschitz test family on two cross-augmented measures. This brings the problem into the scope of Kantorovich–Rubinstein duality for \(\OT_1\), and yields a global Lipschitz bound \(\|S(\mu)-S(\nu)\|_\infty \lesssim \POT_1(\mu,\nu)\) (\Cref{prop:S_Linfty_POT}). We also emphasize that this choice of metric is essentially forced: as shown in \citet{skraba2020wasserstein} in the broader setting of Wasserstein-stable linear representations of persistence diagrams, one cannot expect stability for linear operators on diagram measures with respect to other Wasserstein distances.

\begin{theorem}[$\POT_1$ Convergence $\Rightarrow$ Uniform Spheres]\label{prop:S_Linfty_POT}
For all $\mu,\nu\in\mathcal{M}$,
\[
\|S(\mu)-S(\nu)\|_\infty \;\le\; 2 \sqrt{2}\,\POT_1(\mu,\nu).
\]
\end{theorem}

\begin{proof}
Recall the cross-augmentation identity \eqref{eq:aug_cross_identity}:
\[
\mu^{\aug}-\nu^{\aug}=(\mu\oplus_\Delta\nu)-(\nu\oplus_\Delta\mu),
\]
hence by linearity of $\Lambda$ (Proposition~\ref{prop:SLZT_linearity}),
\begin{equation}\label{eq:Sdiff_Lambda_cross}
S(\mu)-S(\nu)
=
\Big(\Lambda(\mu\oplus_\Delta\nu)-\Lambda(\nu\oplus_\Delta\mu)\Big)\Big|_{\Ss^2}.
\end{equation}

\smallskip
\noindent\textbf{Step 1: establish a Lipschitz test family.}
For $v\in \Ss^2$, define $\psi_v:\overline X\to\R$ by
\[
\psi_v(u)\coloneqq \relu\big(\langle v,(1,u)\rangle\big).
\]
Note that
\begin{equation}\label{eq:sphere_Lip_norm}
\|(v_1,v_2)\|_1\le \sqrt{2}\qquad\text{for all }v=(v_0,v_1,v_2)\in \Ss^2.
\end{equation}
Then $\psi_v$ is $\sqrt{2}$-Lipschitz on $(\overline X,\|\cdot\|_\infty)$. Indeed, using
$|\relu(a)-\relu(b)|\le |a-b|$ and \eqref{eq:sphere_Lip_norm}, for $u,w\in\overline X$,
\[
|\psi_v(u)-\psi_v(w)|
\le |\langle v,(1,u)\rangle-\langle v,(1,w)\rangle|
=|\langle (v_1,v_2),u-w\rangle|
\le \|(v_1,v_2)\|_1\,\|u-w\|_\infty
\le \sqrt{2}\|u-w\|_\infty.
\]

\smallskip
\noindent\textbf{Step 2: apply Kantorovich--Rubinstein duality for $\OT_1$.}
Using the dual formulation of $\OT_1$ (Kantorovich--Rubinstein),
\[
\OT_1(\tilde\mu,\tilde\nu)
=
\sup_{\Lip(f)\le 1}\ \int_{\overline X} f\,d(\tilde\mu-\tilde\nu),
\]
where $\Lip(f)\le 1$ is with respect to $\|\cdot\|_\infty$ on $\overline X$.
Apply this with $\tilde\mu=\mu\oplus_\Delta\nu$ and $\tilde\nu=\nu\oplus_\Delta\mu$. Since each $\psi_v$
is $\sqrt{2}$-Lipschitz, we obtain for every $v\in\Ss^2$:
\begin{align*}
\big|\Lambda(\mu\oplus_\Delta\nu)(v)-\Lambda(\nu\oplus_\Delta\mu)(v)\big|
&=
\left|\int_{\overline X} \psi_v(u)\, d\big((\mu\oplus_\Delta\nu)-(\nu\oplus_\Delta\mu)\big)(u)\right|\\
&\le \sqrt{2} \OT_1(\mu\oplus_\Delta\nu,\ \nu\oplus_\Delta\mu).
\end{align*}
Taking the supremum over $v\in\Ss^2$ and using \eqref{eq:Sdiff_Lambda_cross} gives
\[
\|S(\mu)-S(\nu)\|_\infty
\le
\sqrt{2} \OT_1(\mu\oplus_\Delta\nu,\ \nu\oplus_\Delta\mu).
\]
Finally, Proposition~\ref{prop:POT_vs_OT_aug} yields
\[
\OT_1(\mu\oplus_\Delta\nu,\ \nu\oplus_\Delta\mu)\le 2\,\POT_1(\mu,\nu),
\]
which concludes the proof.
\end{proof}

\subsection{Inverse Continuity}

The proof of inverse continuity is more delicate, holds only for convergence to compactly supported measures, and is organized into several steps. We begin by testing persistence spheres in suitable directions, which allows us to detect geometric regimes implied by uniform convergence. We then prove a local inverse estimate for measures supported on a fixed compact set. These ingredients are finally assembled to derive the global inverse-continuity result.

\subsubsection{Test Directions}\label{sec:test}

The inverse-continuity argument relies on a small set of directions $v\in\Ss^2$ for which $\phi_v$ has a simple form, so that $S(\mu)(v)$ directly controls specific portions of $\mu$. 
We start with two basic estimates: one reads off total persistence from a single direction, and the other quantifies how much $S(\eta)$ changes when $\eta$ is truncated to a measurable subset. 
We then construct two families of directions: one detecting mass far out in the $d$-direction (\Cref{lem:directions}), and one detecting both low- and high-persistence contributions (\Cref{lem:directions_lowhighpers}).

\begin{lem}[Total Persistence]\label{lem:pers_eval}
Let $v_{\pers}=(0,-1/\sqrt2,\,1/\sqrt2)\in\Ss^2$. Then for every $\mu\in\mathcal M$,
\[
S(\mu)(v_{\pers})=\sqrt2\,\pers(\mu),
\qquad\text{hence}\qquad
\pers(\mu)=\frac1{\sqrt2}S(\mu)(v_{\pers}).
\]
Consequently, if $\|S(\mu_n)-S(\mu)\|_\infty\to0$, then $\pers(\mu_n)\to\pers(\mu)$ and $\sup_n\pers(\mu_n)<\infty$.
\end{lem}

\begin{proof}
Let $p=(x,y)\in X$. For $v_{\pers}=(0,-1/\sqrt2,1/\sqrt2)$ we have
\[
s(v_{\pers})=0,
\qquad
t(v_{\pers})=\sqrt2.
\]
Therefore, by \Cref{eq:innerprod_decomp_st},
\[
\langle v_{\pers},(1,p)\rangle
= \sqrt2\,\pers(p),
\qquad
\langle v_{\pers},(1,\pi_\Delta(p))\rangle=0.
\]
Since $\pers(p)>0$ on $X$, both arguments of $\relu$ are nonnegative, hence \Cref{eq:Sv_integrand} yields
\[
\phi_{v_{\pers}}(p)=\sqrt2\,\pers(p).
\]
Integrating against $\mu$ gives $S(\mu)(v_{\pers})=\sqrt2\,\pers(\mu)$, and the remaining claims follow immediately.
\end{proof}

\begin{lem}[Truncation]\label{lem:sphere_trunc}
For any $\eta\in\mathcal M$ and any measurable $A\subset X$,
\[
\|S(\eta)-S(\eta|_A)\|_\infty \ \le\ \sqrt2\int_{A^c}\pers(p)\,d\eta(p).
\]
\end{lem}

\begin{proof}
Fix $v\in\Ss^2$ and $p\in X$. By \Cref{eq:Sv_integrand} and the $1$-Lipschitz property \Cref{eq:relu_1lip},
\[
|\phi_v(p)|
\le \big|\pers(p)\,t(v)\big|
\le |t(v)|\,\pers(p)
\le \|(v_1,v_2)\|_1\,\pers(p)
\le \sqrt2\,\pers(p),
\]
where we used $|t(v)|=|v_2-v_1|\le |v_1|+|v_2|=\|(v_1,v_2)\|_1$ and \Cref{eq:sphere_Lip_norm}.
Therefore,
\[
|S(\eta)(v)-S(\eta|_A)(v)|
=
\Big|\int_{A^c}\phi_v(p)\,d\eta(p)\Big|
\le
\sqrt2\int_{A^c}\pers(p)\,d\eta(p),
\]
and taking the supremum over $v\in\Ss^2$ gives the claim.
\end{proof}

\begin{rmk}\label{rmk:pot_to_zero}
A particularly simple instance of partial transport is transport to the null
measure. Since the only partial plan from $\eta|_A$ to $0$ is the zero plan,
all mass is left unmatched and therefore sent to the diagonal. Hence, for every
$\eta\in\mathcal M$ and every measurable $A\subset X$,
\[
\POT_1(\eta|_A,0)=\int_A \pers(p)\,d\eta(p).
\]

In the same spirit, combining \Cref{lem:pers_eval} with
\Cref{lem:sphere_trunc} yields an exact formula for the distance from the null
persistence sphere:
\[
\|S(\mu)-S(0)\|_\infty=\sqrt2\,\pers(\mu),
\qquad \mu\in\mathcal M.
\]
Indeed, \Cref{lem:pers_eval} gives the lower bound
\[
\|S(\mu)\|_\infty \ge S(\mu)(v_{\pers})=\sqrt2\,\pers(\mu),
\]
while \Cref{lem:sphere_trunc}, applied with $A=\varnothing$, gives the matching
upper bound
\[
\|S(\mu)-S(0)\|_\infty\le \sqrt2\,\pers(\mu).
\]
Thus the uniform distance from $S(0)$ is determined exactly by the total
persistence of $\mu$.
\end{rmk}

The next lemma constructs, for a compactly supported measure $\mu$, a direction
$v\in\Ss^2$ that kills the contribution of the fixed compact support
$\supp\mu$ and isolates persistence mass escaping far away in the $d$-coordinate.
Indeed, using \Cref{eq:innerprod_decomp_st}, we choose $v_0<0$ and small
positive coefficients $s(v),t(v)$ so that both ReLU arguments are nonpositive on
$\supp\mu$, hence $S(\mu)(v)=0$. On the other hand, once $d(p)$ is large enough,
the term $v_0+s(v)d(p)$ becomes positive; in that regime both ReLU arguments are
active and their difference is exactly $t(v)\pers(p)$. The direction $v$ thus
suppresses the compact core and detects only persistence mass drifting to
$+\infty$ along the $d$-axis.

\begin{lem}[Far Away]\label{lem:directions}
Let $\mu\in\mathcal M$ have compact support in $X$. Choose $R_0,M_0>0$ such that
\[
\supp\mu \subset \bigl\{p\in X : |d(p)|\le R_0 \text{ and } \pers(p)\le M_0\bigr\}.
\]
Set
\[
s_0:=\frac{1}{2R_0},\qquad t_0:=\frac{1}{2M_0},
\qquad
\tilde v:=\Big(-1,\ \frac{s_0-t_0}{2},\ \frac{s_0+t_0}{2}\Big)\in\R^3,
\qquad
v:=\frac{\tilde v}{\|\tilde v\|_2}\in\Ss^2.
\]
Then:
\begin{enumerate}
\item $s(v)>0$ and $t(v)>0$.
\item For every $p\in\supp\mu$,
\[
\langle v,(1,p)\rangle\le 0
\quad\text{and}\quad
\langle v,(1,\pi_\Delta(p))\rangle<0,
\]
hence $S(\mu)(v)=0$.
\item There exists $R_*(\mu)>0$ such that for every $p\in X$ with $d(p)\ge R_*$,
\[
\phi_v(p)=\relu(\langle v,(1,p)\rangle)-\relu(\langle v,(1,\pi_\Delta(p))\rangle)
=\pers(p)\,t(v).
\]
\end{enumerate}
\end{lem}

\begin{proof}
\noindent\textbf{Proof of (1).} By construction,
\[
s(\tilde v)=(\tilde v)_1+(\tilde v)_2=s_0>0,
\qquad
t(\tilde v)=(\tilde v)_2-(\tilde v)_1=t_0>0.
\]
Normalization by the positive scalar $\|\tilde v\|_2$ preserves signs, hence $s(v)>0$ and $t(v)>0$.

\smallskip
\noindent\textbf{Proof of (2).} Fix $p\in\supp\mu$. Then $|d(p)|\le R_0$ and $\pers(p)\le M_0$.
Using \Cref{eq:innerprod_decomp_st} with $v=\tilde v$ gives
\[
\langle \tilde v,(1,p)\rangle
= -1 + d(p)\,s(\tilde v) + \pers(p)\,t(\tilde v),
\qquad
\langle \tilde v,(1,\pi_\Delta(p))\rangle
= -1 + d(p)\,s(\tilde v).
\]
Since $s(\tilde v)=s_0$ and $t(\tilde v)=t_0$, we have
\[
|d(p)\,s(\tilde v)|\le R_0 s_0=\frac12,
\qquad
|\pers(p)\,t(\tilde v)|\le M_0 t_0=\frac12,
\]
hence
\[
\langle \tilde v,(1,p)\rangle \le -1+\frac12+\frac12=0,
\qquad
\langle \tilde v,(1,\pi_\Delta(p))\rangle \le -1+\frac12=-\frac12<0.
\]
After normalization the inequalities preserve sign, so the same bounds hold for $v$.
In particular, $\relu(\langle v,(1,p)\rangle)=0$ and $\relu(\langle v,(1,\pi_\Delta(p))\rangle)=0$ for all $p\in\supp\mu$,
hence $\phi_v(p)=0$ on $\supp\mu$ by \Cref{eq:Sv_integrand}. Integrating against $\mu$ yields $S(\mu)(v)=0$.

\smallskip
\noindent\textbf{Proof of (3).} By \Cref{eq:innerprod_decomp_st},
\[
\langle v,(1,\pi_\Delta(p))\rangle = v_0 + d(p)\,s(v).
\]
Since $s(v)>0$ by (1), this quantity tends to $+\infty$ as $d(p)\to+\infty$.
Thus there exists $R_*(\mu)>0$ such that for all $p$ with $d(p)\ge R_*(\mu)$,
\[
\langle v,(1,\pi_\Delta(p))\rangle>0.
\]
Moreover,
\[
\langle v,(1,p)\rangle
=
\langle v,(1,\pi_\Delta(p))\rangle+\pers(p)\,t(v)
\ge
\langle v,(1,\pi_\Delta(p))\rangle>0,
\]
since $\pers(p)\ge 0$ and $t(v)>0$. Hence on $\{d\ge R_*(\mu)\}$ both ReLU arguments are positive, and therefore
\[
\phi_v(p)
=
\langle v,(1,p)\rangle-\langle v,(1,\pi_\Delta(p))\rangle
=
\pers(p)\,t(v).
\]
\end{proof}

\begin{rmk}[$d\to -\infty$]\label{rmk:directions_minus}
Define $\mathcal R:\R^3\to\R^3$ by
\[
\mathcal R(v_0,v_1,v_2):=(v_0,-v_2,-v_1).
\]
Then $s(\mathcal R v)=-s(v)$ and $t(\mathcal R v)=t(v)$, and
\[
\langle \mathcal R v,(1,p)\rangle
= v_0+(-d(p))\,s(v)+\pers(p)\,t(v).
\]
In particular, if $v$ is given by \Cref{lem:directions}, then $\mathcal R v$ satisfies the analogue of
\Cref{lem:directions}(2)--(3) on the $d\to-\infty$ side:
$S(\mu)(\mathcal R v)=0$, and there exists $R_*'(\mu)>0$ such that for $d(p)\le -R_*'$,
\[
\phi_{\mathcal R v}(p)=\pers(p)\,t(v).
\]
\end{rmk}

The next lemma introduces a second family of test directions, this time aimed at probing persistence levels independently of the diagonal coordinate. In contrast with the previous construction, we now impose $s(v)=0$, so that the ReLU arguments no longer depend on $d(p)$ and are controlled only by $\pers(p)$. By choosing the remaining parameters appropriately, one obtains directions $v_\delta$ for which the integrand vanishes below a prescribed persistence threshold $\delta$ and grows linearly above it, namely as
\[
t_\delta(\pers-\delta)_+.
\]
In this way, sphere evaluations in these directions isolate the contribution of the high-persistence tail, while suitable combinations also control the mass concentrated at low persistence.

\begin{lem}[Low and High Persistence]\label{lem:directions_lowhighpers}
For $\delta\ge 0$ define
\[
t_\delta:=\sqrt{\frac{2}{1+2\delta^2}},
\qquad
v_\delta:=\Big(-\delta\,t_\delta,\ -\frac{t_\delta}{2},\ \frac{t_\delta}{2}\Big)\in \Ss^2.
\]
Then for every $\eta\in\mathcal M$ and every $\delta\ge 0$,
\begin{equation}\label{eq:phi_hinge_identity}
S(\eta)(v_\delta)=t_\delta\int_X (\pers-\delta)_+\,d\eta,
\end{equation}
\begin{equation}\label{eq:tail_le_hinge_global}
\int_{\{\pers\ge 2\delta\}}\pers\,d\eta
\le
2\int_X(\pers-\delta)_+\,d\eta
=
\frac{2}{t_\delta}\,S(\eta)(v_\delta),
\end{equation}
and
\begin{equation}\label{eq:lowpers_hat_bound}
\int_{\{\pers<\delta\}} \pers\,d\eta
\ \le\
\int_X g_\delta(\pers)\,d\eta
\ =\
\frac{1}{\sqrt2}S(\eta)(v_0)
-2\,\frac{1}{t_\delta}S(\eta)(v_\delta)
+\frac{1}{t_{2\delta}}S(\eta)(v_{2\delta}),
\end{equation}
where
\[
g_\delta(r):= r - 2(r-\delta)_+ + (r-2\delta)_+,\qquad r\ge 0.
\]
\end{lem}

\begin{proof}
By construction
\[
s(v_\delta)=v_{\delta,1}+v_{\delta,2}=0,
\qquad
t(v_\delta)=v_{\delta,2}-v_{\delta,1}=t_\delta>0.
\]
Hence for every $p\in X$,
\[
\langle v_\delta,(1,p)\rangle=v_{\delta,0}+t_\delta\,\pers(p),
\qquad
\langle v_\delta,(1,\pi_\Delta(p))\rangle=v_{\delta,0}.
\]
Since $v_{\delta,0}=-\delta t_\delta\le 0$, we have $\relu(v_{\delta,0})=0$, and therefore \Cref{eq:Sv_integrand} yields
\[
\phi_{v_\delta}(p)
=
\relu\!\big(v_{\delta,0}+t_\delta\,\pers(p)\big)-\relu(v_{\delta,0})
=
\relu\!\big(t_\delta(\pers(p)-\delta)\big)
=
t_\delta(\pers(p)-\delta)_+.
\]
Integrating against $\eta$ gives \Cref{eq:phi_hinge_identity}.
For $\delta=0$ one has $v_0=v_{\pers}$ and $t_0=\sqrt2$, so \Cref{eq:phi_hinge_identity} reduces to
$S(\eta)(v_0)=\sqrt2\int_X \pers\,d\eta$, consistent with \Cref{lem:pers_eval}.

For any $r\ge 0$,
\[
r\,\mathbf 1_{\{r\ge 2\delta\}}\le 2(r-\delta)_+.
\]
Indeed, if $r<2\delta$ the left-hand side is $0$, while if $r\ge 2\delta$ then
$(r-\delta)_+=r-\delta\ge r/2$, i.e.\ $r\le 2(r-\delta)_+$.
Applying this inequality with $r=\pers(p)$ and integrating yields
\[
\int_{\{\pers\ge 2\delta\}}\pers\,d\eta
\le
2\int_X(\pers-\delta)_+\,d\eta,
\]
and the final identity in \Cref{eq:tail_le_hinge_global} follows from \Cref{eq:phi_hinge_identity}.

Define
\[
g_\delta(r):= r - 2(r-\delta)_+ + (r-2\delta)_+,\qquad r\ge 0.
\]
A direct check gives
\[
g_\delta(r)=
\begin{cases}
r, & 0\le r\le \delta,\\
2\delta-r, & \delta\le r\le 2\delta,\\
0, & r\ge 2\delta,
\end{cases}
\qquad\text{hence}\qquad
g_\delta(r)\ge r\,\mathbf 1_{\{r<\delta\}}.
\]
Therefore,
\[
\int_{\{\pers<\delta\}}\pers\,d\eta
\le \int_X g_\delta(\pers)\,d\eta.
\]
Finally, expanding $g_\delta$ and using \Cref{eq:phi_hinge_identity} with parameters $0$, $\delta$, and $2\delta$ yields
\begin{align*}
\int_X g_\delta(\pers)\,d\eta
&=
\int_X \pers\,d\eta
-2\int_X(\pers-\delta)_+\,d\eta
+\int_X(\pers-2\delta)_+\,d\eta\\
&=
\frac{1}{\sqrt2}S(\eta)(v_0)
-2\,\frac{1}{t_\delta}S(\eta)(v_\delta)
+\frac{1}{t_{2\delta}}S(\eta)(v_{2\delta}),
\end{align*}
where we also used \Cref{lem:pers_eval} to rewrite $\int_X\pers\,d\eta=\frac1{\sqrt2}S(\eta)(v_0)$.
This proves \Cref{eq:lowpers_hat_bound}.
\end{proof}

\subsubsection{Vanishing of the Tails}

We now combine the test directions above with the assumption $\|S(\mu_n)-S(\mu)\|_\infty\to0$, where $\mu$ has compact support in $X$. The goal is to show that $\mu_n$ cannot carry a persistent amount of mass either far away in the $d$-direction, or very near the diagonal, or at very large persistence.

\begin{lem}[Far-away persistence vanishes]\label{lem:d_tail}
Assume $\|S(\mu_n)-S(\mu)\|_\infty\to0$, where $\mu$ has compact support in $X$.
Then there exists $R_\mu>0$, depending only on $\mu$, such that for every
$R\ge R_\mu$,
\[
\limsup_{n\to\infty}\int_{\{|d|\ge R\}} \pers\,d\mu_n =0.
\]
In fact, for every such $R$,
\[
\limsup_{n\to\infty}\int_{\{d\ge R\}} \pers\,d\mu_n =0,
\qquad
\limsup_{n\to\infty}\int_{\{d\le -R\}} \pers\,d\mu_n =0.
\]
\end{lem}

\begin{proof}
Choose $R_0,M_0>0$ such that
\[
\supp\mu\subset\{\,|d|\le R_0,\ \pers\le M_0\,\}.
\]
Let $v\in\Ss^2$ be the direction given by \Cref{lem:directions}, and let
$R_+:=R_*(\mu)$ be the threshold from \Cref{lem:directions}(3). Then
$S(\mu)(v)=0$, hence by uniform convergence
\[
S(\mu_n)(v)\longrightarrow 0.
\]
Moreover, for every $p\in X$,
\[
\phi_v(p)
=
\relu\!\bigl(v_0+s(v)d(p)+t(v)\pers(p)\bigr)
-\relu\!\bigl(v_0+s(v)d(p)\bigr)
\ge 0,
\]
since $t(v)>0$ and $\pers(p)\ge 0$.

Fix any $R\ge R_+$. Since \(\{d\ge R\}\subset\{d\ge R_+\}\), \Cref{lem:directions}(3)
gives
\[
\phi_v(p)=t(v)\pers(p)\qquad\text{for all }p\in\{d\ge R\}.
\]
Therefore
\[
S(\mu_n)(v)
=
\int_X \phi_v\,d\mu_n
\ge
\int_{\{d\ge R\}}\phi_v\,d\mu_n
=
t(v)\int_{\{d\ge R\}}\pers\,d\mu_n.
\]
Taking \(\limsup\) and using \(S(\mu_n)(v)\to 0\), we obtain
\[
\limsup_{n\to\infty}\int_{\{d\ge R\}}\pers\,d\mu_n =0
\qquad\forall R\ge R_+.
\]

Now apply the same argument to the reflected direction \(\mathcal R v\) from
\Cref{rmk:directions_minus}. Let \(R_->0\) be the corresponding threshold. Then for every
\(R\ge R_-\),
\[
\limsup_{n\to\infty}\int_{\{d\le -R\}}\pers\,d\mu_n =0.
\]

Finally, set
\[
R_\mu:=\max\{R_+,R_-\}.
\]
Then for every \(R\ge R_\mu\),
\[
\int_{\{|d|\ge R\}}\pers\,d\mu_n
=
\int_{\{d\ge R\}}\pers\,d\mu_n
+
\int_{\{d\le -R\}}\pers\,d\mu_n,
\]
and taking \(\limsup\) yields
\[
\limsup_{n\to\infty}\int_{\{|d|\ge R\}}\pers\,d\mu_n =0.
\]
\end{proof}

\begin{lem}[Low Persistence Vanishes]\label{lem:lowpers_vanish}
Assume $\|S(\mu_n)-S(\mu)\|_\infty\to0$, where $\mu$ has compact support in $X$.
Let $\delta>0$ be such that $\mu(\{\pers\le 2\delta\})=0$. Then
\[
\int_{\{\pers<\delta\}}\pers\,d\mu_n \longrightarrow 0.
\]
\end{lem}

\begin{proof}
By \Cref{lem:directions_lowhighpers}, for every $n$,
\[
0\le \int_{\{\pers<\delta\}}\pers\,d\mu_n
\le
\int_X g_\delta(\pers)\,d\mu_n
=
\frac{1}{\sqrt2}S(\mu_n)(v_0)
-2\,\frac{1}{t_\delta}S(\mu_n)(v_\delta)
+\frac{1}{t_{2\delta}}S(\mu_n)(v_{2\delta}).
\]
Since $\mu(\{\pers\le 2\delta\})=0$, we have $g_\delta(\pers)=0$ $\mu$-a.e., and therefore
\[
\int_X g_\delta(\pers)\,d\mu=0.
\]
Uniform convergence implies pointwise convergence at the fixed directions
$v_0,v_\delta,v_{2\delta}$, so the right-hand side converges to
\[
\frac{1}{\sqrt2}S(\mu)(v_0)
-2\,\frac{1}{t_\delta}S(\mu)(v_\delta)
+\frac{1}{t_{2\delta}}S(\mu)(v_{2\delta})
=
\int_X g_\delta(\pers)\,d\mu
=0.
\]
Hence, $\int_{\{\pers<\delta\}}\pers\,d\mu_n\to0$.
\end{proof}

\begin{lem}[High Persistence Vanishes]\label{lem:highpers_vanish}
Assume $\|S(\mu_n)-S(\mu)\|_\infty\to0$, where $\mu$ has compact support in $X$.
If $\mu=0$, then, for every $M>0$,
\[
\limsup_{n\to\infty}\int_{\{\pers\ge M\}}\pers\,d\mu_n=0.
\]
If $\mu\ne0$, let
\[
M_\mu:=\sup\{\pers(p):p\in\supp\mu\}<\infty.
\]
Then for every $M>2M_\mu$,
\[
\limsup_{n\to\infty}\int_{\{\pers\ge M\}}\pers\,d\mu_n=0.
\]
\end{lem}

\begin{proof}
If $\mu=0$, then \Cref{lem:pers_eval} gives $\pers(\mu_n)\to0$, and
\[
0\le \int_{\{\pers\ge M\}}\pers\,d\mu_n\le\pers(\mu_n)
\]
for every $M>0$. We may therefore assume $\mu\ne0$.

Set $\delta:=M_\mu$ and consider the direction $v_\delta$ from \Cref{lem:directions_lowhighpers}.
Since $\pers\le \delta$ on $\supp\mu$, we have $(\pers-\delta)_+=0$ $\mu$-a.e., hence by \Cref{eq:phi_hinge_identity},
\[
S(\mu)(v_\delta)=0.
\]
Uniform convergence gives $S(\mu_n)(v_\delta)\to 0$, and therefore
\[
\int_X(\pers-\delta)_+\,d\mu_n=\frac1{t_\delta}S(\mu_n)(v_\delta)\longrightarrow 0.
\]
By \Cref{eq:tail_le_hinge_global},
\[
\int_{\{\pers\ge 2\delta\}}\pers\,d\mu_n
\le 2\int_X(\pers-\delta)_+\,d\mu_n \longrightarrow 0.
\]
If $M>2M_\mu=2\delta$, then $\{\pers\ge M\}\subset\{\pers\ge 2\delta\}$, so the claim follows.
\end{proof}

\subsubsection{A Quantitative Inverse Estimate on Compact Sets}\label{sec:compact_rates}

We now restrict to measures supported in a fixed compact set $K\subset X$. In this regime, one can obtain an explicit H\"older-type inverse estimate: the $\POT_1$ distance is controlled by a power of the sphere discrepancy. This is the quantitative compact-core input that will be used in the global inverse-continuity argument below.

To derive such a rate we use recent approximation results for shallow ReLU networks, in particular the Sobolev-to-variation-space embedding proved in \citet{mao2024approximation}. Related approximation results in the uniform norm also appear in \citet{siegel2025optimal}. Indeed, by construction, differences of persistence spheres can be written as differences of integrated ReLU ridge functions. Since $\POT_1$ is controlled by an $\OT_1$ distance between the corresponding cross-augmented measures, Kantorovich--Rubinstein duality reduces the problem to testing the signed measure against $1$-Lipschitz functions. We then approximate arbitrary Lipschitz test functions on a compact set by ReLU ridge combinations, which turns the sphere discrepancy $\|S(\mu)-S(\nu)\|_\infty$ into an explicit upper bound on $\POT_1(\mu,\nu)$.

\begin{theorem}[Local H\"older Bound]\label{prop:compact_core_holder_siegel}
Let $K\subset X$ be compact and set
\[
\delta_K:=\inf\{\|p-z\|_\infty:\ p\in K,\ z\in\Delta\}>0,
\qquad
\overline K:=K\cup \pi_\Delta(K)\subset \overline X.
\]
For $\mu,\nu\in\mathcal M$ supported in $K$, define
\[
M:=\pers(\mu)+\pers(\nu),
\qquad
\varepsilon:=\|S(\mu)-S(\nu)\|_\infty .
\]
Then there exists a constant $C_K<\infty$, depending only on $K$, such that
\[
\POT_1(\mu,\nu)\ \le\ C_K\, M^{3/5}\,\varepsilon^{2/5}.
\]
Moreover, an admissible explicit choice of $C_K$ is constructed in the proof.
\end{theorem}

\begin{proof}
We use the notation $\mathring B_R$, the Sobolev restriction spaces
$W^s(L^2(B_R))$, the integer-order weak-derivative norms
$W^{s,2}(B_R)$, and the interpolation estimate
\eqref{eq:ball_sobolev_interpolation} introduced in
\Cref{app:sobolev_relu}. The Sobolev-to-ReLU estimate used below, including
its constant $A_R$, is stated and proved there in
\Cref{lem:sobolev_relu_control}.

\smallskip
\noindent\textbf{Step 1: reduce to $\OT_1$.}

Let
\[
\eta:=\mu^{\aug}-\nu^{\aug},
\]
which is a finite signed measure supported in $\overline K$.
By \eqref{eq:aug_cross_identity} and Proposition~\ref{prop:POT_vs_OT_aug},
\[
\eta=(\mu\oplus_\Delta \nu)-(\nu\oplus_\Delta \mu),
\qquad
\POT_1(\mu,\nu)\ \le\ \OT_1(\mu\oplus_\Delta \nu,\nu\oplus_\Delta \mu).
\]
The two cross-augmented measures have the same total mass
$\mu(K)+\nu(K)$, and hence
\begin{equation}\label{eq:eta_zero_mass}
\eta(\overline X)=0.
\end{equation}
By Kantorovich--Rubinstein duality for $\OT_1$,
\begin{equation}\label{eq:KR_aug_rewrite}
\OT_1(\mu\oplus_\Delta \nu,\nu\oplus_\Delta \mu)
=
\sup_{\Lip(f)\le 1}\ \int_{\overline X} f\,d\eta ,
\end{equation}
where the Lipschitz constant is computed with respect to $\|\cdot\|_\infty$ on $\overline X$.

\smallskip
\noindent\textbf{Step 2: mollify the Kantorovich--Rubinstein test.}

Fix $f$ with $\Lip(f)\le 1$ (with respect to $\|\cdot\|_\infty$ on $\overline X$).
Since $\eta$ is supported in $\overline K$, only the values of $f$ on $\overline K$ matter.
By \eqref{eq:eta_zero_mass}, replacing $f$ by $f-f(0)$ does not
change $\int f\,d\eta$. We may therefore assume, without loss of generality,
that $f(0)=0$. By the McShane extension theorem
\citep{mcshane1934extension}, $f$ admits an extension from $\overline X$ to a
$1$-Lipschitz function on $\R^2$ (with respect to $\|\cdot\|_\infty$), still
denoted by $f$. The extension still satisfies $f(0)=0$, and consequently
\begin{equation}\label{eq:normalized_lipschitz_growth}
|f(z)|\le\|z\|_\infty
\qquad\text{for every }z\in\R^2.
\end{equation}

Set
\[
R_K:=1+\sup_{u\in \overline K}\|u\|_2.
\]
Then $\overline K\subset \mathring B_{R_K}$, where
$\mathring B_{R_K}:=\{u\in\R^2:\|u\|_2<R_K\}$ is the interior of
$B_{R_K}$, and $\mathring B_{R_K}$ is a bounded open Lipschitz domain.

Let $\rho\in C_c^\infty(\R^2)$ be a standard mollifier, $\rho\ge0$, $\int\rho=1$, and set
\[
\rho_r(x)=r^{-2}\rho(x/r),
\qquad
f_r=\rho_r*f.
\]
Then $f_r\in C^\infty(\R^2)$ and, for every $x\in\overline K$,
\[
|f_r(x)-f(x)|
=
\Big|\int_{\R^2}\rho_r(y)\bigl(f(x-y)-f(x)\bigr)\,dy\Big|
\le
\int_{\R^2}\rho_r(y)\,\|y\|_\infty\,dy
=
r\int_{\R^2}\rho(z)\|z\|_\infty\,dz,
\]
so that
\begin{equation}\label{eq:moll_sup}
\|f-f_r\|_{L^\infty(\overline K)}\le c_{\rm moll}\,r,
\qquad
c_{\rm moll}:=\int_{\R^2}\rho(z)\|z\|_\infty\,dz.
\end{equation}

Since \eqref{eq:moll_sup} already controls the mollification error $f-f_r$
uniformly over the admissible functions $f$, our aim is now to bound
$\|f_r\|_{W^{5/2}(L^2(B_{R_K}))}$ with explicit dependence on $r$. We begin
with the zeroth-order Sobolev term, where the normalization is essential.
After the change of variables $y=rz$,
\eqref{eq:normalized_lipschitz_growth}
gives, for $x\in B_{R_K}$ and $0<r\le1$,
\[
|f_r(x)|
\le
\int_{\R^2}\rho(z)\|x-rz\|_\infty\,dz
\le
R_K+c_{\rm moll}.
\]
Thus
\begin{equation}\label{eq:moll_L2_zero_order}
\|f_r\|_{L^2(B_{R_K})}
\le |B_{R_K}|^{1/2}(R_K+c_{\rm moll}),
\qquad 0<r\le1.
\end{equation}

We next estimate the derivatives of $f_r$.
Since $f$ is $1$-Lipschitz with respect to $\|\cdot\|_\infty$, Rademacher's theorem implies that $f$ is differentiable a.e., and at every differentiability point one has
\[
\|\nabla f(x)\|_1\le 1.
\]
In particular,
\[
\|\partial_i f\|_{L^\infty(\R^2)}\le 1,\qquad i=1,2.
\]
For a multi-index $\alpha\in\N^2$ with $|\alpha|=s\ge1$, choose $i\in\{1,2\}$ with $\alpha_i\ge1$
and set $\beta:=\alpha-e_i$ (so $|\beta|=s-1$).
By commutation of convolution with weak derivatives and integration by parts
\citep{evans2010pde},
\[
D^\alpha f_r = D^\alpha(\rho_r*f)=(D^\alpha\rho_r)*f=(D^\beta\rho_r)*(\partial_i f).
\]
Hence, using the $L^1$--$L^\infty$ Young inequality,
\[
\|D^\alpha f_r\|_{L^\infty(\R^2)}
\le \|D^\beta\rho_r\|_{L^1(\R^2)}\,\|\partial_i f\|_{L^\infty(\R^2)}
\le \|D^\beta\rho_r\|_{L^1(\R^2)},
\]
and therefore
\begin{equation}\label{eq:Dalpha_L2}
\|D^\alpha f_r\|_{L^2(B_{R_K})}
\le |B_{R_K}|^{1/2}\,\|D^\alpha f_r\|_{L^\infty(\R^2)}
\le |B_{R_K}|^{1/2}\,\|D^\beta\rho_r\|_{L^1(\R^2)}.
\end{equation}
By scaling, for every $|\beta|=s-1$,
\[
\|D^\beta\rho_r\|_{L^1(\R^2)}=r^{-(s-1)}\|D^\beta\rho\|_{L^1(\R^2)}.
\]
Combining with \eqref{eq:Dalpha_L2} yields, for each $|\alpha|=s$,
\[
\|D^\alpha f_r\|_{L^2(B_{R_K})} \le C_{K,\rho,s}\, r^{-(s-1)},
\]
with $C_{K,\rho,s}:=|B_{R_K}|^{1/2}\max_{|\beta|=s-1}\|D^\beta\rho\|_{L^1}$.
Summing over the finitely many $|\alpha|=s$ gives
\begin{equation}\label{eq:Ds_bound}
\|D^s f_r\|_{L^2(B_{R_K})} \le C'_{K,\rho,s}\, r^{-(s-1)},
\qquad
\text{where }\ \|D^s h\|_{L^2(B_{R_K})}^2:=\sum_{|\alpha|=s}\|D^\alpha h\|_{L^2(B_{R_K})}^2.
\end{equation}
Using the weak-derivative norm \eqref{eq:integer_sobolev_norm}, and combining
\eqref{eq:moll_L2_zero_order} with the derivative bounds above, yields constants
\(C^{(2)}_{K,\rho},C^{(3)}_{K,\rho}<\infty\) such that
\[
\|f_r\|_{W^{2,2}(B_{R_K})}\le C^{(2)}_{K,\rho}\,r^{-1},
\qquad
\|f_r\|_{W^{3,2}(B_{R_K})}\le C^{(3)}_{K,\rho}\,r^{-2}.
\]
Here and below these estimates are asserted for $0<r\le1$.

Applying the interpolation inequality
\eqref{eq:ball_sobolev_interpolation} to $h=f_r$ yields
\[
\|f_r\|_{W^{5/2}(L^2(B_{R_K}))}
\le c'_{\rm moll,K}\,r^{-3/2},
\qquad 0<r\le1,
\]
for a suitable constant $c'_{\rm moll,K}<\infty$ depending only on $K$ and
the mollifier.
Combining this with \eqref{eq:moll_sup}, we obtain
\begin{equation}\label{eq:moll_bounds}
\|f-f_r\|_{L^\infty(\overline K)}\le c_{\rm moll}\,r,
\qquad
\|f_r\|_{W^{5/2}(L^2(B_{R_K}))}\le c'_{\rm moll,K}\, r^{-3/2},
\qquad 0<r\le1.
\end{equation}

\smallskip
\noindent\textbf{Step 3: bound the smooth part.}

Apply \Cref{lem:sobolev_relu_control} with $R=R_K$, $h=f_r$, and the
finite signed measure $\eta$, which is supported in
$\overline K\subset B_{R_K}$. The sphere-form estimate
\eqref{eq:sobolev_sphere_integral_control}, together with
\eqref{eq:moll_bounds}, gives
\[
\Big|\int_{\overline K} f_r\,d\eta\Big|
\le
\sqrt{1+R_K^2}\,A_{R_K}\,c'_{\rm moll,K}\,r^{-3/2}
\sup_{v\in \Ss^2}
\Big|\int_{\overline K}\relu(\langle v,(1,u)\rangle)\,d\eta(u)\Big|.
\]
Since $\eta=\mu^{\aug}-\nu^{\aug}$, by construction
\[
\int_{\overline K}\relu(\langle v,(1,u)\rangle)\,d\eta(u)
=
S(\mu)(v)-S(\nu)(v),
\]
so the last supremum equals $\varepsilon$. Therefore
\begin{equation}\label{eq:smooth_bound_rewrite}
\Big|\int_{\overline K} f_r\,d\eta\Big|
\le
\sqrt{1+R_K^2}\,A_{R_K}\,c'_{\rm moll,K}\,r^{-3/2}\,\varepsilon.
\end{equation}

\smallskip
\noindent\textbf{Step 4: bound the mollification error.}

By \eqref{eq:moll_bounds},
\[
\Big|\int_{\overline K}(f-f_r)\,d\eta\Big|
\le
\|f-f_r\|_{L^\infty(\overline K)}\,\|\eta\|_{\rm TV}
\le
c_{\rm moll}\,r\,\|\eta\|_{\rm TV}.
\]
Since $\mu,\nu$ are supported in $K$ and $\pers(p)\ge \delta_K$ on $K$, we have
\[
\mu(K)\le \delta_K^{-1}\pers(\mu),
\qquad
\nu(K)\le \delta_K^{-1}\pers(\nu).
\]
Moreover,
\[
\|\eta\|_{\rm TV}
\le
\|\mu^{\aug}\|_{\rm TV}+\|\nu^{\aug}\|_{\rm TV}
\le
2\mu(K)+2\nu(K),
\]
because $\mu^{\aug}=\mu-(\pi_\Delta)_\#\mu$ and similarly for $\nu^{\aug}$.
Hence
\[
\|\eta\|_{\rm TV}
\le
\frac{2}{\delta_K}\big(\pers(\mu)+\pers(\nu)\big)
=
\frac{2}{\delta_K}M,
\]
and therefore
\begin{equation}\label{eq:error_bound_rewrite}
\Big|\int_{\overline K}(f-f_r)\,d\eta\Big|
\le
\frac{2}{\delta_K}\,c_{\rm moll}\,r\,M.
\end{equation}

\smallskip
\noindent\textbf{Step 5: optimize the mollification scale.}

Combining \eqref{eq:smooth_bound_rewrite} and \eqref{eq:error_bound_rewrite}, we find
\[
\Big|\int_{\overline K} f\,d\eta\Big|
\le
\underbrace{\sqrt{1+R_K^2}\,A_{R_K}\,c'_{\rm moll,K}}_{=:\kappa_K}\,r^{-3/2}\varepsilon
+
\underbrace{\frac{2}{\delta_K}\,c_{\rm moll}}_{=:\lambda_K}\,rM.
\]
This estimate holds for $0<r\le1$. If $M=0$, then $\mu=\nu=0$ and there is
nothing to prove; if $\varepsilon=0$, letting $r\downarrow0$ gives the claim.
Assume henceforth that $M,\varepsilon>0$, and set
\[
r_*:=\left(\frac{3\kappa_K\varepsilon}{2\lambda_KM}\right)^{2/5}.
\]
If $r_*\le1$, a direct minimization shows that, for every
$\kappa,\lambda>0$,
\begin{equation}\label{eq:Copt_rewrite}
\inf_{r>0}\big(\kappa\,r^{-3/2}+\lambda\,r\big)
=
C_{\mathrm{opt}}\,\kappa^{2/5}\lambda^{3/5},
\qquad
C_{\mathrm{opt}}:=\frac{5}{2}\Big(\frac{3}{2}\Big)^{-3/5}.
\end{equation}
The objective above has coefficients $\kappa_K\varepsilon$ and
$\lambda_KM$. Applying \eqref{eq:Copt_rewrite} to this objective, then taking
the supremum in \eqref{eq:KR_aug_rewrite} and using
$\POT_1\le\OT_1$ from Proposition~\ref{prop:POT_vs_OT_aug}, gives
\[
\POT_1(\mu,\nu)
\le C_{\mathrm{opt}}\kappa_K^{2/5}\lambda_K^{3/5}
M^{3/5}\varepsilon^{2/5}.
\]

If $r_*>1$, then
\[
M<\frac{3\kappa_K}{2\lambda_K}\,\varepsilon.
\]
Using the zero partial plan gives $\POT_1(\mu,\nu)\le M$, and hence
\[
\POT_1(\mu,\nu)
\le
\left(\frac{3\kappa_K}{2\lambda_K}\right)^{2/5}
M^{3/5}\varepsilon^{2/5}.
\]
It follows in either case that
\[
\POT_1(\mu,\nu)
\le C_KM^{3/5}\varepsilon^{2/5},
\]
where one may take
\begin{equation}\label{eq:CK_explicit}
C_K:=\max\left\{
C_{\mathrm{opt}}\kappa_K^{2/5}\lambda_K^{3/5},
\left(\frac{3\kappa_K}{2\lambda_K}\right)^{2/5}
\right\}.
\end{equation}
Finally,
\[
\POT_1(\mu,\nu)\le \OT_1(\mu\oplus_\Delta \nu,\nu\oplus_\Delta \mu),
\]
so the thesis follows.
\end{proof}

\begin{rmk}\label{rmk:CK_explicit}
In \eqref{eq:CK_explicit},
\[
\kappa_K:=\sqrt{1+R_K^2}\,A_K\,c'_{\rm moll,K},
\qquad
\lambda_K:=\frac{2}{\delta_K}\,c_{\rm moll},
\]
where
\[
R_K:=1+\sup_{u\in \overline K}\|u\|_2,
\qquad
B_{R_K}:=\{p\in \R^2:\|p\|_2\le R_K\},
\qquad
C_{\mathrm{opt}}:=\frac{5}{2}\Big(\frac{3}{2}\Big)^{-3/5}.
\]
Here \(c_{\rm moll}\) is the constant from \eqref{eq:moll_sup}, and
\(c'_{\rm moll,K}\) is any constant such that
\[
\|f_r\|_{W^{5/2}(L^2(B_{R_K}))}\le c'_{\rm moll,K}\,r^{-3/2}
\]
for the normalized mollifications used in the proof. Moreover,
\[
A_K=A_{R_K}=\frac{A_0\,C_{{\rm sc},R_K}}{R_K},
\]
where \(A_0\) is the unit-ball constant in \eqref{eq:mao_unit} and
\(C_{{\rm sc},R_K}\) is the scaling constant from \eqref{eq:scale_frac}.
In particular, once the mollifier is fixed, all quantities entering \eqref{eq:CK_explicit}
depend only on \(K\).
\end{rmk}

\subsubsection{Main Inverse-Continuity Result}

We can now complete the inverse-continuity proof. The vanishing lemmas show that, under uniform convergence of persistence spheres toward a compactly supported target $\mu$, the persistence mass of $\mu_n$ outside a suitable compact core becomes negligible. Once restricted to such a fixed compact core, the local H\"older estimate above gives quantitative control of $\POT_1$ in terms of the residual sphere discrepancy. This yields the desired global convergence.

\begin{theorem}[Uniform Spheres $\Rightarrow$ $\POT_1$ Convergence]\label{thm:Sunif_to_POT}
Let $\mu\in\mathcal M$ have compact support in $X$, and let $\{\mu_n\}\subset\mathcal M$.
If $\|S(\mu_n)-S(\mu)\|_\infty\to0$, then $\POT_1(\mu_n,\mu)\to0$.
\end{theorem}

\begin{proof}
If $\mu=0$, then \Cref{lem:pers_eval} and the assumed uniform convergence give
\[
\pers(\mu_n)=\frac1{\sqrt2}S(\mu_n)(v_{\pers})\longrightarrow0.
\]
By \Cref{rmk:pot_to_zero},
\[
\POT_1(\mu_n,0)=\pers(\mu_n)\longrightarrow0.
\]
Hence, in the remainder of the proof, we may assume $\mu\ne0$.

\smallskip
\noindent\textbf{Step 1: choose a fixed compact core.}
Since $\supp\mu\subset X$ is compact, there exist
\[
\delta_\mu:=\inf_{p\in\supp\mu}\pers(p)>0,
\qquad
M_\mu:=\sup_{p\in\supp\mu}\pers(p)<\infty,
\qquad
R_\mu:=\sup_{p\in\supp\mu}|d(p)|<\infty.
\]
Choose
\[
0<\delta<\frac{\delta_\mu}{2},
\qquad
M>2M_\mu.
\]
Then
\[
\mu(\{\pers\le 2\delta\})=0,
\qquad
\supp\mu\subset \{|d|\le R_\mu,\ \delta\le \pers\le M_\mu\}.
\]

By the strengthened form of \Cref{lem:d_tail}, there exists $R>R_\mu$,
depending only on $\mu$, such that
\[
\limsup_{n\to\infty}\int_{\{|d|\ge R\}}\pers\,d\mu_n =0.
\]
Moreover, by \Cref{lem:lowpers_vanish},
\[
\int_{\{\pers<\delta\}}\pers\,d\mu_n \to 0,
\]
and by \Cref{lem:highpers_vanish},
\[
\limsup_{n\to\infty}\int_{\{\pers\ge M\}}\pers\,d\mu_n = 0.
\]

Set
\[
K:=K_{R,\delta,M}:=\{p\in X:\ |d(p)|\le R,\ \delta\le \pers(p)\le M\}.
\]
Then \(K\) is compact, \(\supp\mu\subset K\), and
\begin{equation}\label{eq:outside_core_zero_rewrite}
\limsup_{n\to\infty}\int_{K^c}\pers\,d\mu_n =0.
\end{equation}

\smallskip
\noindent\textbf{Step 2: control convergence on the fixed compact core.}
Set
\[
\nu_n:=\mu_n|_K.
\]
Then \(\nu_n\) and \(\mu\) are both supported in the fixed compact set \(K\).

By \Cref{lem:sphere_trunc} and \eqref{eq:outside_core_zero_rewrite},
\[
\limsup_{n\to\infty}\|S(\mu_n)-S(\nu_n)\|_\infty
\le
\sqrt2\,\limsup_{n\to\infty}\int_{K^c}\pers\,d\mu_n
=0.
\]
Since \(\|S(\mu_n)-S(\mu)\|_\infty\to0\), it follows that
\begin{equation}\label{eq:compact_sphere_defect_zero_rewrite}
\|S(\nu_n)-S(\mu)\|_\infty\to 0.
\end{equation}

Now apply \Cref{prop:compact_core_holder_siegel} on the fixed compact set \(K\).
Since
\[
\pers(\nu_n)\le \pers(\mu_n),
\]
and \Cref{lem:pers_eval} gives \(\pers(\mu_n)\to\pers(\mu)\), there exists
\(C_{\pers}<\infty\) such that
\[
\sup_n\big(\pers(\nu_n)+\pers(\mu)\big)\le C_{\pers}.
\]
Hence, for all \(n\),
\[
\POT_1(\nu_n,\mu)
\le
C_K\,\big(\pers(\nu_n)+\pers(\mu)\big)^{3/5}\,\|S(\nu_n)-S(\mu)\|_\infty^{2/5}
\le
C_K\,C_{\pers}^{3/5}\,\|S(\nu_n)-S(\mu)\|_\infty^{2/5}.
\]
Using \eqref{eq:compact_sphere_defect_zero_rewrite}, we conclude that
\begin{equation}\label{eq:compact_pot_zero_rewrite}
\POT_1(\nu_n,\mu)\to 0.
\end{equation}

\smallskip
\noindent\textbf{Step 3: return to the full measures.}
By the triangle inequality,
\[
\POT_1(\mu_n,\mu)\le \POT_1(\mu_n,\nu_n)+\POT_1(\nu_n,\mu).
\]
For the first term, match \(\nu_n\) to itself and leave
\(\mu_n-\nu_n=\mu_n|_{K^c}\) unmatched. By \Cref{rmk:pot_to_zero},
\[
\POT_1(\mu_n,\nu_n)\le \POT_1(\mu_n|_{K^c},0)=\int_{K^c}\pers\,d\mu_n.
\]
Therefore, by \eqref{eq:outside_core_zero_rewrite},
\[
\limsup_{n\to\infty}\POT_1(\mu_n,\nu_n)\le
\limsup_{n\to\infty}\int_{K^c}\pers\,d\mu_n=0,
\]
hence
\begin{equation}\label{eq:full_to_core_zero_rewrite}
\POT_1(\mu_n,\nu_n)\to0.
\end{equation}

Combining \eqref{eq:compact_pot_zero_rewrite} and \eqref{eq:full_to_core_zero_rewrite},
we obtain
\[
\POT_1(\mu_n,\mu)\to0.
\]
\end{proof}

\begin{corollary}[Counting Measures]\label{cor:counting}
Assume $\mu_n=\sum_{i=1}^{N_n}\delta_{p_{n,i}}$ and $\mu=\sum_{i=1}^{N}\delta_{p_{i}}$.
If $\|S(\mu_n)-S(\mu)\|_\infty\to0$, then $\POT_1(\mu_n,\mu)\to0$.
\end{corollary}

\section{Hilbert Space Results}
\label{sec:hilbert}

In addition to the bi-continuity statement proved above, we collect convergence results that are particularly relevant in applications.

In practice, persistence spheres are used as elements of the Hilbert space $L^2(\Ss^2)$ to exploit inner products, PCA-type decompositions, and gradient-based learning pipelines. However, our main stability and inverse-continuity results are naturally formulated in the uniform norm on $\Ss^2$. The aim of this section is to relate these two viewpoints. We first show that, under a uniform Lipschitz bound on the sphere, $L^2$ and uniform convergence are equivalent. We then introduce a growth-controlled class of measures for which this regularity follows from explicit moment bounds on the associated augmented measures.

\subsection{Uniform vs $L^2$ Convergence on the Sphere}

We begin with a general comparison between the uniform and \(L^2\) topologies on
\(\Ss^2\). Uniform convergence always implies \(L^2\) convergence, while the converse
holds for equi-Lipschitz families. Applied to persistence spheres, this shows that once
a uniform Lipschitz bound is available, the Hilbert-space viewpoint in \(L^2(\Ss^2)\)
is fully compatible with the sup-norm framework used in the stability and inverse-continuity
results above.

\begin{prop}[Uniform vs.\ $L^2$]\label{cor:Linfty_L2}
Let $h_n,h\in C(\Ss^2)$. Then
\[
\|h_n-h\|_{L^2(\Ss^2)} \le |\Ss^2|^{1/2}\,\|h_n-h\|_\infty = (4\pi)^{1/2}\,\|h_n-h\|_\infty,
\]
so $\|h_n-h\|_\infty\to0$ implies $\|h_n-h\|_{L^2(\Ss^2)}\to0$.

Conversely, assume that $g_n:=h_n-h$ is $L$-Lipschitz on $\Ss^2$ (for the Euclidean metric inherited from $\R^3$),
with a constant $L$ independent of $n$. Then for every $n$,
\begin{equation}\label{eq:Linfty_from_L2_explicit}
\|g_n\|_\infty
\le
2\,\max\Big\{\pi^{-1/4}\,L^{1/2}\,\|g_n\|_{L^2(\Ss^2)}^{1/2},\ \pi^{-1/2}\,\|g_n\|_{L^2(\Ss^2)}\Big\}.
\end{equation}
In particular, under a uniform Lipschitz bound, $L^2$ and uniform convergence are equivalent.
\end{prop}

\begin{proof}
The first inequality follows immediately from
\[
\|g\|_{L^2(\Ss^2)}\le |\Ss^2|^{1/2}\|g\|_\infty.
\]

For the converse, fix \(n\) and write \(g:=g_n\), \(M:=\|g\|_\infty\).
Choose \(x\in \Ss^2\) such that \(|g(x)|=M\), and let \(\sigma\) denote the
surface measure on \(\Ss^2\). For \(r\in(0,2)\), set
\[
B(x,r):=\{y\in \Ss^2:\|y-x\|_2\le r\}.
\]
Since \(g\) is \(L\)-Lipschitz, for every \(y\in B(x,r)\) one has
\[
|g(y)|
\ge |g(x)|-L\|y-x\|_2
\ge M-Lr.
\]
Hence
\[
|g(y)|\ge (M-Lr)_+
\qquad\forall y\in B(x,r),
\]
and therefore
\begin{equation}\label{eq:L2_ball_lower}
\|g\|_{L^2(\Ss^2)}^2
=
\int_{\Ss^2}|g|^2\,d\sigma
\ge
\int_{B(x,r)}|g(y)|^2\,d\sigma(y)
\ge
\sigma(B(x,r))\,(M-Lr)_+^2.
\end{equation}

We now compute \(\sigma(B(x,r))\). Writing \(y\) in polar angle \(\theta\) from
\(x\), so that \(\cos\theta=\langle x,y\rangle\), one has
\[
\|y-x\|_2^2
=
\|x\|_2^2+\|y\|_2^2-2\langle x,y\rangle
=
2-2\cos\theta
=
4\sin^2(\theta/2).
\]
Thus \(\|y-x\|_2\le r\) is equivalent to \(\theta\le 2\arcsin(r/2)\). The area
of a spherical cap of angular radius \(\theta\) is
\[
\sigma(B(x,r))=2\pi(1-\cos\theta).
\]
With \(\theta=2\arcsin(r/2)\), using \(\cos(2\alpha)=1-2\sin^2\alpha\), we get
\[
\cos\theta
=
\cos\bigl(2\arcsin(r/2)\bigr)
=
1-\frac{r^2}{2},
\]
hence
\[
\sigma(B(x,r))
=
2\pi\Bigl(1-\Bigl(1-\frac{r^2}{2}\Bigr)\Bigr)
=
\pi r^2
\qquad\text{for all }0<r<2.
\]
Substituting into \eqref{eq:L2_ball_lower}, we obtain
\[
\|g\|_{L^2(\Ss^2)} \ge \sqrt{\pi}\,r\,(M-Lr)_+
\qquad\forall r\in(0,2).
\]

\smallskip
\noindent\textbf{Case 1: \(M\le 2L\).}
Choose \(r=M/(2L)\in(0,1]\subset(0,2)\). Then \(M-Lr=M/2\), so
\[
\|g\|_{L^2(\Ss^2)}
\ge
\sqrt{\pi}\,\frac{M}{2L}\cdot\frac{M}{2}
=
\frac{\sqrt{\pi}}{4L}\,M^2.
\]
Therefore
\[
M\le 2\,\pi^{-1/4}\,L^{1/2}\,\|g\|_{L^2(\Ss^2)}^{1/2}.
\]

\smallskip
\noindent\textbf{Case 2: \(M>2L\).}
Choose \(r=1\in(0,2)\). Then
\[
M-Lr=M-L>\frac M2,
\]
hence
\[
\|g\|_{L^2(\Ss^2)}
\ge
\sqrt{\pi}\cdot 1\cdot \frac M2,
\]
and so
\[
M\le 2\,\pi^{-1/2}\,\|g\|_{L^2(\Ss^2)}.
\]

Combining the two cases proves \eqref{eq:Linfty_from_L2_explicit}.
\end{proof}

The previous proposition shows that passing from \(L^2\) convergence to uniform convergence on \(\Ss^2\) reduces to controlling the Lipschitz constants of the sphere functions. For persistence spheres, this regularity is naturally linked to quantitative bounds on the underlying measures. We now introduce a class tailored to this purpose, and then show that on this class \(L^2\) convergence of persistence spheres upgrades to uniform convergence.

In the compact-support regime treated in \Cref{sec:compact_rates}, one can go further and obtain explicit inverse estimates in \(\POT_1\) from sup-norm discrepancies of the associated sphere functions. The results below should be viewed as a Hilbert-space counterpart to that compact-core theory.

\subsection{A growth--controlled class $\mathcal M(\mathcal A,\mathcal B)$}

As explained by \Cref{cor:Linfty_L2}, the main obstacle to upgrading \(L^2\) convergence to uniform convergence is the possible lack of a uniform Lipschitz bound on the differences \(S(\mu_n)-S(\mu)\). To obtain such control, we work with measures whose local first moment and persistence-weighted far field are quantitatively bounded. Since the argument proceeds by truncating to large bounded cores and then applying \Cref{cor:Linfty_L2} on each core, we also impose a compatibility condition ensuring that the truncation error decays fast enough relative to the growth of the local Lipschitz constants.

\begin{defi}\label{def:MA_B}
Let $\mathcal A,\mathcal B$ be positive Borel measures on $\R^2$. We say that the pair $(\mathcal A,\mathcal B)$ is compatible if:
\begin{itemize}
\item $\mathcal A$ is a Radon measure;
\item $\mathcal B$ is a finite measure;
\item
\begin{equation}\label{eq:AB_compatibility}
\frac{\mathcal A(B_R)\,\mathcal B(\R^2\setminus B_{R-1})}{R}\xrightarrow{R\to\infty} 0,
\end{equation}
where $B_R=\{p\in\R^2:\|p\|_2\le R\}$.
\end{itemize}

Given a compatible pair $(\mathcal A,\mathcal B)$, we define $\mathcal M(\mathcal A,\mathcal B)$ as the set of integrable measures $\mu$ on $X$ such that for every Borel set $U\subset\R^2$,
\begin{align}
\int_{U\cap X} \|(1,p)\|_2\,d\mu(p) &\le \mathcal A(U), \label{eq:A_control}\\
\int_{U\cap X} \|(1,p)\|_2\,\pers(p)\,d\mu(p) &\le \mathcal B(U). \label{eq:B_control}
\end{align}
\end{defi}

\begin{rmk}\label{rmk:interpretation}
Condition \eqref{eq:A_control} gives a local control of the first moment: on each bounded region, the weighted mass
$\int \|(1,p)\|_2\,d\mu$ is uniformly bounded by $\mathcal A$.
Condition \eqref{eq:B_control} controls the persistence-weighted far field:
since $\mathcal B(\R^2)<\infty$, the measure
$\|(1,p)\|_2\,\pers(p)\,\mu(dp)$ has uniformly bounded total mass.

The compatibility condition \eqref{eq:AB_compatibility} says, roughly speaking, that $\mathcal A(B_R)$ is not allowed to grow too fast compared with the rate at which the tail
$\mathcal B(\R^2\setminus B_{R-1})$ decays. 
\end{rmk}

\begin{example}\label{rmk:compatible_example}
A simple polynomial example is obtained by taking \(\mathcal A\) and \(\mathcal B\)
absolutely continuous with respect to Lebesgue measure, with densities
\[
\rho_{\mathcal A}(p)=(1+\|p\|_2)^{\alpha},
\qquad
\rho_{\mathcal B}(p)=(1+\|p\|_2)^{-\beta},
\qquad p\in\R^2,
\]
for fixed \(\alpha\ge 0\) and \(\beta>\alpha+3\). Then
\[
\mathcal A(B_R)
=
\int_{B_R}(1+\|p\|_2)^{\alpha}\,dp
=
2\pi\int_0^R (1+r)^{\alpha}r\,dr
\lesssim R^{\alpha+2},
\]
and
\[
\mathcal B(\R^2\setminus B_{R-1})
=
\int_{\|p\|_2\ge R-1}(1+\|p\|_2)^{-\beta}\,dp
=
2\pi\int_{R-1}^{\infty}(1+r)^{-\beta}r\,dr
\lesssim R^{2-\beta}.
\]
Therefore
\[
\frac{\mathcal A(B_R)\,\mathcal B(\R^2\setminus B_{R-1})}{R}
\lesssim
R^{\alpha+2}\,R^{2-\beta}\,R^{-1}
=
R^{\alpha+3-\beta}\longrightarrow 0.
\]
Hence \((\mathcal A,\mathcal B)\) is compatible.

In particular, any faster decay for \(\rho_{\mathcal B}\), for instance
\[
\rho_{\mathcal B}(p)=e^{-\|p\|_2},
\]
also yields a compatible pair.
\end{example}

We now show that $L^2$ convergence implies uniform convergence on $\mathcal M(\mathcal A,\mathcal B)$.

For $R>0$, set
\[
K_R:=\{p\in X:\ \|(1,p)\|_2\le R\},
\qquad
K_R^\Delta:=K_R\cup \pi_\Delta(K_R)\subset\overline X.
\]

The first ingredient is a uniform truncation estimate, obtained by combining the truncation lemma
\[
\|S(\eta)-S(\eta|_A)\|_\infty \le \sqrt2\int_{A^c}\pers\,d\eta
\]
with the mixed moment control \eqref{eq:B_control}.

\begin{lem}[Uniform Truncation on $\mathcal M(\mathcal A,\mathcal B)$]\label{lem:uniform_truncation}
Let $\mu\in\mathcal M(\mathcal A,\mathcal B)$ and $R\ge 1$. Then
\begin{equation}\label{eq:trunc_bound_B}
\|S(\mu)-S(\mu|_{K_R})\|_\infty
\le
\frac{\sqrt2}{R}\,\mathcal B\!\big(\R^2\setminus B_{R-1}\big),
\end{equation}
where $B_{R-1}:=\{p\in\R^2:\|p\|_2\le R-1\}$.
In particular, the right-hand side tends to $0$ as $R\to\infty$, uniformly over $\mu\in\mathcal M(\mathcal A,\mathcal B)$.
\end{lem}

\begin{proof}
Since $R\ge 1$, we have the implication $\|(1,p)\|_2\ge R \Rightarrow \|p\|_2\ge R-1$, hence
\[
K_R^c\subset \R^2\setminus B_{R-1}.
\]
Moreover, since on $K_R^c$ we have $\|(1,p)\|_2\ge R$, and
\[
\pers(p)\le \frac{1}{R}\,\|(1,p)\|_2\,\pers(p).
\]
Therefore
\[
\int_{K_R^c}\pers\,d\mu
\le
\frac1R\int_{K_R^c}\|(1,p)\|_2\,\pers(p)\,d\mu(p)
\le
\frac1R\,\mathcal B(\R^2\setminus B_{R-1}),
\]
by \eqref{eq:B_control} with $U=\R^2\setminus B_{R-1}$.
The truncation lemma then gives \eqref{eq:trunc_bound_B}.
\end{proof}

Next, on the truncated core $K_R$ one has a uniform Lipschitz control in $v$, depending only on $\mathcal A(B_R)$.

\begin{lem}[Local Uniform Lipschitz Control]\label{lem:core_lip}
Fix $R\ge 1$. There exists a constant $L_R<\infty$, depending only on $\mathcal A(B_R)$, such that for every
$\mu\in\mathcal M(\mathcal A,\mathcal B)$,
\[
\Lip_{\Ss^2}\big(S(\mu|_{K_R})\big)\le L_R.
\]
One may take, for instance,
\[
L_R:=2\,\mathcal A(B_R).
\]
\end{lem}

\begin{proof}
Let $\nu:=\mu|_{K_R}$. Using that $\relu$ is $1$-Lipschitz and the definition of $S$,
for $v,w\in\Ss^2$,
\begin{align*}
|S(\nu)(v)-S(\nu)(w)|
&\le
\int_{\overline X}
\big|\relu(\langle v,(1,u)\rangle)-\relu(\langle w,(1,u)\rangle)\big|
\,d|\nu^{\aug}|(u)\\
&\le
\int_{\overline X}
|\langle v-w,(1,u)\rangle|
\,d|\nu^{\aug}|(u)\\
&\le
\|v-w\|_2
\int_{\overline X}\|(1,u)\|_2\,d|\nu^{\aug}|(u).
\end{align*}
Moreover, $|\nu^{\aug}|= \nu+(\pi_\Delta)_\#\nu$, and $\|(1,\pi_\Delta(p))\|_2\le \|(1,p)\|_2$, hence
\[
\int_{\overline X}\|(1,u)\|_2\,d|\nu^{\aug}|(u)
\le
2\int_{K_R}\|(1,p)\|_2\,d\mu(p)
\le
2\,\mathcal A(B_R),
\]
by \eqref{eq:A_control} with $U=B_R$.
This yields the stated Lipschitz bound.
\end{proof}

Putting the truncation and core Lipschitz bounds together yields the desired upgrade from $L^2$ to uniform convergence.

\begin{theorem}[$L^2\Rightarrow L^\infty$ on $\mathcal M(\mathcal A,\mathcal B)$]\label{thm:L2_to_Linfty_MAB}
Let $\mu_n,\mu\in\mathcal M(\mathcal A,\mathcal B)$ and assume
\[
\|S(\mu_n)-S(\mu)\|_{L^2(\Ss^2)}\longrightarrow 0.
\]
Then
\[
\|S(\mu_n)-S(\mu)\|_\infty\longrightarrow 0.
\]
\end{theorem}

\begin{proof}
For \(R\ge 1\), define the bounded core
\[
K_R:=\{p\in X:\ \|(1,p)\|_2\le R\},
\]
the truncated measures
\[
\nu_n:=\mu_n|_{K_R},
\qquad
\nu:=\mu|_{K_R},
\]
and the truncation defect
\[
\tau_R:=\frac{\sqrt2}{R}\,\mathcal B(\R^2\setminus B_{R-1}).
\]
By \Cref{lem:uniform_truncation}, for every \(\eta\in\mathcal M(\mathcal A,\mathcal B)\),
\begin{equation}\label{eq:L2Linfty_trunc_unif}
\|S(\eta)-S(\eta|_{K_R})\|_\infty\le \tau_R.
\end{equation}
Since \(\mathcal B\) is finite, \(\tau_R\to 0\) as \(R\to\infty\).

Fix \(R\ge 1\). By \Cref{lem:core_lip}, each of \(S(\nu_n)\) and \(S(\nu)\) is \(L_R\)-Lipschitz on \(\Ss^2\), where
\[
L_R:=2\,\mathcal A(B_R).
\]
Hence the differences
\[
h_n:=S(\nu_n)-S(\nu)
\]
are \(2L_R\)-Lipschitz on \(\Ss^2\).

We now estimate \(\|h_n\|_{L^2(\Ss^2)}\) for every fixed \(R\).
Indeed, by the triangle inequality,
\[
\|h_n\|_{L^2(\Ss^2)}
\le
\|S(\nu_n)-S(\mu_n)\|_{L^2(\Ss^2)}
+
\|S(\mu_n)-S(\mu)\|_{L^2(\Ss^2)}
+
\|S(\mu)-S(\nu)\|_{L^2(\Ss^2)}.
\]
Applying the first part of \Cref{cor:Linfty_L2} to the first and third terms, and using
\eqref{eq:L2Linfty_trunc_unif}, we obtain
\[
\|S(\nu_n)-S(\mu_n)\|_{L^2(\Ss^2)}
\le
(4\pi)^{1/2}\tau_R,
\qquad
\|S(\mu)-S(\nu)\|_{L^2(\Ss^2)}
\le
(4\pi)^{1/2}\tau_R.
\]
Therefore, using the assumed $L^2$ convergence,
\[
\limsup_{n\to\infty}\|h_n\|_{L^2(\Ss^2)}
\le
2(4\pi)^{1/2}\tau_R.
\]

Now apply the converse part of \Cref{cor:Linfty_L2} to the \(2L_R\)-Lipschitz functions \(h_n\):
\[
\|h_n\|_\infty
\le
2\,\max\Big\{
\pi^{-1/4}(2L_R)^{1/2}\|h_n\|_{L^2(\Ss^2)}^{1/2},
\ \pi^{-1/2}\|h_n\|_{L^2(\Ss^2)}
\Big\}.
\]
Passing to the limsup gives
\begin{align}
\limsup_{n\to\infty}\|h_n\|_\infty
&\le
2\,\max\Big\{
\pi^{-1/4}(2L_R)^{1/2}\bigl(2(4\pi)^{1/2}\tau_R\bigr)^{1/2},
\ \pi^{-1/2}2(4\pi)^{1/2}\tau_R
\Big\} \notag\\
&=
\max\Big\{
4\sqrt{2L_R\tau_R},
\ 8\tau_R
\Big\}. \label{eq:L2Linfty_core_limsup}
\end{align}

On the other hand, by the triangle inequality and \eqref{eq:L2Linfty_trunc_unif},
\[
\|S(\mu_n)-S(\mu)\|_\infty
\le
\|S(\mu_n)-S(\nu_n)\|_\infty
+
\|h_n\|_\infty
+
\|S(\nu)-S(\mu)\|_\infty
\le
2\tau_R+\|h_n\|_\infty.
\]
Hence, using \eqref{eq:L2Linfty_core_limsup},
\begin{equation}\label{eq:L2Linfty_final_limsup_rewrite}
\limsup_{n\to\infty}\|S(\mu_n)-S(\mu)\|_\infty
\le
2\tau_R+\max\Big\{
4\sqrt{2L_R\tau_R},
\ 8\tau_R
\Big\}.
\end{equation}

It remains to let \(R\to\infty\). Since \(\tau_R\to0\), the pure \(\tau_R\)-terms vanish.
Moreover,
\[
L_R\tau_R
=
2\,\mathcal A(B_R)\cdot \frac{\sqrt2}{R}\,\mathcal B(\R^2\setminus B_{R-1}),
\]
which tends to \(0\) by the compatibility condition \eqref{eq:AB_compatibility}. Therefore the right-hand side of
\eqref{eq:L2Linfty_final_limsup_rewrite} tends to \(0\), and thus
\[
\limsup_{n\to\infty}\|S(\mu_n)-S(\mu)\|_\infty=0.
\]
Hence
\[
\|S(\mu_n)-S(\mu)\|_\infty\to0.
\]
\end{proof}

\Cref{thm:L2_to_Linfty_MAB} ensures that, on the class $\mathcal M(\mathcal A,\mathcal B)$,
working in the Hilbert space $L^2(\Ss^2)$ is consistent with the uniform--norm theory developed for persistence spheres:
$L^2$ control of sphere errors automatically yields uniform control.

In particular, whenever the reference measure $\mu$ is compactly supported in $X$, \Cref{thm:Sunif_to_POT} can now be invoked from an $L^2$ premise.

\begin{cor}[From $L^2$ sphere convergence to $\POT_1$ convergence]\label{cor:L2_to_POT}
Let $\mu\in\mathcal M(\mathcal A,\mathcal B)$ have compact support in $X$, and let $\mu_n\in\mathcal M(\mathcal A,\mathcal B)$.
If $\|S(\mu_n)-S(\mu)\|_{L^2(\Ss^2)}\to 0$, then $\POT_1(\mu_n,\mu)\to 0$.
\end{cor}

\begin{defi}\label{def:McAB}
Let $\mathcal A,\mathcal B$ be as in \Cref{def:MA_B}. We define
\[
\mathcal M_c(\mathcal A,\mathcal B)\subset \mathcal M(\mathcal A,\mathcal B),
\]
as the set of measures in $\mathcal M(\mathcal A,\mathcal B)$ whose support is compact in $X$.
\end{defi}

\begin{corollary}[Bi-Continuous Embedding]\label{cor:McAB_bicont_L2}
Fix $\mathcal A,\mathcal B$ as in \Cref{def:MA_B}. The persistence-sphere map
\[
S:\big(\mathcal M_c(\mathcal A,\mathcal B),\ \POT_1\big)\ \longrightarrow\ \big(L^2(\Ss^2),\ \|\cdot\|_{L^2}\big),
\qquad
\mu\longmapsto S(\mu),
\]
is injective and bi-continuous onto its image. More precisely:
\begin{enumerate}
\item for all $\mu,\nu\in\mathcal M_c(\mathcal A,\mathcal B)$,
\[
\|S(\mu)-S(\nu)\|_{L^2(\Ss^2)}
\le (4\pi)^{1/2}\|S(\mu)-S(\nu)\|_\infty
\le 4\sqrt{2\pi}\ \POT_1(\mu,\nu).
\]
\item if $\mu_n,\mu\in\mathcal M_c(\mathcal A,\mathcal B)$ and
$\|S(\mu_n)-S(\mu)\|_{L^2(\Ss^2)}\to0$, then $\POT_1(\mu_n,\mu)\to0$.
\end{enumerate}
\end{corollary}

\begin{rmk}[Positioning of persistence spheres]\label{rmk:positioning}
The results above describe how
convergence of the representation determines $\POT_1$ convergence at every
compactly supported target and viceversa: this holds in the uniform norm by
\Cref{thm:Sunif_to_POT}, and in the $L^2(\Ss^2)$ norm on
$\mathcal M_c(\mathcal A,\mathcal B)$ by \Cref{cor:McAB_bicont_L2}.
In particular, the class \(\mathcal M_c(\mathcal A,\mathcal B)\) from \Cref{def:McAB}, in contrast to most Hilbert-embedding results
\citep{carriere2017sliced,mitra2024geometric,bate2024bi}, imposes no bound on total mass. Hence, for
counting measures, the number of points may diverge, provided the additional mass concentrates near
the diagonal, allowing unbounded cardinality and noisy settings. To refine the comparison, further details on the role of diagonal augmentation can be found in Appendix~\ref{sec:augment_discussion}, while
Appendix~\ref{sec:deforming_geometry} gives concrete examples of the different
geometric biases introduced by commonly used vectorizations.
\end{rmk}

\section{Unsupervised Simulations}
\label{sec:unsupervised}

In this section we investigate how the qualitative geometric mechanisms discussed in
\Cref{sec:deforming_geometry} manifest themselves in practice. We focus on unsupervised
settings, where the geometry induced by the chosen summary has a more direct impact on the
analysis than in supervised tasks, where part of the bias may be absorbed by the downstream
learner.

A second point concerns optional preprocessing. Although persistence spheres do not require
any parameter to be well defined or stable, one may still reweight or smooth a diagram as a
separate preprocessing step, for instance to mitigate noise. More generally, this possibility is
available for all measure-based representations, and may be useful in unsupervised settings
where there is no predictive model to automatically downweight irrelevant topological
information. For this reason, besides comparing the summaries themselves, we also examine
how optional reweighting choices affect the resulting unsupervised behavior.

\subsubsection{FDA}
\label{sec:fda_benchmark}

We first consider an unsupervised simulation based on a standard functional data analysis
(FDA) generative model \citep{book_fda}, adapted from \citet{pegoraro2025persistence}. The
purpose is to construct a regime in which a strong bias toward high-persistence features is
actually advantageous: the class signal is primarily carried by the largest oscillations of the
underlying smooth functions, while noise mostly produces smaller-scale topological clutter.
Representative samples from the two FDA setups are shown in \Cref{fig:FDA}. 

\paragraph{Generative model.}
We first construct two smooth random functions \(f_1,f_2:[0,1]\to\R\) by cubic-spline
interpolation of random values on a regular grid. Let \(0=t_1<\dots<t_m=1\) be a uniform grid
with \(m=40\) points. For each class \(c\in\{1,2\}\), we sample independently
\[
u_j^{(c)} \sim \mathrm{Unif}([-50,50]), \qquad j=1,\dots,m,
\]
and define \(f_c\) as the cubic spline interpolant of
\(\{(t_j,u_j^{(c)})\}_{j=1}^m\).

Given a noise level \(\sigma\in\{10,15,20\}\), we generate \(N=50\) noisy realizations from each
class as follows. For each replicate \(i=1,\dots,N\) and class \(c\in\{1,2\}\), we first sample an
integer
\[
n_i^{(c)} \sim \mathrm{Unif}\{200,\dots,800\},
\]
then sample locations
\[
X_{i,1}^{(c)},\dots,X_{i,n_i^{(c)}}^{(c)}
\stackrel{\mathrm{i.i.d.}}{\sim}
\mathrm{Unif}([0,1]),
\]
and finally observe noisy values
\[
Y_{i,\ell}^{(c)}
=
f_c\!\left(X_{i,\ell}^{(c)}\right)
+
Z_{i,\ell}^{(c)},
\qquad
Z_{i,\ell}^{(c)}\stackrel{\mathrm{i.i.d.}}{\sim}\mathcal N(0,(\sigma/3)^2),
\qquad
\ell=1,\dots,n_i^{(c)}.
\]
Thus the parameter \(\sigma\) corresponds to a ``\(3\sigma\)'' noise amplitude. Each noisy curve
is encoded as a \(0\)-dimensional persistence diagram, computed from the sublevel-set filtration
of the piecewise-linear interpolation of the sampled points
\(\{(X_{i,\ell}^{(c)},Y_{i,\ell}^{(c)})\}_{\ell=1}^{n_i^{(c)}}\).

The randomization of the sample size \(n_i^{(c)}\) in the interval \([200,800]\) is deliberate.
Larger values of \(n_i^{(c)}\) create more spurious local oscillations, so the amount of
noise-induced topological clutter becomes highly unbalanced across statistical units. This is
particularly unfavorable to \(\POT_1\) and to summaries that remain sensitive to low-persistence
structure. Conversely, summaries with a strong bias toward high persistence may benefit precisely
because they suppress this imbalance. Typical realizations are displayed in \Cref{fig:raw}.

For each method, we compute the pairwise distance matrix between diagrams using \(\POT_1\),
\(\mathrm{SW}\), and the distances induced by persistence spheres (PSphs), persistence landscapes
(PLs), persistence images (PIs), and persistence splines (PSpls). We then apply hierarchical
clustering with average linkage, cut the dendrogram into two clusters, and evaluate the resulting
partition via the Rand index. All reported values are means and standard deviations over \(100\)
independent runs.

\paragraph{Four scenarios.}
We consider four FDA simulations.

\smallskip
\noindent
\textbf{(i) Baseline FDA simulation.}
We use the generative model above with variable sample size
\(n_i^{(c)}\sim \mathrm{Unif}\{200,\dots,800\}\) and compare the methods in their standard
configurations. For persistence images we use persistence as weight; persistence spheres are used
without reweighting. This setup is designed to favor summaries that emphasize high-persistence
features.

\smallskip
\noindent
\textbf{(ii) FDA simulation with persistence-squared reweighting.}
We keep exactly the same generative model, but reweight the diagram measures by
\(\pers(\cdot)^2\) before applying persistence images and persistence spheres. This tests
directly whether an even stronger emphasis on high persistence improves performance. If so, it
supports the interpretation that the strong results of persistence landscapes in (i) are driven
by their bias toward high-persistence variability.

\smallskip
\noindent
\textbf{(iii) FDA simulation with step reweighting.}
Again we keep the same generative model, but now use the simple denoising-type weight
\[
w(x,y)=\frac{2}{\pi}\arctan\!\left(\frac{y-x}{\sigma}\right),
\]
which vanishes on the diagonal and rapidly saturates to \(1\) away from it. This is meant to
mimic a basic ``remove small lifetimes'' strategy. Here we deliberately use the true noise scale
\(\sigma\) from the generative model, so as not to confound the comparison with the additional
problem of estimating a threshold parameter, which is not the point of the present simulation.
Comparing (ii) and (iii) helps separate generic near-diagonal denoising from a genuine
preference for high-persistence features.

\smallskip
\noindent
\textbf{(iv) FDA simulation against high-persistence bias.}
We finally modify the generative model so as to turn the landscape bias against itself. For both
classes, when sampling the smooth function, we fix the first four control-point values to
\[
(-100,100,-100,100),
\]
while the remaining control-point values are still sampled uniformly from \([-50,50]\). Hence
both classes share two very large oscillations, so the variability introduced by those
high-persistence features is no longer discriminative. Moreover, we fix \(n=200\) for every
statistical unit, thereby removing the noise-imbalance confounding factor present in
(i)--(iii). This simulation is designed to show that, once very persistent features cease to
carry class information, a strong bias toward them can become detrimental. Typical realizations
for this modified setup are shown in \Cref{fig:confounding}. 

\paragraph{Parameter choices.}
We use \(\POT_1\) without parameters and approximate sliced Wasserstein using \(50\) random
lines, as in the default \texttt{GUDHI} implementation. Persistence spheres are evaluated on a
polar grid on \(\mathbb S^2\) with \(100\times 200\) samples (latitude \(\times\) longitude). For
persistence landscapes, we retain all landscape functions and compute distances exactly in
\(L^2\). For persistence splines, we use a \(20\times 20\) grid and \(100\) iterations with
stopping tolerance \(10^{-10}\), in line with \citet{dong2024persistence}. For persistence
images, we set the pixel size to \(1/500\) of the minimum side length of the smallest
axis-aligned rectangle supporting the diagrams, rounded to the closest power of \(10\), and take
the Gaussian bandwidth equal to \(10\) times the pixel size. In scenarios (ii) and (iii), the
stated reweightings are applied only to persistence images and persistence spheres.

\paragraph{Results.}
The results for scenarios (i)--(iii), which share the same generative process, are reported in
\Cref{tab:fda_unsup_main}; scenario (iv) is reported separately in
\Cref{tab:fda_unsup_counterbias}.

In the baseline simulation (i), persistence landscapes clearly dominate at all noise levels,
with mean Rand index \(0.984\), \(0.949\), and \(0.846\) for \(\sigma=10,15,20\),
respectively. This strongly suggests that their high-persistence bias is favorable in this
regime, where the class signal is carried mainly by the largest oscillations while variable
sample size introduces abundant, highly unbalanced low-persistence clutter. Persistence
splines also perform well at low noise, while \(\POT_1\), SW, PSph, and PI are substantially
less competitive.

Scenario (ii) confirms this interpretation. Reweighting by \(\pers^2\) markedly improves both
PI and PSph, especially PI. Thus the success of PL in scenario (i) cannot be explained only by
generic robustness to near-diagonal noise: in this FDA regime, a strong emphasis on high
persistence is itself beneficial.

Scenario (iii) separates this effect from simple thresholding. The step weight suppresses small lifetimes, but performs clearly worse than
\(\pers^2\)-reweighting, especially for moderate and large noise. This shows that removing
low-persistence points is not enough: here it is specifically advantageous to bias the geometry
toward the largest topological features.  

Scenario (iv) reverses the picture. Once the two classes share the largest oscillations and the
sample size is fixed at \(n=200\), the strong high-persistence bias of PL becomes detrimental,
and PL drops to essentially random performance. In contrast, \(\POT_1\) becomes the best method
overall, with mean Rand index \(0.975\), \(0.907\), and \(0.888\) for
\(\sigma=10,15,20\). PI, still using persistence as weight, remains strong, plausibly because
its bias toward high persistence is milder than that of PL and can still help linearize
\(\POT_1\) while damping part of the lower-persistence noise. PSph also improves relative to
scenario (i), consistently with the removal of the sample-size imbalance effect. 

Overall, these simulations show that the geometric bias induced by a topological summary can be
highly consequential in unsupervised settings. This makes it important not only to compare
empirical performance, but also to understand which persistence scales a given summary tends to
emphasize or suppress.

\begin{table}[t]
\centering
\renewcommand{\arraystretch}{1.15}
\setlength{\tabcolsep}{4.2pt}
\caption{\textbf{FDA unsupervised simulations (i)--(iii).} Mean $\pm$ standard deviation of
Rand index over $100$ independent runs. Since only persistence images (PI) and persistence
spheres (PSph) change across scenarios (i)--(iii), the values of POT$_1$, SW, PL, and PSpl are
reported only once for each noise level.}
\label{tab:fda_unsup_main}
\resizebox{\linewidth}{!}{%
\begin{tabular}{|c|c|c|c|c|c|c|c|c|c|c|}
\hline
& \multicolumn{4}{c|}{Shared across (i)--(iii)} & \multicolumn{2}{c|}{(i) Baseline} & \multicolumn{2}{c|}{(ii) $\pers^2$} & \multicolumn{2}{c|}{(iii) Step} \\
\hline
$\sigma$ & POT$_1$ & SW & PL & PSpl & PI & PSph & PI & PSph & PI & PSph \\
\hline
$10$
& $0.760 \pm 0.228$
& $0.527 \pm 0.044$
& $\mathbf{0.984 \pm 0.086}$
& $0.964 \pm 0.111$
& $0.783 \pm 0.238$
& $0.530 \pm 0.076$
& $0.958 \pm 0.129$
& $0.762 \pm 0.206$
& $0.933 \pm 0.163$
& $0.672 \pm 0.200$
\\
$15$
& $0.631 \pm 0.152$
& $0.522 \pm 0.039$
& $\mathbf{0.949 \pm 0.140}$
& $0.771 \pm 0.238$
& $0.581 \pm 0.165$
& $0.512 \pm 0.045$
& $0.867 \pm 0.202$
& $0.700 \pm 0.188$
& $0.554 \pm 0.122$
& $0.565 \pm 0.130$
\\
$20$
& $0.547 \pm 0.068$
& $0.526 \pm 0.048$
& $\mathbf{0.846 \pm 0.218}$
& $0.641 \pm 0.215$
& $0.527 \pm 0.053$
& $0.502 \pm 0.010$
& $0.812 \pm 0.214$
& $0.642 \pm 0.163$
& $0.523 \pm 0.050$
& $0.513 \pm 0.027$
\\
\hline
\end{tabular}%
}
\end{table}

\begin{table}[t]
\centering
\renewcommand{\arraystretch}{1.15}
\setlength{\tabcolsep}{4.5pt}
\caption{\textbf{FDA unsupervised simulation (iv).}
Mean $\pm$ standard deviation of Rand index over $100$ independent runs for scenario (iv), where
the two classes share the largest oscillations and the sample size is fixed to $n=200$.}
\label{tab:fda_unsup_counterbias}
\resizebox{\linewidth}{!}{%
\begin{tabular}{|l|c|c|c|c|c|c|c|}
\hline
$\sigma$ & POT$_1$ & SW & PSph & PL & PI & PSpl \\
\hline
$10$ & $\mathbf{0.975 \pm 0.107}$ & $0.616 \pm 0.172$ & $0.641 \pm 0.186$ & $0.515 \pm 0.071$ & $0.795 \pm 0.240$ & $0.496 \pm 0.001$ \\
$15$ & $\mathbf{0.907 \pm 0.191}$ & $0.625 \pm 0.181$ & $0.653 \pm 0.195$ & $0.518 \pm 0.081$ & $0.753 \pm 0.243$ & $0.496 \pm 0.003$ \\
$20$ & $\mathbf{0.888 \pm 0.205}$ & $0.650 \pm 0.191$ & $0.645 \pm 0.190$ & $0.511 \pm 0.054$ & $0.704 \pm 0.240$ & $0.496 \pm 0.002$ \\
\hline
\end{tabular}%
}
\end{table}

\begin{figure}[t]
\centering
	\begin{subfigure}[c]{0.45\textwidth}
		\centering
		\includegraphics[width=\textwidth]{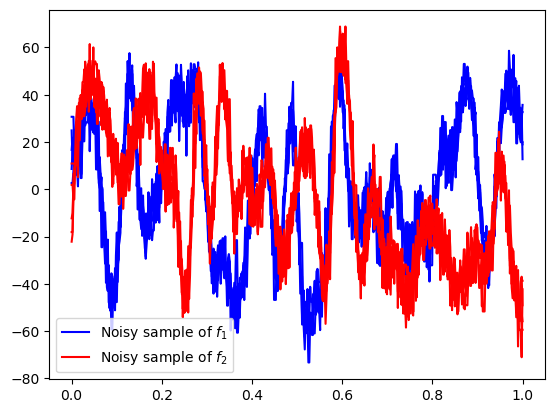}
		\captionsetup{singlelinecheck=off, margin={0.05cm, 0.05cm}}
		\caption{Noisy functions sampled from the model used in scenarios (i)--(iii).}
		\label{fig:raw}
	\end{subfigure}\hfill
	\begin{subfigure}[c]{0.45\textwidth}
		\centering
		\includegraphics[width=\textwidth]{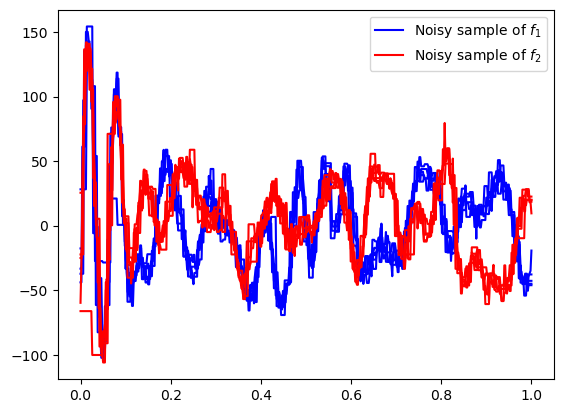}
		\captionsetup{singlelinecheck=off, margin={0.05cm, 0.05cm}}
		\caption{Noisy functions sampled from the model used in scenario (iv), with the first two large oscillations shared across the classes.}
		\label{fig:confounding}
	\end{subfigure}
    \caption{\textbf{FDA simulations.}
    Left: realizations from the baseline generative model used in scenarios (i)--(iii), where the
    discriminative signal is more likely to be carried by large oscillations while noise induces abundant
    low-persistence clutter. Right: realizations from scenario (iv), where the largest
    oscillations are shared across the two classes, so that an excessive bias toward
    high-persistence features becomes detrimental.}
    \label{fig:FDA}
\end{figure}

\subsubsection{Point Processes}\label{sec:pp_benchmark}

We next turn to a benchmark in which the relevant signal is expected to be much less concentrated
at the highest persistence scales. To evaluate whether topological summaries extracted from
persistence diagrams can detect differences in spatial interaction structure, we consider four
families of planar point processes exhibiting distinct behaviors (complete spatial randomness,
mild clustering, inhibition, and regularity) and sample point clouds on a square observation
window
\[
\mathcal W=[0,L]^2 \subset \R^2.
\]
For each family we generate \(30\) independent point clouds and aim to recover the generating
process via unsupervised clustering of their topological summaries.

Unlike the FDA simulations above, this setup is not designed to reward a strong bias toward the
most persistent topological features. For spatial point processes at comparable density, much of
the relevant information is carried by low- and medium-persistence structure, since local
interaction patterns first affect short- and intermediate-scale connectivity and loop formation.
Accordingly, summaries with too strong a high-persistence bias may become disadvantageous. This
motivates, in particular, the persistence-image choices used below: we employ weights that retain
sensitivity to lower-persistence features rather than amplifying only the largest lifetimes. 

Throughout, we construct for each model a point cloud of common target cardinality \(n\) in
\(\mathcal W\), which mitigates the confounding effect of varying sample size in persistence and
focuses the comparison on interaction structure. We then compute Vietoris--Rips persistence
diagrams in homological degrees \(0\) and \(1\), compute pairwise distances using the same family
of methods as above, and cluster the resulting samples using hierarchical clustering with average
linkage. Finally, we cut the dendrogram into four clusters and compare the obtained partition to
the ground-truth family labels. Background on spatial point processes may be found in
\citet{diggle2013spatiotemporal,illian2008spatialpp,mollerwaagepetersen2004spatialpp}.  

\paragraph{Point-process families}\label{subsec:pp_families}

\begin{itemize}
    \item[(i)] \textbf{Homogeneous Poisson (CSR).}
    As a baseline we use complete spatial randomness (CSR): conditional on a fixed sample size \(n\),
    this corresponds to i.i.d.\ points uniformly distributed on \(\mathcal W\). This model has no
    inter-point interaction and is a canonical null in spatial statistics
    \citep{diggle2013spatiotemporal,illian2008spatialpp}.

    \item[(ii)] \textbf{Thomas (Neyman--Scott) cluster process.}
    To model mild aggregation we use a Thomas process, a classical Neyman--Scott cluster model
    \citep{thomas1949ecology,mollerwaagepetersen2004spatialpp}. Parents are sampled uniformly in
    \(\mathcal W\); each parent generates a Poisson number of offspring; offspring are displaced
    from their parent by an isotropic Gaussian perturbation. The resulting pattern exhibits
    clustering at short-to-intermediate scales while remaining approximately Poisson at larger
    scales.

    \item[(iii)] \textbf{Mat\'ern type-II hard-core process.}
    To represent inhibition, we use a Mat\'ern type-II hard-core process obtained by thinning of a
    Poisson candidate set in \(\mathcal W\), equipped with i.i.d.\ marks
    \citep{matern1986spatialvariation,mollerwaagepetersen2004spatialpp}.
    Here a mark is an auxiliary random variable attached to each candidate point (in our case
    \(U\sim\mathrm{Unif}(0,1)\)), independent across points and independent of the candidate
    locations. The parameter \(r_{\mathrm{hc}}>0\) is the hard-core distance: it specifies an
    exclusion radius enforcing short-range repulsion. Concretely, given a candidate point \(y\) with
    mark \(U_y\), we retain \(y\) if its mark is the smallest among all candidates within distance
    \(r_{\mathrm{hc}}\), i.e.,
    \[
    U_y \;=\; \min\{U_z:\ z\in Y\cap B(y,r_{\mathrm{hc}})\}.
    \]
    This rule ensures that among any pair of candidates separated by less than
    \(r_{\mathrm{hc}}\), at most one can survive.

    \item[(iv)] \textbf{Jittered lattice.}
    Finally, we consider a jittered lattice on \(\mathcal W=[0,L]^2\). We partition \(\mathcal W\)
    into a regular \(k\times k\) grid of congruent square cells of side length \(L/k\) (with
    \(k^2\ge n\)), select \(n\) cells uniformly without replacement, and draw one point from each
    selected cell by applying a bounded random displacement (``jitter'') within the cell. This
    provides a controlled regularity baseline: the ``one point per cell'' constraint suppresses
    extremely small inter-point distances relative to CSR, mimicking a mild repulsive effect while
    being governed by a single interpretable jitter parameter.
\end{itemize}

\paragraph{Sampling schemes.}
For CSR and the jittered lattice design the construction yields \(|X|=n\) by definition. For the
Thomas and Mat\'ern type-II models, the underlying point process has random cardinality; we
therefore apply a simple conditioning procedure to obtain \(n\) points in \(\mathcal W\).

\begin{itemize}
\item[(i)] \textbf{CSR.}
Sample \(x_1,\dots,x_n \stackrel{\text{i.i.d.}}{\sim} \mathrm{Unif}(\mathcal W)\).

\item[(ii)] \textbf{Thomas cluster model (parameters: \(n_{\mathrm{par}},\lambda_{\mathrm{off}},\sigma\)).}
Sample \(n_{\mathrm{par}}\) parent locations
\(p_1,\dots,p_{n_{\mathrm{par}}}\stackrel{\text{i.i.d.}}{\sim}\mathrm{Unif}(\mathcal W)\).
Initialize an empty offspring set \(X=\emptyset\). While \(|X|<n\), choose an index
\(J\sim\mathrm{Unif}\{1,\dots,n_{\mathrm{par}}\}\), draw
\(K\sim\mathrm{Poisson}(\lambda_{\mathrm{off}})\), and generate candidate offspring
\[
p_J+\xi_\ell,\qquad \ell=1,\dots,K,\qquad \xi_\ell\stackrel{\text{i.i.d.}}{\sim}\mathcal N(0,\sigma^2 I_2).
\]
Candidates falling outside \(\mathcal W\) are discarded, and all remaining points are appended to
\(X\). If the same parent index \(J\) is selected multiple times, multiple independent offspring
batches are generated around \(p_J\) and accumulated in \(X\). Finally, if \(|X|>n\), we subsample
uniformly without replacement to obtain exactly \(n\) points.

\item[(iii)] \textbf{Mat\'ern type-II hard-core model (parameters: \(m,r_{\mathrm{hc}}\)).}
Generate a Poisson candidate set \(Y=\{y_1,\dots,y_M\}\subset\mathcal W\) with
\(M\sim\mathrm{Poisson}(m)\), and attach i.i.d.\ marks
\(U_{y_i}\stackrel{\text{i.i.d.}}{\sim}\mathrm{Unif}(0,1)\) independent of \(Y\).
Given a hard-core distance \(r_{\mathrm{hc}}>0\), retain a candidate \(y\in Y\) if it has the
smallest mark within its \(r_{\mathrm{hc}}\)-neighborhood, i.e.\
\(U_y=\min\{U_z:\ z\in Y\cap B(y,r_{\mathrm{hc}})\}\).
Denote by \(Y_{\mathrm{hc}}\) the retained set. If \(|Y_{\mathrm{hc}}|\ge n\) we select \(n\) points
uniformly without replacement from \(Y_{\mathrm{hc}}\). If \(|Y_{\mathrm{hc}}|< n\), we discard the
realization, update \(m\leftarrow 1.5\,m\), and resimulate until \(|Y_{\mathrm{hc}}|\ge n\).

\item[(iv)] \textbf{Jittered lattice (parameters: \(k,\rho_{\mathrm{jitt}}\)).}
Choose \(k=\lceil\sqrt{n}\rceil\) and partition \(\mathcal W\) into \(k^2\) congruent square cells
of side length \(L/k\). Select \(n\) distinct cells uniformly without replacement. From each
selected cell \(C\) draw one point by sampling a bounded random displacement around the cell
center \(c(C)\):
\[
x \;=\; c(C)+\delta,\qquad
\delta\sim \mathrm{Unif}\!\Big(\big[-\tfrac{\rho_{\mathrm{jitt}}}{2}\tfrac{L}{k},\,\tfrac{\rho_{\mathrm{jitt}}}{2}\tfrac{L}{k}\big]^2\Big),
\]
where \(\rho_{\mathrm{jitt}}\in[0,1]\) controls the jitter amplitude as a portion of the cell. Thus
\(\rho_{\mathrm{jitt}}=0\) yields deterministic cell centers, while \(\rho_{\mathrm{jitt}}=1\)
yields a uniform draw over the entire cell. By construction, this produces a point cloud with one
point per selected cell.
\end{itemize}

\paragraph{Parameter choices.}
Our goal is to construct a challenging clustering case study from degree-\(0\) and degree-\(1\)
persistence diagrams, in the sense that separation should not be dominated by extreme aggregation
or regularity. To make the interaction effects comparable across models, we scale all length
parameters using the typical spacing
\[
s \;:=\; \sqrt{\frac{|\mathcal W|}{n}} \;=\; \frac{L}{\sqrt n},
\]
which is the side length of a square of area \(|\mathcal W|/n\).
In our experiments we fix \(n=200\) and \(L=1000\). These values strike a practical balance
between (i) enough geometric complexity to produce informative topological features,
(ii) feasible computational cost for Vietoris--Rips persistence, and (iii) numerical stability, by
avoiding extremely small coordinate differences that may amplify floating-point effects.

\begin{itemize}
\item[(i)] \textbf{CSR.}
No additional parameters are required beyond the target cardinality \(n\) and the window size \(L\).

\item[(ii)] \textbf{Thomas cluster model.}
We fix the Gaussian cluster spread to
\[
\sigma \;=\; 0.9\,s,
\]
so that offspring clouds overlap at the spacing scale and clustering is intentionally mild. This
reduces the prevalence of extremely short edges and very early merges in degree-\(0\) persistent
homology compared to strongly clustered regimes (\(\sigma\ll s\)), thereby making CSR versus
clustering less trivial and leaving more discriminative signal to higher-order effects, including
loop formation and filling summarized by degree-\(1\) persistent homology. For the remaining
Thomas parameters we set
\[
n_{\mathrm{par}} \;=\; \big\lceil \sqrt{n}\,\big\rceil,
\qquad
\lambda_{\mathrm{off}} \;=\; \frac{n}{n_{\mathrm{par}}},
\]
so that the expected number of generated offspring before boundary rejection is approximately
\(n\) and neither the ``few large clusters'' nor the ``many tiny clusters'' regime dominates.

\item[(iii)] \textbf{Mat\'ern type-II hard-core model.}
We choose the hard-core distance as a moderate fraction of the spacing,
\[
r_{\mathrm{hc}} \;=\; 0.5\,s,
\]
so that inhibition is primarily visible through a depletion of very small inter-point distances,
without producing near-lattice configurations. For the Poisson candidate mean we set
\[
m \;=\; 3n,
\]
which oversamples candidates to compensate for Mat\'ern-II thinning; if a realization yields fewer
than \(n\) retained points, we increase \(m\) geometrically (by a fixed multiplicative factor of
\(1.5\)) and resimulate until at least \(n\) points are obtained.

\item[(iv)] \textbf{Jittered lattice.}
We set
\[
k \;=\; \lceil \sqrt{n}\,\rceil,
\qquad
\rho_{\mathrm{jitt}} \;=\; 0.75,
\]
where \(\rho_{\mathrm{jitt}}\in[0,1]\) controls the jitter amplitude as a proportion of the cell.
The choice \(\rho_{\mathrm{jitt}}=0.75\) avoids both extremes: nearly deterministic cell centers
(\(\rho_{\mathrm{jitt}}\approx 0\)) and nearly uniform-in-cell sampling
(\(\rho_{\mathrm{jitt}}\approx 1\)). Consequently, the ``one point per selected cell'' constraint
still suppresses the smallest inter-point distances relative to CSR, but the random displacement
blurs the grid sufficiently that the resulting persistence summaries overlap substantially with
those of CSR and of mild hard-core processes.
\end{itemize}

Finally, when building the Vietoris--Rips filtration we cap the diameter parameter as \(7.0\,s\).
To motivate this choice, note that \(s^2\) is the area available per point in an \(n\)-point
configuration on \(\mathcal W\). Equivalently, \(s\) is the grid spacing of the regular
\(l\times l\) lattice with \(l^2=n\) points filling \(\mathcal W\). Thus \(s\) provides a
canonical unit for inter-point distances at the given density. In the Vietoris--Rips complex,
edges appear at scale \(r\) between pairs of points at Euclidean distance at most \(r\); taking
\(r\) to be a moderate multiple of \(s\) therefore probes neighborhoods containing a controlled
(density-dependent) number of points. In particular, for a homogeneous configuration of intensity
\(n/|\mathcal W|\), the expected number of neighbors of a typical point within radius \(r\) is
approximately
\[
\frac{n}{|\mathcal W|}\,\pi r^2 \;=\; \pi\Big(\frac{r}{s}\Big)^2,
\]
ignoring boundary effects. Hence setting \(r=7s\) yields an expected number of neighbors of about
\(\pi\cdot 7^2\approx 154\) in the underlying proximity graph, amounting to roughly \(75\%\) of
the sample for \(n=200\). 

\paragraph{Distances and topological-summary parameters.}
For \(\POT_1\) no parameter is required. For \(\mathrm{SW}\), we approximate the
metric using \(50\) random lines, which is the default choice in the \texttt{GUDHI}
implementation. PSphs are evaluated on a polar grid on \(\mathbb S^2\) with \(100\times 200\)
samples (latitude \(\times\) longitude). For PLs, we retain all landscape functions and compute
pairwise distances exactly using the \(L^2\)-norm. For PSpls, we use a \(20\times 20\) grid and
\(100\) iterations, with a stopping tolerance \(10^{-10}\) for the iterative procedure; these
values follow the supervised case studies reported in \citet{dong2024persistence}.

For PIs, we set the pixel size to \(s/100\) and the Gaussian kernel bandwidth to \(s/10\), where
\(s=L/\sqrt n\) is the typical spacing of the point clouds. In this PP scenario, where low- and
medium-persistence features are expected to be informative, we use the bounded weight
\[
w(x,y)=\frac{2}{\pi}\arctan(y-x).
\]
We also tested the linear weight \(w(x,y)=y-x\) and observed similar results. In addition, for
\(H_0\) diagrams only, when computing persistence images we add an artificial point on the
diagonal at coordinates \((s/10,s/10)\), so that the birth coordinates are not all identically
equal to \(0\); otherwise the resulting PI would degenerate to a one-dimensional object rather
than a genuine two-dimensional image. These PI parameters were chosen by testing a small set of
candidates on smaller datasets generated from the same models and then freezing the configuration
for the clustering experiments. 

\paragraph{Results.}
\Cref{tab:pp_unsup,tab:pp_corr} report, respectively, clustering performance on the PP benchmark
and the empirical agreement of each distance with \(\POT_1\).

In terms of Rand index (\Cref{tab:pp_unsup}), \(\POT_1\), SW, and PSph perform comparably in
\(H_0\), with means \(0.859\), \(0.859\), and \(0.858\), respectively. In \(H_0\), however, PI
is the clear winner, reaching a mean Rand index of \(0.950\), well above all other methods. In
\(H_1\), SW attains the best mean performance (\(0.873\)), with \(\POT_1\) very close behind
(\(0.872\)) and with overlapping confidence intervals; PSph remains very competitive (\(0.861\)),
whereas PL and PSpl are substantially worse in both degrees. 

As in the other benchmarks, there is no particular reason to expect SW to outperform
\(\POT_1\) systematically; any small advantage is plausibly due to the numerical approximation
with \(50\) random lines together with the downstream clustering pipeline. We also stress that,
unlike the linearizations below, SW remains a distance rather than a vectorization. 

Among the distances induced by linearized representations, PSph is again the most faithful to
\(\POT_1\); see \Cref{tab:pp_corr}. Its correlation with \(\POT_1\) is \(0.935\),
substantially above PL (\(0.822\)), PI (\(0.886\)), and PSpl (\(0.527\)). This is consistent
with the clustering results: PSph stays close to \(\POT_1\) in both \(H_0\) and \(H_1\), whereas
PSpl performs poorly both in clustering and in correlation, suggesting that its induced geometry
is much less aligned with that of \(\POT_1\) in this regime. 

The behavior of PL and PI is particularly informative when compared with the FDA simulations
(i)--(iii). There, summaries with a strong bias toward high persistence, especially PL, were
favored because the discriminative signal was largely carried by the largest oscillations. Here
the situation is different. Since low- and medium-persistence features carry relevant information
about local spatial interaction, summaries with a strong high-persistence bias can discard too
much signal. Coherently with this intuition, PL performs poorly in both homological degrees, with
mean Rand index around \(0.60\) in \(H_0\) and \(0.63\) in \(H_1\). By contrast, PI performs
extremely well in \(H_0\), plausibly because its milder and tunable bias is better matched to the
relevant persistence scales in this regime. 

Finally, the very high correlation of SW with \(\POT_1\) is expected: on persistence diagrams
with uniformly bounded cardinality, SW is bi-Lipschitz equivalent to \(\POT_1\), and should
therefore track it closely up to multiplicative distortion. More broadly, this PP study
complements the FDA one by showing that the usefulness of a given topological summary depends
strongly on which persistence scales carry the relevant information.

\begin{table}[t]
\centering
\renewcommand{\arraystretch}{1.2}
\resizebox{\linewidth}{!}{
\begin{tabular}{|l|c|c|c|c|c|c|}
\hline
 & POT$_1$ & SW & PSph & PL & PI & PSpl \\
\hline
$H_0$ &
$0.859 \pm 0.010$ &
$0.859 \pm 0.010$ &
$0.858 \pm 0.010$ &
$0.596 \pm 0.052$ &
$\mathbf{0.950 \pm 0.058}$ &
$0.283 \pm 0.015$ \\
\hline
$H_1$ &
$0.872 \pm 0.006^{\dagger}$ &
$\mathbf{0.873 \pm 0.009}$ &
$0.861 \pm 0.013$ &
$0.630 \pm 0.211$ &
$0.754 \pm 0.019$ &
$0.285 \pm 0.018$ \\
\hline
\end{tabular}
}
\caption{\textbf{PP unsupervised.} Mean $\pm$ standard deviation of Rand index over $100$ runs.
Bold denotes the best mean in each row. A dagger $^{\dagger}$ marks methods whose 95\% confidence
interval for the mean (with $n=100$ runs) overlaps with that of the best method.}
\label{tab:pp_unsup}
\end{table}

\begin{table}[t]
\centering
\renewcommand{\arraystretch}{1.2}
\begin{tabular}{|l|c|}
\hline
Method & Corr.\ with POT$_1$ \\
\hline
SW   & $0.996 \pm 0.004$ \\
PSph & $0.935 \pm 0.055$ \\
PL   & $0.822 \pm 0.071$ \\
PI   & $0.886 \pm 0.092$ \\
PSpl & $0.527 \pm 0.102$ \\
\hline
\end{tabular}
\caption{\textbf{PP correlation.} Mean $\pm$ standard deviation of the per-iteration,
per-homological-degree correlation between POT$_1$ and each competing distance, computed over
repeated runs. Since $H_0$ and $H_1$ are computed on the same data in each iteration, the two
correlations within an iteration are not independent; accordingly, we only report descriptive
statistics.}
\label{tab:pp_corr}
\end{table}

\begin{figure}[t]
\centering
\newcommand{\ppimg}[1]{\includegraphics[width=\linewidth]{#1}}

\begin{subfigure}[t]{0.24\textwidth}
  \centering
  \ppimg{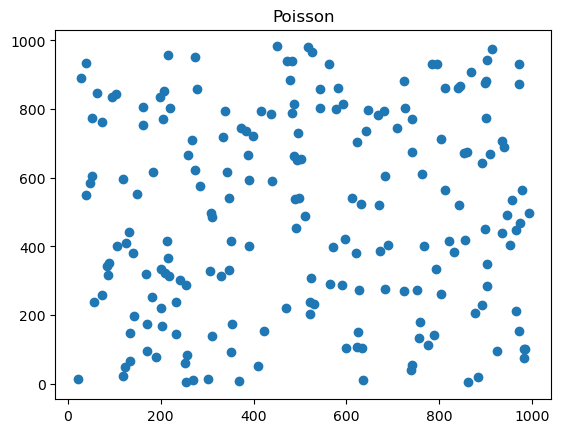}
  \caption{Poisson: sample}
\end{subfigure}\hfill
\begin{subfigure}[t]{0.24\textwidth}
  \centering
  \ppimg{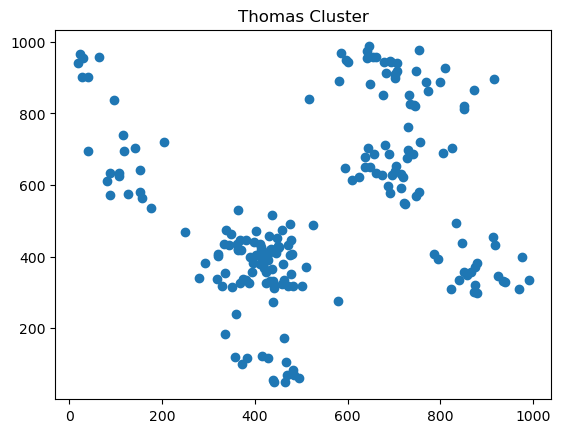}
  \caption{Thomas: sample}
\end{subfigure}\hfill
\begin{subfigure}[t]{0.24\textwidth}
  \centering
  \ppimg{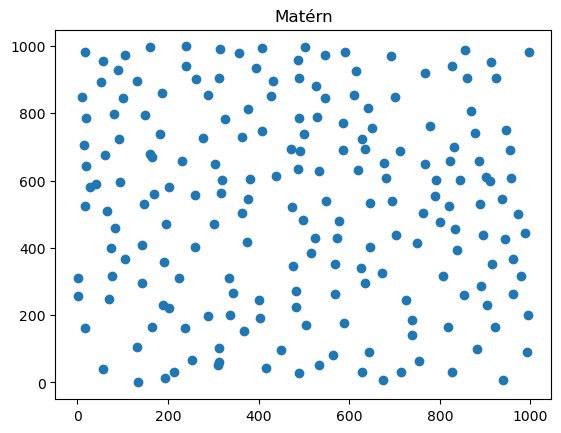}
  \caption{Mat\'ern: sample}
\end{subfigure}\hfill
\begin{subfigure}[t]{0.24\textwidth}
  \centering
  \ppimg{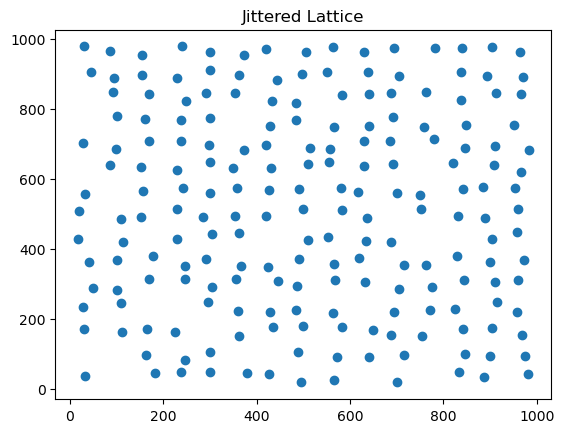}
  \caption{Lattice: sample}
\end{subfigure}

\vspace{2mm}

\begin{subfigure}[t]{0.24\textwidth}
  \centering
  \ppimg{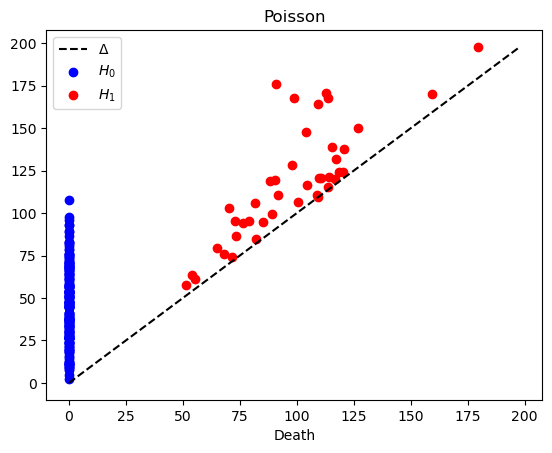}
  \caption{Poisson: diagrams}
\end{subfigure}\hfill
\begin{subfigure}[t]{0.24\textwidth}
  \centering
  \ppimg{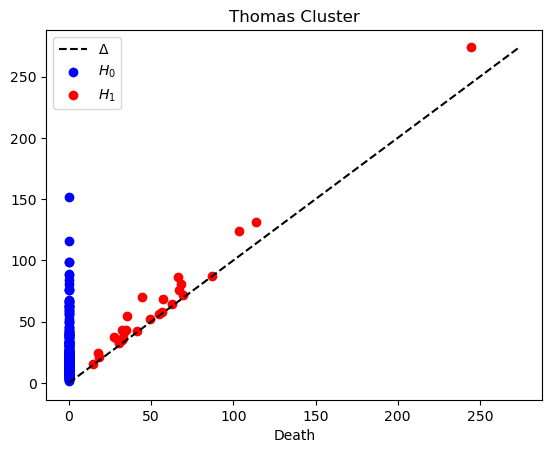}
  \caption{Thomas: diagrams}
\end{subfigure}\hfill
\begin{subfigure}[t]{0.24\textwidth}
  \centering
  \ppimg{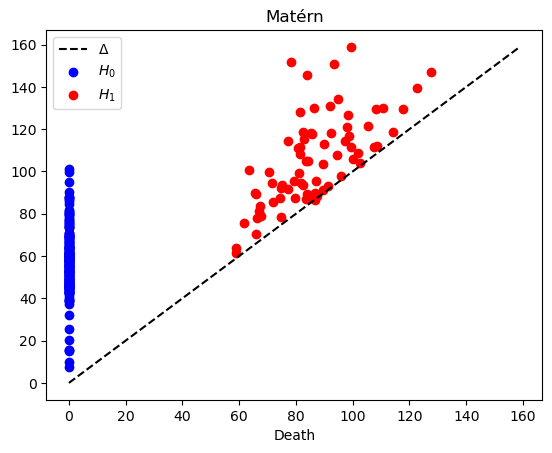}
  \caption{Mat\'ern: diagrams}
\end{subfigure}\hfill
\begin{subfigure}[t]{0.24\textwidth}
  \centering
  \ppimg{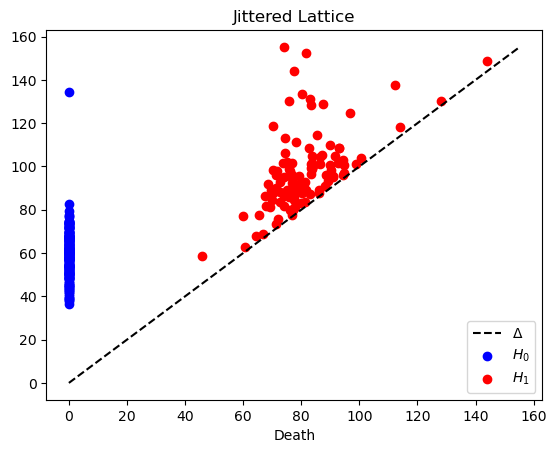}
  \caption{Lattice: diagrams}
\end{subfigure}

\vspace{-2mm}
\caption{\textbf{PP simulation.} Top row: one simulated point pattern for each point process. Bottom row: corresponding persistence diagrams in degrees $H_0$ (blue) and $H_1$ (red). Ordering: Poisson, Thomas, Mat\'ern, jittered lattice.}
\label{fig:pp_simulation}
\vspace{-2mm}
\end{figure}

\section{Supervised Case Studies}
\label{sec:supervised}

We assess PSph on a collection of supervised regression and classification problems, comparing it with PI, PL, PSpl, SWK, and, whenever the sample size is sufficiently large, PersLay. For PSph, PI, PSpl, and PL we train random-forest classifiers or regressors, whereas SWK is paired with SVMs. Performance is measured through $R^2$ in regression and accuracy in classification.

All case studies in this section are inherited from \citet{pegoraro2025persistence}, to which we refer for full details of the original benchmark design and implementation. In the present paper, we rerun only the PSph and PI pipelines: PSph is recomputed using the new definition introduced here, while PI is recomputed using improved parameter choices. The remaining baselines are reported unchanged from \citet{pegoraro2025persistence}. This allows for a direct comparison between the original weighted persistence-sphere construction, denoted here by PSph*, and the new representation PSph.

\subsubsection{Datasets}

We briefly summarize the supervised benchmarks inherited from \citet{pegoraro2025persistence}. Unless otherwise stated, we use the same train--test splits and cross-validation protocols as in that paper.

\paragraph{\virgolette{Eyeglasses} regression.}
This synthetic regression task is built from the \texttt{eyeglasses} generator in \texttt{scikit-tda} \citep{scikittda2019}. One radius is fixed at $20$, while the second is sampled from a Gaussian distribution with mean $10$ and standard deviation $2.5$, and serves as the regression target. We generate $2000$ point clouds, compute their one-dimensional Vietoris--Rips persistence diagrams, and evaluate performance over $5$ independent repetitions, each using a $70\%$--$30\%$ train--test split and threefold cross-validation on the training set.

\paragraph{Functional datasets from \texttt{scikit-fda}.}
We also consider the FDA benchmarks \virgolette{Tecator}, \virgolette{$\mathrm{NO}_x$}, and \virgolette{Growth}, following the preprocessing used in \citet{pegoraro2025persistence}. In all three cases, zero-dimensional persistence diagrams are extracted from sublevel-set filtrations of the observed curves, and we use a $70\%$--$30\%$ train--test split with threefold cross-validation. For \virgolette{Tecator}, the task is regression of fat content from derivatives of absorbance curves. For \virgolette{$\mathrm{NO}_x$}, the task is to classify weekdays versus weekends from daily emission curves. For \virgolette{Growth}, the task is to classify sex from derivatives of height trajectories.

\paragraph{Classification benchmarks from \citet{bandiziol2024persistence}.}
The datasets \virgolette{DYN SYS}, \virgolette{ENZYMES JACC}, \virgolette{POWER}, and \virgolette{SHREC14} are the same diagram-based classification benchmarks used in \citet{pegoraro2025persistence}, originally taken from \citet{bandiziol2024persistence}. We retain the same $70\%$--$30\%$ train--test split and threefold cross-validation protocol. These datasets cover a range of modalities, including point clouds, graphs, time series, and 3D shapes.

\paragraph{\virgolette{Human Poses} and \virgolette{McGill 3D Shapes}.}
The final two benchmarks are again inherited from \citet{pegoraro2025persistence}. For \virgolette{Human Poses}, persistence diagrams are obtained from sublevel sets of the height function; for \virgolette{McGill 3D Shapes}, they are obtained from an HKS-based filtration. Because each class contains only $10$ samples, both datasets are evaluated with a fixed $80\%$--$20\%$ train--test split.

\begin{figure}
\centering
	\begin{subfigure}[c]{0.25\textwidth}
    	\includegraphics[width = \textwidth]{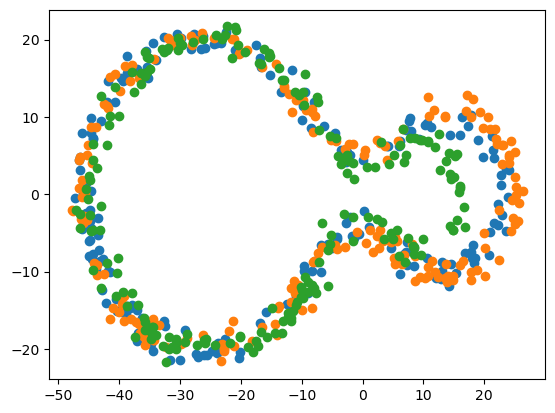}
		\captionsetup{singlelinecheck=off, margin={0.05cm, 0.05cm}}
    	\caption{Three synthetic point clouds from the \virgolette{Eyeglasses} experiment.}
    \end{subfigure}
    	\begin{subfigure}[c]{0.25\textwidth}
    	\centering
    	\includegraphics[width = \textwidth]{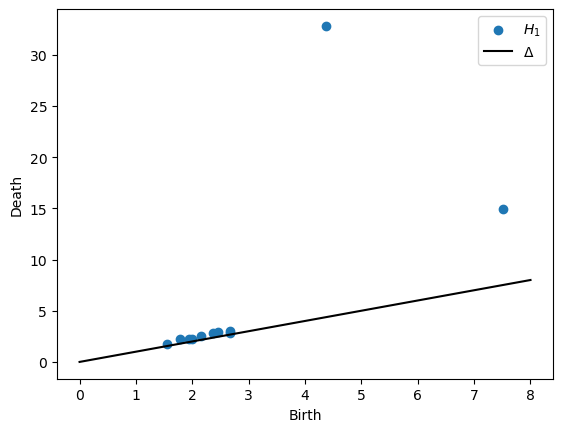}
		\captionsetup{singlelinecheck=off, margin={0.05cm, 0.05cm}}
    	\caption{A persistence diagram from the \virgolette{Eyeglasses} regression problem.}
    	\label{fig:PD}
    \end{subfigure}
    	\begin{subfigure}[c]{0.25\textwidth}
    	\centering
    	\includegraphics[width = \textwidth]{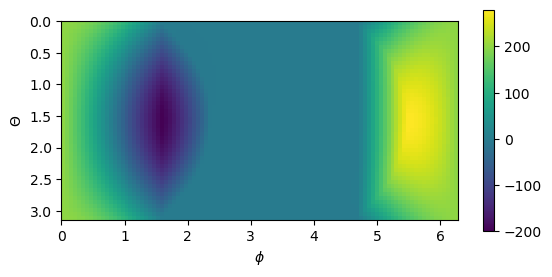}
		\captionsetup{singlelinecheck=off, margin={0.05cm, 0.05cm}}
    	\caption{The corresponding PSph representation, shown in polar coordinates.}
    \end{subfigure}

	\begin{subfigure}[c]{0.25\textwidth}
		\centering
		\includegraphics[width = \textwidth]{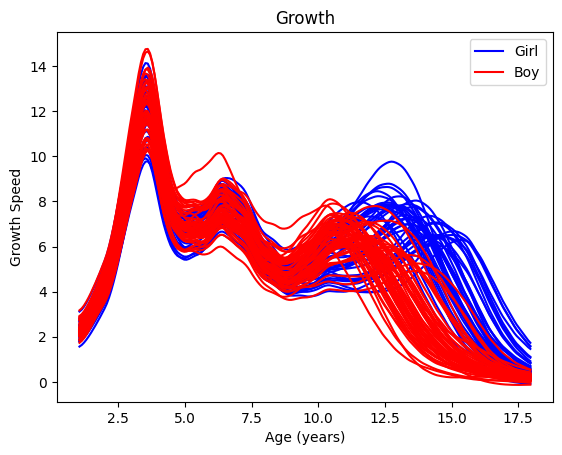}
		\captionsetup{singlelinecheck=off, margin={0.05cm, 0.05cm}}
		\caption{Derivatives of the curves in the \virgolette{Growth} dataset, colored by class.}
		\label{fig:growth}
	\end{subfigure}
	\begin{subfigure}[c]{0.25\textwidth}
		\centering
		\includegraphics[width = \textwidth]{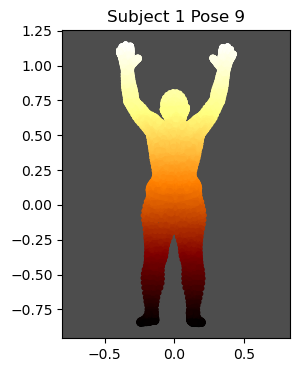}
		\captionsetup{singlelinecheck=off, margin={0.05cm, 0.05cm}}
		\caption{One example from the \virgolette{Human Poses} dataset.}
		\label{fig:human}
	\end{subfigure}
	\begin{subfigure}[c]{0.25\textwidth}
		\centering
		\includegraphics[width = \textwidth]{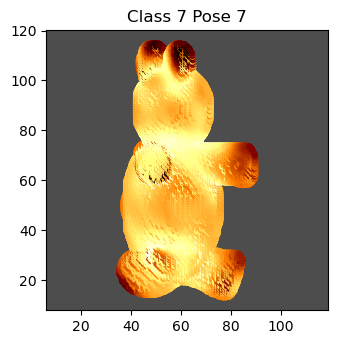}
		\captionsetup{singlelinecheck=off, margin={0.05cm, 0.05cm}}
		\caption{One example from the \virgolette{McGill 3D Shapes} dataset.}
		\label{fig:mcgill}
	\end{subfigure}
    \caption{Examples of data, persistence diagrams, and PSph representations appearing in the
    supervised experiments of \Cref{sec:supervised}.}
    \label{fig:experiments}
\end{figure}

\subsubsection{Parameter Details}
\label{sec:implementation_details}

We summarize here the hyperparameters used for the vectorization and kernel baselines; PersLay architecture details are deferred to \citet{pegoraro2025persistence}.

For SWK we use SVM pipelines with precomputed kernels and regularization parameter
\[
C \in \{10^{-3},10^{-2},10^{-1},1,10,10^{2},10^{3},10^{4}\}.
\]
For random forests, the number of trees is selected in $\{100,200\}$. All models are implemented with \texttt{scikit-learn} \citep{scikit-learn}.

\paragraph{Parameters for topological summaries.}
We now detail the hyperparameters used by the different vectorization methods. As in
\citet{pegoraro2025persistence}, for PIs, PSpl, and PLs, support or range parameters are
determined by inspecting the full dataset, independently of the train--test split. This introduces
a mild inconsistency, although in practice it could be avoided by selecting sufficiently
conservative bounds from the training sample alone. In contrast, PSphs are defined on a fixed
compact domain, so no analogous preprocessing choice is required.

\begin{itemize}
    \item \textbf{PSph:} PSphs are functions on $\mathbb{S}^2$, expressed in spherical coordinates
    and expanded in spherical harmonics \citep{muller2006spherical} using \texttt{pyshtools}
    \citep{wieczorek2018shtools}. This yields an orthonormal feature representation compatible with
    \texttt{scikit-learn}. Using a Driscoll--Healy grid \citep{driscoll1994computing} with
    $2N_\theta$ latitudinal and $4N_\theta$ longitudinal nodes, the resulting feature dimension is
    $N_\theta^2/2$. We cross-validate
    \[
    2N_\theta \in \{30,40,50,60,70\}.
    \]

    \item \textbf{PI:} using \texttt{persim} from \texttt{scikit-tda}, we enclose all diagrams in
a birth--persistence rectangle and set \texttt{pixel\_size} by dividing the shortest side of
this rectangle by a prescribed number \(N_{\mathrm{pix}}\) of pixels per side. With the
default Gaussian kernel, we set
\[
\sigma=\texttt{pixel\_size}/N_\sigma,
\qquad
N_\sigma \in \{1,0.5,0.1,0.05,0.01\},
\]
so that smaller values of \(N_\sigma\) correspond to broader kernels. We also tune the
persistence exponent in \texttt{weight\_params} over
\[
n \in \{1,2,4,8\}.
\]
The range of \(N_{\mathrm{pix}}\) depends on the dataset family:
\[
N_{\mathrm{pix}}\in\{10,100,500\}
\]
for the datasets inherited from \citet{bandiziol2024persistence} and for
\virgolette{Human Poses},
\[
N_{\mathrm{pix}}\in\{10,100\}
\]
for the \virgolette{Eyeglasses} and the FDA datasets \virgolette{Tecator}, \virgolette{Growth}, and \virgolette{NOx}, and
\[
N_{\mathrm{pix}}\in\{10,50\}
\]
for \virgolette{McGill 3D Shapes}. 

    \item \textbf{PL:} persistence landscapes are sampled on a common grid of $5000$ points and
    then concatenated, with no additional hyperparameters.

    \item \textbf{PSpl:} following \citet{dong2024persistence}, we use spline grids of size $h^2$
    with
    \[
    h \in \{5,10,20,40,50\},
    \]
    and otherwise keep the same parameter ranges as in \citet{pegoraro2025persistence}.

    \item \textbf{SWK:} we use the sliced Wasserstein kernel from \texttt{gudhi} \citep{gudhi},
    fixing the number of directions to $M=100$ and tuning the kernel bandwidth over
    \[
    \sigma \in \{10^{-5},10^{-4},10^{-3},10^{-2},10^{-1},1,10\}.
    \]
\end{itemize}

\begin{table}[t]
\centering
\renewcommand{\arraystretch}{1.2}
\resizebox{\linewidth}{!}{
\begin{tabular}{llllllll}
\hline
\multicolumn{1}{|l|}{} &
\multicolumn{1}{l|}{PSph} &
\multicolumn{1}{l|}{PSph*} &
\multicolumn{1}{l|}{PI} &
\multicolumn{1}{l|}{PL} &
\multicolumn{1}{l|}{PSpl} &
\multicolumn{1}{l|}{PersLay} &
\multicolumn{1}{l|}{SWK} \\
\hline
\textbf{Regression} & & & & & & & \\ \hline
\multicolumn{1}{|l|}{Eyeglasses} &
\multicolumn{1}{l|}{$0.960 \pm 0.004^{\dagger}$} &
\multicolumn{1}{l|}{$0.966 \pm 0.003^{\dagger}$} &
\multicolumn{1}{l|}{$0.969 \pm 0.006^{\dagger}$} &
\multicolumn{1}{l|}{$0.955 \pm 0.018^{\dagger}$} &
\multicolumn{1}{l|}{$\mathbf{0.971 \pm 0.011^{\dagger}}$} &
\multicolumn{1}{l|}{$0.248 \pm 0.031$} &
\multicolumn{1}{l|}{$\mathbf{0.971 \pm 0.003^{\dagger}}$} \\ \hline
\multicolumn{1}{|l|}{Tecator} &
\multicolumn{1}{l|}{$\mathbf{0.973 \pm 0.007^{\dagger}}$} &
\multicolumn{1}{l|}{$0.969 \pm 0.009^{\dagger}$} &
\multicolumn{1}{l|}{$0.940 \pm 0.018$} &
\multicolumn{1}{l|}{$0.954 \pm 0.011$} &
\multicolumn{1}{l|}{$\mathbf{0.973 \pm 0.009^{\dagger}}$} &
\multicolumn{1}{l|}{$0.895 \pm 0.029$} &
\multicolumn{1}{l|}{$0.953 \pm 0.010$} \\ \hline
\textbf{Classification} & & & & & & & \\ \hline
\multicolumn{1}{|l|}{Growth} &
\multicolumn{1}{l|}{$\mathbf{0.900 \pm 0.047^{\dagger}}$} &
\multicolumn{1}{l|}{$0.850 \pm 0.052^{\dagger}$} &
\multicolumn{1}{l|}{$0.836 \pm 0.056$} &
\multicolumn{1}{l|}{$0.768 \pm 0.060$} &
\multicolumn{1}{l|}{$0.836 \pm 0.043$} &
\multicolumn{1}{l|}{$0.807 \pm 0.043$} &
\multicolumn{1}{l|}{$0.768 \pm 0.058$} \\ \hline
\multicolumn{1}{|l|}{NOx} &
\multicolumn{1}{l|}{$0.863 \pm 0.044^{\dagger}$} &
\multicolumn{1}{l|}{$\mathbf{0.869 \pm 0.041^{\dagger}}$} &
\multicolumn{1}{l|}{$0.803 \pm 0.043$} &
\multicolumn{1}{l|}{$0.789 \pm 0.062$} &
\multicolumn{1}{l|}{$0.860 \pm 0.058^{\dagger}$} &
\multicolumn{1}{l|}{$0.717 \pm 0.078$} &
\multicolumn{1}{l|}{$0.840 \pm 0.055^{\dagger}$} \\ \hline
\multicolumn{1}{|l|}{DYN\_SYS} &
\multicolumn{1}{l|}{$0.809 \pm 0.020$} &
\multicolumn{1}{l|}{$0.829 \pm 0.028^{\dagger}$} &
\multicolumn{1}{l|}{$0.823 \pm 0.028^{\dagger}$} &
\multicolumn{1}{l|}{$\mathbf{0.840 \pm 0.024^{\dagger}}$} &
\multicolumn{1}{l|}{$0.791 \pm 0.029$} &
\multicolumn{1}{l|}{$0.696 \pm 0.044$} &
\multicolumn{1}{l|}{$0.828 \pm 0.028^{\dagger}$} \\ \hline
\multicolumn{1}{|l|}{ENZYMES\_JACC} &
\multicolumn{1}{l|}{$\mathbf{0.391 \pm 0.025^{\dagger}}$} &
\multicolumn{1}{l|}{$0.349 \pm 0.036$} &
\multicolumn{1}{l|}{$0.325 \pm 0.061$} &
\multicolumn{1}{l|}{$0.377 \pm 0.032^{\dagger}$} &
\multicolumn{1}{l|}{$0.362 \pm 0.029^{\dagger}$} &
\multicolumn{1}{l|}{$0.243 \pm 0.023$} &
\multicolumn{1}{l|}{$0.283 \pm 0.055$} \\ \hline
\multicolumn{1}{|l|}{POWER} &
\multicolumn{1}{l|}{$\mathbf{0.771 \pm 0.017^{\dagger}}$} &
\multicolumn{1}{l|}{$0.769 \pm 0.021^{\dagger}$} &
\multicolumn{1}{l|}{$0.740 \pm 0.017$} &
\multicolumn{1}{l|}{$0.756 \pm 0.018^{\dagger}$} &
\multicolumn{1}{l|}{$0.734 \pm 0.021$} &
\multicolumn{1}{l|}{$0.725 \pm 0.038$} &
\multicolumn{1}{l|}{$0.767 \pm 0.150^{\dagger}$} \\ \hline
\multicolumn{1}{|l|}{SHREC14} &
\multicolumn{1}{l|}{$0.920 \pm 0.020^{\dagger}$} &
\multicolumn{1}{l|}{$0.931 \pm 0.022^{\dagger}$} &
\multicolumn{1}{l|}{$0.938 \pm 0.019^{\dagger}$} &
\multicolumn{1}{l|}{$\mathbf{0.943 \pm 0.024^{\dagger}}$} &
\multicolumn{1}{l|}{$0.942 \pm 0.015^{\dagger}$} &
\multicolumn{1}{l|}{$0.879 \pm 0.018$} &
\multicolumn{1}{l|}{$0.886 \pm 0.092^{\dagger}$} \\ \hline
\multicolumn{1}{|l|}{Human Poses} &
\multicolumn{1}{l|}{$0.540 \pm 0.073$} &
\multicolumn{1}{l|}{$\mathbf{0.640 \pm 0.077^{\dagger}}$} &
\multicolumn{1}{l|}{$0.515 \pm 0.084$} &
\multicolumn{1}{l|}{$0.405 \pm 0.106$} &
\multicolumn{1}{l|}{$0.560 \pm 0.094^{\dagger}$} &
\multicolumn{1}{l|}{-} &
\multicolumn{1}{l|}{$0.345 \pm 0.082$} \\ \hline
\multicolumn{1}{|l|}{McGill 3D Shapes} &
\multicolumn{1}{l|}{$\mathbf{0.689 \pm 0.075^{\dagger}}$} &
\multicolumn{1}{l|}{$0.544 \pm 0.085$} &
\multicolumn{1}{l|}{$0.667 \pm 0.056^{\dagger}$} &
\multicolumn{1}{l|}{$0.678 \pm 0.102^{\dagger}$} &
\multicolumn{1}{l|}{$0.656 \pm 0.108^{\dagger}$} &
\multicolumn{1}{l|}{-} &
\multicolumn{1}{l|}{$0.567 \pm 0.130^{\dagger}$} \\ \hline
\end{tabular}
}
\caption{\textbf{Supervised case studies.} We report average $R^2$ for regression and average
accuracy for classification, across $5$ runs for \virgolette{Eyeglasses} and $10$ runs for the
remaining tasks. Values are reported as mean $\pm$ standard deviation. PSph* denotes the weighted
persistence-sphere construction introduced in \citet{pegoraro2025persistence}. All benchmarks are
taken from \citet{pegoraro2025persistence}; in the present paper, only the PSph and PI columns are
recomputed, while the remaining columns are reported from that work. Bold entries indicate the
best-performing method in each row. A dagger $^{\dagger}$ marks methods whose 95\% confidence
interval, computed using the appropriate number of repetitions for the corresponding row, overlaps
with that of the best method.}
\label{table:results}
\end{table}

\subsection{Results}
\label{sec:results}

The supervised results in \Cref{table:results} show that PSph remains highly competitive across a broad range of tasks. Relative to the weighted construction PSph* from \citet{pegoraro2025persistence}, performance is mostly consistent and in several datasets slightly improved. In particular, PSph attains the best mean on \virgolette{Tecator}, \virgolette{Growth}, \virgolette{POWER}, and \virgolette{McGill 3D Shapes}, and remains within overlapping confidence intervals of the best method on several others.

The differences between PSph and PSph* are dataset-dependent. PSph improves markedly on \virgolette{McGill 3D Shapes}, and also slightly on \virgolette{Tecator}, \virgolette{Growth}, \virgolette{ENZYMES JACC}, and \virgolette{POWER}. On the other hand, PSph* remains stronger on \virgolette{Human Poses}, and has a mild advantage on \virgolette{DYN SYS}, \virgolette{NOx}, and \virgolette{SHREC14}. A plausible explanation for the gap on \virgolette{Human Poses} is the very small sample size: in such regimes, explicit reweighting may act as a useful denoising bias before the downstream learning stage. This does not preclude combining reweighted diagrams with the new PSph representation; it only means that such a choice is no longer built into the representation itself.

PIs are competitive on several datasets. These include \virgolette{Eyeglasses}, \virgolette{DYN SYS}, \virgolette{SHREC14}, and \virgolette{McGill 3D Shapes}, although their performance is less uniform than that of PSph-type summaries.

More broadly, the results confirm the main supervised finding already observed in \citet{pegoraro2025persistence}: PSph-type representations are consistently among the strongest performers across tasks. Persistence splines also perform very well in most supervised settings, whereas PersLay was likely penalized by the relatively small sample sizes of many benchmarks. Finally, none of the considered methods is uniformly poor, which reinforces the broader message of the paper: different summaries encode different geometric biases, and their effectiveness depends on the regime in which they are deployed.

\section{Discussion}
\label{sec:discussion}

In this paper we introduced a refined definition of persistence spheres for integrable measures on the upper half-plane, proved stability with respect to $\POT_1$, established continuity of the inverse on the image, and showed how these results extend naturally to a Hilbert-space setting relevant for applications. Conceptually, the key novelty is that the representation incorporates the deletion-to-diagonal mechanism of partial optimal transport directly through signed diagonal augmentation, rather than through ad hoc persistence-dependent reweighting. Empirically, the revised construction also leads to improved performance across the supervised benchmarks considered here.

These results suggest several directions for further work. At a conceptual level, the signed diagonal augmentation used here appears relevant beyond persistence spheres themselves. As discussed in \Cref{sec:augmentation_strategy}, it provides a general way of reconciling linear summaries of persistence diagrams with the deletion-to-diagonal mechanism of $\POT_1$, without imposing ad hoc vanishing weights near the diagonal. A natural question is therefore to understand how far this principle extends: for which classes of integral summaries, kernel embeddings, or smoothed diagram representations can augmentation yield not only stability, but also some form of inverse continuity or local metric faithfulness? In particular, it would be interesting to identify feature families rich enough to replicate, beyond the ReLU ridge setting considered here, the duality-and-approximation mechanism underlying the local inverse bounds of \Cref{prop:compact_core_holder_siegel}.

A second direction concerns statistical modeling directly on persistence-valued outputs. For instance, one may consider a parametric or nonparametric family of predictors \(h_\theta(x)\) taking values in a class of persistence measures, and estimate \(\theta\) by minimizing an empirical risk of the form
\[
\frac1n\sum_{i=1}^n \ell\!\left(S(h_\theta(x_i)),\,S(\mu_i)\right),
\]
where \(\mu_i\) are observed persistence measures and \(\ell\) is an \(L^2(\Ss^2)\), uniform, or kernel-induced discrepancy. Because the representation is stable and bi-continuous on suitable classes, such procedures would allow one to optimize in a linear/Hilbert space while still controlling the induced error in $\POT_1$ on the underlying persistence objects.

A third direction concerns inversion. Given a function \(f\in C(\Ss^2)\), for instance produced by a regression model or by averaging in sphere space, can one recover a persistence measure \(\widehat\mu\) such that \(S(\widehat\mu)\approx f\)? This matters mainly for interpretation: passing back from a learned or averaged sphere representation to a diagram would allow one to read the result again in terms of persistence pairs rather than only as a function on \(\Ss^2\). Even in the discrete case this raises nontrivial questions, including uniqueness, stability under approximation error, and effective reconstruction from finitely sampled values of \(f\).

A natural first approach would be to optimize over discrete candidate measures
\[
\widehat\mu=\sum_{j=1}^m w_j\delta_{p_j},
\]
and fit the locations \(p_j\) and weights \(w_j\) so that \(S(\widehat\mu)\) matches the observed function. By definition,
\[
S(\widehat\mu)(v)
=
\sum_{j=1}^m w_j\Big[
\relu\!\big(\langle v,(1,p_j)\rangle\big)
-
\relu\!\big(\langle v,(1,\pi_\Delta(p_j))\rangle\big)
\Big],
\]
so the inversion problem can be viewed as training a shallow ReLU model with a highly structured architecture, whose parameters are constrained by the geometry of persistence diagrams. This suggests exploiting the large body of optimization methods and software developed for neural networks. In this way, inversion would provide not only a computational tool, but also a
concrete mechanism for translating learned or averaged sphere representations back into diagrams.

More broadly, the present work points to a larger problem in topological machine learning: understanding which linearizations of persistence are merely stable, which are geometrically faithful, and which structural modifications, such as diagonal augmentation, can help bridge the gap between the two.

\appendix
\renewcommand{\theHsection}{appendix.\Alph{section}}
\section{Sobolev and ReLU Variation-Space Details}\label{app:sobolev_relu}

This appendix collects the analytic notation and the Sobolev-to-ReLU estimate
used in the proof of \Cref{prop:compact_core_holder_siegel}. We first specify
the Sobolev restriction norms, and then compare the closed-hull variation norm
of \citet{mao2024approximation} with the exact signed-measure representation
needed to integrate against arbitrary finite signed measures.

We first fix the Sobolev convention needed below. Following
\citet[eqs.~(1.5) and (1.8)]{mao2024approximation}, for $s\ge0$ we set
\begin{equation}\label{eq:whole_space_sobolev}
W^s(L^2(\R^2))
:=
\left\{H\in L^2(\R^2):
\int_{\R^2}(1+\|\xi\|_2)^{2s}|\widehat H(\xi)|^2\,d\xi<\infty
\right\},
\end{equation}
with norm
\[
\|H\|_{W^s(L^2(\R^2))}^2
:=
\int_{\R^2}(1+\|\xi\|_2)^{2s}|\widehat H(\xi)|^2\,d\xi,
\]
where, as in \citet[eq.~(1.7)]{mao2024approximation},
\[
\widehat H(\xi):=\int_{\R^2}e^{i\xi\cdot x}H(x)\,dx,
\qquad H\in L^1(\R^2)\cap L^2(\R^2).
\]
For general $H\in L^2(\R^2)$,
$\widehat H$ denotes the $L^2$ limit of the Fourier transforms of any
$L^1\cap L^2$ sequence converging to $H$ in $L^2$; this is well-defined by
Plancherel's theorem. For $R>0$, let
\[
\mathring B_R:=\{u\in\R^2:\|u\|_2<R\}
\]
denote the interior of the ball $B_R$ defined in \Cref{sec:POT}, and define
the restriction space
\begin{equation}\label{eq:restriction_sobolev_space}
W^s(L^2(B_R))
:=
\left\{h\in L^2(\mathring B_R):
H|_{\mathring B_R}=h\ \text{a.e. for some }H\in W^s(L^2(\R^2))
\right\}
\end{equation}
with norm
\begin{equation}\label{eq:restriction_sobolev_norm}
\|h\|_{W^s(L^2(B_R))}
:=
\inf_{\substack{H\in W^s(L^2(\R^2))\\
H|_{\mathring B_R}=h\ \text{a.e.}}}
\|H\|_{W^s(L^2(\R^2))}.
\end{equation}
For integer $s\ge0$, the whole-space Fourier norm above is equivalent to the
standard weak-derivative norm. On a ball, the restriction space
$W^s(L^2(B_R))$ agrees with the usual weak-derivative space
$W^{s,2}(B_R)$, whose norm is
\begin{equation}\label{eq:integer_sobolev_norm}
\|h\|_{W^{s,2}(B_R)}^2
:=
\sum_{|\alpha|\le s}\|D^\alpha h\|_{L^2(B_R)}^2.
\end{equation}
Moreover, the restriction spaces on a ball satisfy the standard complex
interpolation inequality: for some $C_{\mathrm{interp},R}<\infty$,
\begin{equation}\label{eq:ball_sobolev_interpolation}
\|h\|_{W^{5/2}(L^2(B_R))}
\le C_{\mathrm{interp},R}
\|h\|_{W^{2,2}(B_R)}^{1/2}
\|h\|_{W^{3,2}(B_R)}^{1/2},
\qquad h\in W^{3,2}(B_R),
\end{equation}
where the two integer-order norms on the right are those in
\eqref{eq:integer_sobolev_norm}. Here is the precise extension argument.
The whole-space estimate follows from
\citet[Theorem~6.4.5(7)]{bergh2012interpolation}. Since $\mathring B_R$ is a
bounded Lipschitz domain, Rychkov's universal extension theorem
\citep{rychkov1999restrictions}, specialized to the
$L^2$-Sobolev scale, provides a linear extension operator, independent of
$s$,
\[
\mathcal E_R:W^s(L^2(B_R))\longrightarrow W^s(L^2(\R^2)),
\qquad s\in\R,
\]
whose restriction to $\mathring B_R$ is the identity and which is bounded for
each $s$. In particular, the same operator $\mathcal E_R$ is bounded at
$s=2$ and $s=3$: after absorbing the equivalence constants between the
restriction and weak-derivative norms, there are constants
$C_{\mathcal E,2,R},C_{\mathcal E,3,R}<\infty$ such that
\[
\|\mathcal E_Rh\|_{W^j(L^2(\R^2))}
\le C_{\mathcal E,j,R}\|h\|_{W^{j,2}(B_R)},
\qquad j=2,3.
\]
Therefore, for $h\in W^{3,2}(B_R)$,
\[
\begin{aligned}
\|h\|_{W^{5/2}(L^2(B_R))}
&\le \|\mathcal E_Rh\|_{W^{5/2}(L^2(\R^2))}\\
&\le C_R
\|\mathcal E_Rh\|_{W^2(L^2(\R^2))}^{1/2}
\|\mathcal E_Rh\|_{W^3(L^2(\R^2))}^{1/2}\\
&\le C_R C_{\mathcal E,2,R}^{1/2}C_{\mathcal E,3,R}^{1/2}
\|h\|_{W^{2,2}(B_R)}^{1/2}
\|h\|_{W^{3,2}(B_R)}^{1/2}\\
&=: C_{\mathrm{interp},R}
\|h\|_{W^{2,2}(B_R)}^{1/2}
\|h\|_{W^{3,2}(B_R)}^{1/2}.
\end{aligned}
\]
This proves \eqref{eq:ball_sobolev_interpolation}. We will use
\eqref{eq:integer_sobolev_norm} and \eqref{eq:ball_sobolev_interpolation} below.

\begin{lem}[Sobolev-to-ReLU control on a ball]\label{lem:sobolev_relu_control}
For $R>0$, let
\[
\Theta_R:=\Ss^1\times[-R,R],
\qquad
\Phi_{\omega,b}(u):=\relu(\omega\cdot u+b),
\]
and write
\[
\mathcal P_R:=\{\Phi_{\omega,b}:(\omega,b)\in\Theta_R\}\subset C(B_R).
\]
Let $\mathfrak M(\Theta_R)$ be the space of finite signed Borel measures on
$\Theta_R$, define
\[
T_R\beta(u):=\int_{\Theta_R}\Phi_{\omega,b}(u)\,d\beta(\omega,b),
\]
which belongs to $C(B_R)$ by joint continuity of the integrand, and set
\[
\mathcal V_R:=T_R\bigl(\mathfrak M(\Theta_R)\bigr)\subset C(B_R).
\]
For $h\in\mathcal V_R$, define the ReLU variation norm by
\begin{equation}\label{eq:variation_measure_norm}
\|h\|_{\mathcal V_R}
:=
\inf_{\substack{\beta\in\mathfrak M(\Theta_R)\\T_R\beta=h}}
\|\beta\|_{\rm TV},
\end{equation}
where $T_R\beta=h$ means equality at every point of $B_R$. Then every
$h\in W^{5/2}(L^2(B_R))$, identified with its unique continuous
representative on $B_R$, belongs to $\mathcal V_R$, and there is a constant
$A_R<\infty$, depending only on $R$, such that
\begin{equation}\label{eq:Sob_to_var_rewrite}
\|h\|_{\mathcal V_R}
\le A_R\|h\|_{W^{5/2}(L^2(B_R))}
\qquad\text{for every }h\in W^{5/2}(L^2(B_R)).
\end{equation}
Consequently, if $\eta$ is a finite signed Borel measure supported in $B_R$,
then
\begin{equation}\label{eq:sobolev_relu_integral_control}
\left|\int_{B_R}h\,d\eta\right|
\le
A_R\|h\|_{W^{5/2}(L^2(B_R))}
\sup_{(\omega,b)\in\Theta_R}
\left|\int_{B_R}\Phi_{\omega,b}(u)\,d\eta(u)\right|.
\end{equation}
Moreover,
\begin{equation}\label{eq:relu_parameters_to_sphere}
\sup_{(\omega,b)\in\Theta_R}
\left|\int_{B_R}\Phi_{\omega,b}(u)\,d\eta(u)\right|
\le
\sqrt{1+R^2}\,
\sup_{v\in\Ss^2}
\left|\int_{B_R}\relu\!\big(\langle v,(1,u)\rangle\big)\,d\eta(u)\right|,
\end{equation}
and therefore
\begin{equation}\label{eq:sobolev_sphere_integral_control}
\left|\int_{B_R}h\,d\eta\right|
\le
\sqrt{1+R^2}\,A_R\|h\|_{W^{5/2}(L^2(B_R))}
\sup_{v\in\Ss^2}
\left|\int_{B_R}\relu\!\big(\langle v,(1,u)\rangle\big)\,d\eta(u)\right|.
\end{equation}
\end{lem}

\begin{proof}
Throughout the proof, $d$ denotes the Euclidean metric inherited by
$\Theta_R\subset\R^3$.
For $R=1$, $\mathcal P_1$ is exactly the ReLU dictionary used by
\citet{mao2024approximation}. As in that paper, define the unit ball of
$K_1(\mathcal P_1)$ to be
$\overline{\operatorname{aconv}(\mathcal P_1)}^{\,L^2(B_1)}$, where
\[
\operatorname{aconv}(\mathcal P_1)
:=
\left\{
\sum_{j=1}^N a_j\phi_j:
N\in\N,\ \phi_j\in\mathcal P_1,\
\sum_{j=1}^N|a_j|\le1
\right\}
\]
is the finite absolutely convex hull. The corresponding variation-space norm
on $K_1(\mathcal P_1)$ is
\[
\|g\|_{K_1(\mathcal P_1)}
:=
\inf\left\{a>0:
g\in a\,\overline{\operatorname{aconv}(\mathcal P_1)}^{\,L^2(B_1)}
\right\}.
\]
We now compare this closed-hull norm with the exact signed-measure norm
\eqref{eq:variation_measure_norm}. Extending the terminology and notation of
\Cref{def:weak_vague} to finite signed measures on the compact space
$\Theta_R$, we write
\begin{equation}\label{eq:parameter_measure_weak}
\beta_n\xrightarrow{w}\beta
\quad\Longleftrightarrow\quad
\int_{\Theta_R}\psi\,d\beta_n\longrightarrow
\int_{\Theta_R}\psi\,d\beta
\quad\text{for every }\psi\in C(\Theta_R).
\end{equation}
This is weak convergence in the measure-theoretic terminology used in
\Cref{def:weak_vague}; every continuous function on $\Theta_R$ is bounded.
By the Riesz representation theorem,
$\mathfrak M(\Theta_R)=C(\Theta_R)^*$, and this same convergence is the
functional-analytic weak-* topology
$\sigma(\mathfrak M(\Theta_R),C(\Theta_R))$.
Consequently, the total-variation unit ball is weak-* compact by
Banach--Alaoglu. Since $\Theta_R$ is compact and metrizable,
$C(\Theta_R)$ is separable, so the weak-* topology is metrizable on this unit
ball. Thus every sequence in the unit ball has a subsequence
$\beta_{n_j}\xrightarrow{w}\beta$ with $\|\beta\|_{\rm TV}\le1$.

\smallskip
\noindent\textbf{Claim 1.} For every $R>0$,
\begin{equation}\label{eq:uniform_hull_measure_ball}
T_R\{\beta\in\mathfrak M(\Theta_R):\|\beta\|_{\rm TV}\le1\}
=
\overline{\operatorname{aconv}(\mathcal P_R)}^{\,C(B_R)},
\end{equation}
where the closure is taken in the uniform norm on $C(B_R)$.

We first record that weak convergence on a total-variation bounded sequence
is carried by $T_R$ to uniform convergence. For $u\in B_R$, define
\[
\Psi_u(\omega,b):=\Phi_{\omega,b}(u),
\qquad (\omega,b)\in\Theta_R.
\]
Joint continuity of the integrand on the compact set
$\Theta_R\times B_R$ implies that
$u\mapsto\Psi_u$ is continuous from $B_R$ into $C(\Theta_R)$ in the uniform
norm. Hence $\{\Psi_u:u\in B_R\}$ is compact in $C(\Theta_R)$. Suppose that
$\beta_n\xrightarrow{w}\beta$ and
$\sup_n\|\beta_n\|_{\rm TV}<\infty$, and set
$L:=\sup_n\|\beta_n-\beta\|_{\rm TV}<\infty$. Given $\epsilon>0$, the open
balls in $C(\Theta_R)$ of radius $\epsilon/(2(L+1))$ centered at the elements
of $\{\Psi_u:u\in B_R\}$ cover this compact set. A finite subcover therefore
provides $u_1,\ldots,u_N\in B_R$ such that, for every $u\in B_R$, there is an
index $j$ for which
\[
\|\Psi_u-\Psi_{u_j}\|_{C(\Theta_R)}<\frac{\epsilon}{2(L+1)}.
\]
By \eqref{eq:parameter_measure_weak}, for all sufficiently large $n$,
\[
\max_{1\le j\le N}
\left|\int_{\Theta_R}\Psi_{u_j}\,d(\beta_n-\beta)\right|<\frac{\epsilon}{2}.
\]
For such $u$ and $j$,
\[
\left|\int_{\Theta_R}\Psi_u\,d(\beta_n-\beta)\right|
\le
\left|\int_{\Theta_R}\Psi_{u_j}\,d(\beta_n-\beta)\right|
+
\|\Psi_u-\Psi_{u_j}\|_{C(\Theta_R)}
\|\beta_n-\beta\|_{\rm TV}
<
\frac{\epsilon}{2}+\frac{\epsilon L}{2(L+1)}
<\epsilon.
\]
Taking the supremum over $u\in B_R$ proves
\begin{equation}\label{eq:weak_to_uniform}
\beta_n\xrightarrow{w}\beta,\qquad
\sup_n\|\beta_n\|_{\rm TV}<\infty
\quad\Longrightarrow\quad
T_R\beta_n\longrightarrow T_R\beta\ \text{in }C(B_R).
\end{equation}

We now prove the two inclusions in \eqref{eq:uniform_hull_measure_ball}.
Every $g\in\operatorname{aconv}(\mathcal P_R)$ has the form
\[
g=\sum_{j=1}^N a_j\Phi_{\omega_j,b_j},
\qquad \sum_{j=1}^N|a_j|\le1.
\]
Thus $g=T_R\beta$ for the signed atomic measure
$\beta:=\sum_{j=1}^N a_j\delta_{(\omega_j,b_j)}$, which satisfies
$\|\beta\|_{\rm TV}=\sum_j|a_j|\le1$. Now let
$g_n\in\operatorname{aconv}(\mathcal P_R)$ converge uniformly to $g$, and
choose corresponding signed atomic measures $\beta_n$ such that
$g_n=T_R\beta_n$ and $\|\beta_n\|_{\rm TV}\le1$. By the compactness
conclusion above, a subsequence
satisfies $\beta_{n_j}\xrightarrow{w}\beta$ for some
$\|\beta\|_{\rm TV}\le1$. By \eqref{eq:weak_to_uniform},
$T_R\beta_{n_j}\to T_R\beta$ uniformly. Hence $g=T_R\beta$, proving
\[
\overline{\operatorname{aconv}(\mathcal P_R)}^{\,C(B_R)}
\subset
T_R\{\beta:\|\beta\|_{\rm TV}\le1\}.
\]

For the reverse inclusion, fix $\beta\in\mathfrak M(\Theta_R)$ with
$\|\beta\|_{\rm TV}\le1$. For each $n$, choose a finite Borel partition
$\{E_{n,j}\}_{j=1}^{N_n}$ of $\Theta_R$ into nonempty sets of diameter at most
$1/n$, choose $\theta_{n,j}\in E_{n,j}$, and define
\[
\beta_n:=\sum_{j=1}^{N_n}\beta(E_{n,j})\delta_{\theta_{n,j}}.
\]
Then
\[
\|\beta_n\|_{\rm TV}
=\sum_{j=1}^{N_n}|\beta(E_{n,j})|
\le\sum_{j=1}^{N_n}|\beta|(E_{n,j})
=\|\beta\|_{\rm TV}\le1.
\]
For every $\psi\in C(\Theta_R)$, uniform continuity gives
\[
\begin{aligned}
\left|\int_{\Theta_R}\psi\,d\beta_n-
\int_{\Theta_R}\psi\,d\beta\right|
&\le
\sum_{j=1}^{N_n}\int_{E_{n,j}}
|\psi(\theta_{n,j})-\psi(\theta)|\,d|\beta|(\theta)\\
&\le
\left(
\sup_{\substack{\theta,\theta'\in\Theta_R\\
d(\theta,\theta')\le1/n}}
|\psi(\theta)-\psi(\theta')|
\right)\|\beta\|_{\rm TV}
\longrightarrow0.
\end{aligned}
\]
Thus $\beta_n\xrightarrow{w}\beta$. Since
$T_R\beta_n\in\operatorname{aconv}(\mathcal P_R)$,
\eqref{eq:weak_to_uniform} shows that $T_R\beta$ belongs to the uniform
closure of $\operatorname{aconv}(\mathcal P_R)$. This proves the reverse
inclusion and hence \eqref{eq:uniform_hull_measure_ball}.

\smallskip
\noindent\textbf{Claim 2.} The corresponding identity in $L^2(B_R)$ is
\begin{equation}\label{eq:L2_hull_measure_ball}
\overline{\operatorname{aconv}(\mathcal P_R)}^{\,L^2(B_R)}
=
\{T_R\beta:\beta\in\mathfrak M(\Theta_R),\ \|\beta\|_{\rm TV}\le1\},
\end{equation}
where the continuous functions on the right are viewed as $L^2(B_R)$
equivalence classes. To prove the inclusion from right to left, let
$T_R\beta$ belong to the right-hand side. By
\eqref{eq:uniform_hull_measure_ball}, there are
$g_n\in\operatorname{aconv}(\mathcal P_R)$ such that
$g_n\to T_R\beta$ uniformly. Since $B_R$ has finite Lebesgue measure,
\[
\|g_n-T_R\beta\|_{L^2(B_R)}
\le |B_R|^{1/2}\|g_n-T_R\beta\|_{C(B_R)}
\longrightarrow0.
\]
Thus $T_R\beta$ belongs to the $L^2(B_R)$ closed hull. Conversely, suppose that
$g_n=T_R\beta_n\in\operatorname{aconv}(\mathcal P_R)$ converges to $g$ in
$L^2(B_R)$, where the $\beta_n$ are atomic and
$\|\beta_n\|_{\rm TV}\le1$. As above, a subsequence converges weakly to some
$\beta$ with $\|\beta\|_{\rm TV}\le1$, and
\eqref{eq:weak_to_uniform} gives
$T_R\beta_{n_j}\to T_R\beta$ uniformly, hence also in $L^2(B_R)$. Uniqueness
of the $L^2$ limit yields $g=T_R\beta$ almost everywhere. Thus every element
of the $L^2$ closed hull has a continuous representative in $\mathcal V_R$;
this representative is unique because two continuous functions on $B_R$ that
agree almost everywhere agree everywhere. This proves
\eqref{eq:L2_hull_measure_ball}.

For every $a>0$, linearity of $T_R$ and
\eqref{eq:L2_hull_measure_ball} give
\[
a\,\overline{\operatorname{aconv}(\mathcal P_R)}^{\,L^2(B_R)}
=
\{T_R\beta:\beta\in\mathfrak M(\Theta_R),\ \|\beta\|_{\rm TV}\le a\}.
\]
Consequently, for every $g\in\mathcal V_R$,
\[
\inf\left\{a>0:
g\in a\,\overline{\operatorname{aconv}(\mathcal P_R)}^{\,L^2(B_R)}
\right\}
=
\inf_{\substack{\beta\in\mathfrak M(\Theta_R)\\T_R\beta=g}}
\|\beta\|_{\rm TV}
=\|g\|_{\mathcal V_R}.
\]
In particular, for $R=1$ the $K_1(\mathcal P_1)$ norm of
\citet{mao2024approximation} agrees with the exact signed-measure norm
\eqref{eq:variation_measure_norm}.

The Sobolev norm used in Theorem~1 of \citet{mao2024approximation} is equivalent
to the restriction norm \eqref{eq:restriction_sobolev_norm}. After absorbing the
fixed norm-equivalence constant, their result, specialized to $(d,k)=(2,1)$,
states that
\begin{equation}\label{eq:mao_unit}
\|g\|_{K_1(\mathcal P_1)}
\le A_0\|g\|_{W^{5/2}(L^2(B_1))}
\qquad\text{for every }g\in W^{5/2}(L^2(B_1))
\end{equation}
for a constant $A_0<\infty$. At this stage, $g$ is an $L^2(B_1)$
equivalence class. By \eqref{eq:L2_hull_measure_ball}, every element of
$K_1(\mathcal P_1)$ agrees almost everywhere with $T_1\beta$ for some finite
signed measure $\beta$, and $T_1\beta$ is continuous on $B_1$. This continuous
representative is unique, because two continuous functions on $B_1$ that
agree almost everywhere agree everywhere. Thus
\eqref{eq:L2_hull_measure_ball} and the equality of the two norms proved above
turn \eqref{eq:mao_unit} into the same inequality for the exact representation
norm $\|\cdot\|_{\mathcal V_1}$. In particular, the class of representing
measures in \eqref{eq:variation_measure_norm} is nonempty for every
$g\in W^{5/2}(L^2(B_1))$.

It remains to pass from $B_1$ to $B_R$. For
$h\in W^{5/2}(L^2(B_R))$, set $\widetilde h(x):=h(Rx)$. If
$H\in W^{5/2}(L^2(\R^2))$ extends $h$, then $\widetilde H(x):=H(Rx)$ extends
$\widetilde h$. Since
$\widehat{\widetilde H}(\xi)=R^{-2}\widehat H(\xi/R)$, the change of variables
$\xi=R\zeta$ in \eqref{eq:whole_space_sobolev} gives
\[
\begin{aligned}
\|\widetilde H\|_{W^s(L^2(\R^2))}^2
&=
R^{-2}\int_{\R^2}(1+R\|\zeta\|_2)^{2s}
|\widehat H(\zeta)|^2\,d\zeta\\
&\le
R^{-2}\max\{1,R\}^{2s}
\|H\|_{W^s(L^2(\R^2))}^2.
\end{aligned}
\]
Taking square roots yields
\[
\|\widetilde H\|_{W^s(L^2(\R^2))}
\le R^{-1}\max\{1,R\}^{s}\|H\|_{W^s(L^2(\R^2))},
\qquad s\ge0.
\]
Taking the infimum over all extensions $H$ and setting
$C_{{\rm sc},R}:=R^{-1}\max\{1,R\}^{5/2}$ yields
\begin{equation}\label{eq:scale_frac}
\|\widetilde h\|_{W^{5/2}(L^2(B_1))}
\le C_{{\rm sc},R}\|h\|_{W^{5/2}(L^2(B_R))}.
\end{equation}
If $\alpha$ represents $\widetilde h$ on $B_1$, let
\[
T_R^{\rm par}(\omega,b):=(\omega,Rb),
\qquad
\beta:=\frac1R(T_R^{\rm par})_\#\alpha.
\]
Positive homogeneity of ReLU gives, for every $u\in B_R$,
\[
h(u)
=\widetilde h(u/R)
=\int_{\Theta_R}\relu(\omega\cdot u+b)\,d\beta(\omega,b),
\qquad
\|\beta\|_{\rm TV}=\frac1R\|\alpha\|_{\rm TV}.
\]
Combining this with \eqref{eq:mao_unit} and \eqref{eq:scale_frac} proves
\eqref{eq:Sob_to_var_rewrite} with
\[
A_R:=\frac{A_0C_{{\rm sc},R}}{R}.
\]

Finally, let $\beta\in\mathfrak M(\Theta_R)$ represent $h$. The integrand
$((\omega,b),u)\mapsto\Phi_{\omega,b}(u)$ is continuous, hence bounded, on
the compact set $\Theta_R\times B_R$. Since $|\beta|$ and $|\eta|$ are finite,
it is integrable with respect to $|\beta|\otimes|\eta|$, so Fubini's theorem
gives
\[
\begin{aligned}
\int_{B_R}h(u)\,d\eta(u)
&=
\int_{B_R}\int_{\Theta_R}
\Phi_{\omega,b}(u)\,d\beta(\omega,b)\,d\eta(u)\\
&=
\int_{\Theta_R}\left(\int_{B_R}
\Phi_{\omega,b}(u)\,d\eta(u)\right)d\beta(\omega,b).
\end{aligned}
\]
Consequently,
\[
\left|\int_{B_R}h(u)\,d\eta(u)\right|
\le
\|\beta\|_{\rm TV}
\sup_{(\omega,b)\in\Theta_R}
\left|\int_{B_R}\Phi_{\omega,b}(u)\,d\eta(u)\right|.
\]
Taking the infimum over all representing measures $\beta$ yields the same
bound with $\|h\|_{\mathcal V_R}$ in place of $\|\beta\|_{\rm TV}$. Combining
this bound with \eqref{eq:Sob_to_var_rewrite} proves
\eqref{eq:sobolev_relu_integral_control}.

To pass from the affine ReLU parameters to sphere directions, fix
$(\omega,b)\in\Theta_R$ and set
\[
v:=\frac{(b,\omega)}{\sqrt{b^2+\|\omega\|_2^2}}
=\frac{(b,\omega)}{\sqrt{b^2+1}}\in\Ss^2.
\]
Positive homogeneity of ReLU gives, for every $u\in\R^2$,
\[
\Phi_{\omega,b}(u)
=\relu(\omega\cdot u+b)
=\sqrt{b^2+1}\,\relu\!\big(\langle v,(1,u)\rangle\big).
\]
Since $|b|\le R$, taking suprema proves
\eqref{eq:relu_parameters_to_sphere}; combining it with
\eqref{eq:sobolev_relu_integral_control} proves
\eqref{eq:sobolev_sphere_integral_control}.
\end{proof}

\section{Discussion: Linear Summaries, Augmentation, and Related Embeddings}
\label{sec:augment_discussion}

\subsection{Augmentation Strategy}\label{sec:augmentation_strategy}

This appendix first records how signed diagonal augmentation makes linear summaries of persistence diagrams compatible with the deletion-to-diagonal mechanism of $\POT_1$, without prescribing ad hoc vanishing weights near the diagonal.

Let $\Phi$ be a linear integral summary of the form
\[
\Phi(\tilde\mu)=\int_{\overline X}\varphi(\cdot,p)\,d\tilde\mu(p),
\]
defined for positive Borel measures $\tilde\mu$ on $\overline X$, with values in a normed linear space $(\mathsf F,\|\cdot\|_{\mathsf F})$. Many standard constructions fit this template, including persistence images/surfaces and several kernel-based mean embeddings, provided the feature map extends naturally from \(X\) to \(\overline X=X\cup\Delta\).

For a measure $\mu$ on $X$, we consider its signed diagonal augmentation
\[
\mu^{\aug}=\mu-(\pi_\Delta)_\#\mu .
\]
Then, by linearity and the cross-augmentation identity \Cref{eq:aug_cross_identity},
\[
\Phi(\mu^{\aug})-\Phi(\nu^{\aug})
=
\Phi(\mu\oplus_\Delta\nu)-\Phi(\nu\oplus_\Delta\mu).
\]
Thus differences of signed augmented features can be rewritten as differences of the same feature map evaluated on positive cross-augmented measures. This is the key structural point: once deletions are encoded at the level of measures, comparisons reduce to ordinary optimal transport on $\overline X$.

Assume now that the feature family is uniformly Lipschitz in the measure variable, namely that there exists $L<\infty$ such that
\[
\|\varphi(\cdot,p)-\varphi(\cdot,q)\|_{\mathsf F}\le L\,\|p-q\|_\infty
\qquad\forall p,q\in\overline X.
\]
Then $\Phi$ is $L$-Lipschitz with respect to $\OT_1$ on positive measures of equal mass. Indeed, if $\tilde\mu,\tilde\nu$ are such measures and $\Gamma\in\Pi(\tilde\mu,\tilde\nu)$ is any coupling, then
\begin{align*}
\|\Phi(\tilde\mu)-\Phi(\tilde\nu)\|_{\mathsf F}
&=
\left\|
\int_{\overline X\times\overline X}
\bigl(\varphi(\cdot,p)-\varphi(\cdot,q)\bigr)\,d\Gamma(p,q)
\right\|_{\mathsf F}\\
&\le
\int_{\overline X\times\overline X}
\|\varphi(\cdot,p)-\varphi(\cdot,q)\|_{\mathsf F}\,d\Gamma(p,q)\\
&\le
L\int_{\overline X\times\overline X}\|p-q\|_\infty\,d\Gamma(p,q).
\end{align*}
Taking the infimum over $\Gamma$ gives
\[
\|\Phi(\tilde\mu)-\Phi(\tilde\nu)\|_{\mathsf F}\le L\,\OT_1(\tilde\mu,\tilde\nu).
\]
Applying this with $\tilde\mu=\mu\oplus_\Delta\nu$ and $\tilde\nu=\nu\oplus_\Delta\mu$, and using \Cref{prop:POT_vs_OT_aug}, yields
\[
\|\Phi(\mu^{\aug})-\Phi(\nu^{\aug})\|_{\mathsf F}
\le
L\,\OT_1(\mu\oplus_\Delta\nu,\nu\oplus_\Delta\mu)
\le
2L\,\POT_1(\mu,\nu).
\]
Hence, once the deletion mechanism is built into the representation through diagonal augmentation, no additional diagonal weighting is needed to obtain $\POT_1$-stability for linear integral summaries. Note that this also matches the general role of $\POT_1$ for linear representations emphasized in \citet{skraba2020wasserstein}.

Stability, however, is only one side of the problem. To obtain inverse continuity, the feature family $\{\varphi(\cdot,p)\}_{p\in\overline X}$ must also be rich enough to approximate $1$-Lipschitz test functions on compact subsets of $\overline X$. When such an approximation property is available, the same duality-and-approximation mechanism used in \Cref{prop:compact_core_holder_siegel} can in principle be adapted to derive local inverse bounds on compact cores. Persistence spheres realize this strategy through ReLU ridge features, for which recent approximation theory provides explicit rates. In this sense, the augmentation principle is more general than persistence spheres themselves: it suggests a systematic route for designing parameter-free, $\POT_1$-coherent linear summaries of persistence diagrams.

\subsection{Comparison with classical summaries of persistence diagrams}
\label{sec:comparison_classical_summaries}

We compare persistence spheres with persistence images, persistence landscapes, sliced Wasserstein
kernels, and persistence splines
\citep{bubenik2015statistical,adams2017persistence,carriere2017sliced,dong2024persistence}.
These constructions differ both in how they interact with the transport geometry of diagrams and in
the restrictions under which Hilbert-space embedding results are available.

\paragraph{Persistence images.}
Persistence images \citep{adams2017persistence} are obtained by first associating to a persistence
diagram a persistence surface, namely a weighted sum of smoothing kernels centered at the
diagram points, and then discretizing this surface on a fixed grid. The persistence surface is a
function-valued summary, whereas the discretized image is a vector in a finite-dimensional Euclidean
space, hence in a Hilbert space. Under suitable regularity assumptions on the weighting function, in
particular its vanishing on the diagonal, and on the smoothing kernel, Adams et al.\ prove
\(\POT_1\)-type stability for both levels of the construction: the persistence surface is stable in
sup norm, and the persistence image is stable after discretization \citep{adams2017persistence}.

The Hilbert-valued object is the discretized image rather than the persistence surface itself, so the
result is a family of finite-dimensional Euclidean embeddings indexed by the chosen grid. Kernel
smoothing and the weighting scheme also affect the induced geometry, as illustrated in
\Cref{sec:deforming_geometry}.

\paragraph{Persistence landscapes.}
Persistence landscapes \citep{bubenik2015statistical} associate to each diagram a sequence of tent-like
functions and naturally take values in \(L^p(\mathbb N\times\mathbb R)\); for \(p=2\) this is a
Hilbert space. The original theory proves, among other results, sup-norm stability with respect to the
bottleneck distance. For finite \(q\), the available bounds use persistence-weighted transport
quantities or additional assumptions such as bounded total persistence, rather than a global
Lipschitz statement for \(D\mapsto\lambda(D)\) from arbitrary persistence diagrams endowed with
\(W_p\) (\(1\le p<\infty\)) into \(L^q\) for finite \(q\).
Indeed, \citet{skraba2020wasserstein} show that for finite \(p,q\), this map is in general not even
H\"older continuous from diagrams endowed with \(W_p\) to landscapes endowed with the \(L^q\) norm.
The elementary example in \Cref{sec:deforming_geometry} exhibits the corresponding amplification of
high-persistence variability.

\paragraph{Sliced Wasserstein kernels.}
Sliced Wasserstein kernels \citep{carriere2017sliced} induce an implicit embedding into a reproducing
kernel Hilbert space through sliced optimal transport. On classes of diagrams with uniformly bounded
cardinality, the sliced Wasserstein distance is quantitatively comparable to \(\POT_1\), and the
kernel distance admits two-sided control through continuous comparison functions. Their metric control
therefore relies on bounded cardinality, whereas ours uses control of tail-to-diagonal mass.

\paragraph{Comparison of regimes.}
The class \(\mathcal M_c(\mathcal A,\mathcal B)\) from \Cref{def:McAB} differs from the
bounded-cardinality classes used in most Hilbert-embedding results
\citep{carriere2017sliced,mitra2024geometric,bate2024bi}. It imposes no bound on total mass. For
counting measures, the number of points may diverge, provided the additional mass concentrates near
the diagonal in a quantified way through the mixed moment bound
\[
\int_X \|(1,p)\|_2\,\pers(p)\,d\mu(p)\le \mathcal B(\mathbb R^2),
\]
together with the local first-moment control encoded by \(\mathcal A\) and the compatibility condition
linking the growth of \(\mathcal A(B_R)\) to the decay of
\(\mathcal B(\mathbb R^2\setminus B_{R-1})\). This allows increasing cardinality in noisy settings,
while a cardinality bound alone does not prevent even a one-point diagram from drifting arbitrarily far
in the diagonal direction. Thus the mechanism behind \Cref{cor:McAB_bicont_L2} is a
\(\POT_1\)-adapted control of tail-to-diagonal mass rather than a hard cap on the number of points.

\section{Deforming the Wasserstein Geometry}
\label{sec:deforming_geometry}

This appendix uses elementary examples to illustrate how different summaries reshape pairwise
$\POT_1$ distances after mapping diagrams into a linear or Hilbert space. These examples are not a
formal distortion analysis; they isolate geometric biases that may matter in applications. We first
compare the revised persistence-sphere definition
with the original weighted version from \citet{pegoraro2025persistence}, isolating the role of
signed diagonal augmentation. We then consider persistence landscapes and Gaussian persistence
images, for which high-persistence amplification and kernel-dependent saturation can be computed
explicitly.

\subsection{Comparison with the Original Persistence Spheres}
\label{sec:compare_old}

We start with a minimal synthetic example aimed at isolating the effect of the signed diagonal
augmentation in the revised definition. The example serves two purposes: it illustrates the improved
alignment with the \(\POT_1\) geometry under diagonal translations, and it provides an empirical
check of the decay behavior predicted by
\Cref{cor:counting_shift_pointwise_limit,lem:uniform_decay}.

For \(k\ge 0\), let
\[
p_k:=(0,1)+(k,k),
\qquad
q_k:=p_k+\frac{1}{\sqrt{2}}(1,1),
\]
and define the one-point diagrams
\[
D_k:=\{p_k\},
\qquad
D_k':=\{q_k\}.
\]
Equivalently, in measure form, let \(\mu_k:=\delta_{p_k}\) and \(\nu_k:=\delta_{q_k}\). We monitor
two quantities:
\begin{enumerate}
    \item the diagonal-shift discrepancy
    \[
    \|S(\mu_k)-S(\nu_k)\|_{L^2(\Ss^2)},
    \]
    measuring how strongly a pure \((1,1)\)-offset is detected as both points drift along the
    diagonal; and
    \item the deletion cost
    \[
    \|S(\mu_k)-S(0)\|_{L^2(\Ss^2)}=\|S(\mu_k)\|_{L^2(\Ss^2)},
    \]
    measuring how the cost of sending \(p_k\) to the diagonal is encoded by the representation.
\end{enumerate}

We compare these curves with the original weighted construction of
\citet{pegoraro2025persistence}, using the weighting scheme
\[
\lambda(p):=\frac{y-x}{2\|(1,p)\|_2},
\qquad
\omega_K(p)=\frac{2}{\pi}\arctan\!\left(\frac{\lambda(p)}{K}\right),
\]
for different values of \(K\). This is a natural benchmark: among the weighting schemes considered
in \cite{pegoraro2025persistence}, \(\omega_K\) gave the best practical behavior, and in
retrospect this is not surprising, since it qualitatively mimics the diagonal-deletion attenuation
that the new definition now encodes intrinsically.

The results are shown in \Cref{fig:vs_old}. In \Cref{fig:decay}, all methods exhibit decay of
\(\|S(\mu_k)-S(\nu_k)\|_{L^2}\) as \(k\to\infty\), but for different reasons. For the updated
persistence spheres, the decay is intrinsic and reflects
\Cref{cor:counting_shift_pointwise_limit}: far along the diagonal, differences in the
\(d\)-coordinate are progressively washed out, and the comparison reduces pointwise in \(v\) to the
total persistence weighted by \(t(v)\). For the older construction, instead, the decay is enforced
by the reweighting itself: for every fixed \(K>0\), the factor \(\omega_K\) tends to \(0\) along
each diagonal line
\[
\ell_q:=\left\{\,q+s\frac{1}{\sqrt{2}}(1,1):\ s\in\R\,\right\},
\]
so points translated in the direction \(\frac{1}{\sqrt{2}}(1,1)\) are eventually progressively
downweighted.

The deletion plot in \Cref{fig:deletion} makes the main difference even clearer. In the revised
definition, \Cref{rmk:pot_to_zero} shows that comparison with \(0\) in the sup norm depends only on
persistence mass, hence is insensitive to diagonal translations. Since
\(\pers(p_k)=\pers(p_0)\) for all \(k\), the deletion cost in the new persistence spheres is
therefore essentially constant in \(k\); the
small variability visible for small \(k\) in the figure comes from averaging the directional factor
\(t(v)\) through the \(L^2\)-norm. In the weighted construction of
\cite{pegoraro2025persistence}, the same qualitative behavior is obtained only indirectly. As
\(p_k\) moves away from the origin, the geometric contribution first grows, and only later is this
compensated by the decay of \(\omega_K(p_k)\), producing for both values of \(K\) a more
pronounced transient regime and a stabilization that depends on \(K\).

Overall, \Cref{fig:vs_old} shows that the augmentation-based formulation induces a milder
deformation of the \(\POT_1\) geometry than the original weighted definition and that the updated
persistence spheres recover the same desirable behavior in a parameter-free and geometrically
intrinsic way.

\begin{figure}
\centering
	\begin{subfigure}[c]{0.45\textwidth}
		\centering
		\includegraphics[width = \textwidth]{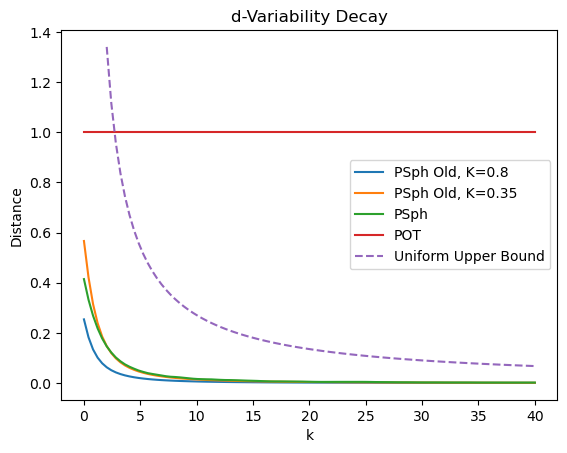}
		\captionsetup{singlelinecheck=off, margin={0.05cm, 0.05cm}}
		\caption{\textbf{Diagonal-shift discrepancy.} The quantity
        $\|S(\mu_k)-S(\nu_k)\|_{L^2(\Ss^2)}$ as a function of $k$, for the updated persistence
        spheres and for the definition of \cite{pegoraro2025persistence} using $\omega_K$ with
		$K\in\{0.35,0.8\}$. For reference we also plot $\POT_1(\mu_k,\nu_k)$ and the uniform upper
        bound from \Cref{lem:uniform_decay}: the latter controls
        $\|S(\mu_k)-S(\nu_k)\|_\infty$, whereas the plotted curves average the discrepancy over
		$\Ss^2$ via the $L^2$ norm.}
		\label{fig:decay}
	\end{subfigure}\hfill
	\begin{subfigure}[c]{0.45\textwidth}
		\centering
		\includegraphics[width = \textwidth]{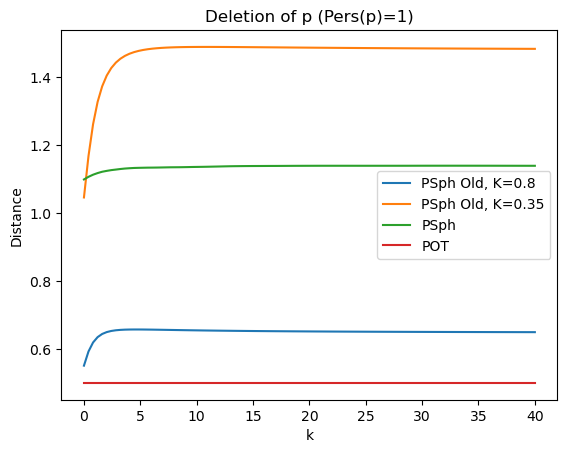}
		\captionsetup{singlelinecheck=off, margin={0.05cm, 0.05cm}}
		\caption{\textbf{Deletion cost.} The quantity
        $\|S(\mu_k)\|_{L^2(\Ss^2)}=\|S(\mu_k)-S(0)\|_{L^2(\Ss^2)}$ as a function of $k$,
        compared across the updated persistence spheres, the construction of
        \cite{pegoraro2025persistence} with $\omega_K$ ($K\in\{0.35,0.8\}$), and the
         deletion cost $\POT_1(\mu_k,0)=\pers(p_0)$, constant in $k$.}
		\label{fig:deletion}
	\end{subfigure}
    \caption{
    One-point diagrams drifting along $(1,1)$. Left: attenuation of a pure diagonal offset
    between $D_k$ and $D_k'$. Right: behavior of the deletion-to-diagonal cost for $D_k$. The
    updated definition matches the $\POT_1$ geometry by construction, whereas in
    \cite{pegoraro2025persistence} both effects are mediated by the decay of the reweighting
    $\omega_K$ along diagonal lines.}
    \label{fig:vs_old}
\end{figure}

\subsection{Persistence Landscapes}

Let
\[
p=(1,2),\qquad q=(1,3),\qquad
u=\frac{1}{\sqrt{2}}(-1,1),
\qquad p_k=p+k u,\qquad q_k=q+k u,
\]
and define the one-point diagrams
\[
D_k:=\{p_k\},
\qquad D_k':=\{q_k\}.
\]
For a one-point diagram, the persistence landscape is a single tent centered at
the midpoint of the persistence pair, with height equal to its persistence. Write
\(h_k:=\pers(p_k)=\frac12+\frac{k}{\sqrt2}\). The exact $L^2$ distance from
$\PL(D_k)$ to the zero landscape is
\[
\|\PL(D_k)\|_{L^2(\R)}
=\sqrt{\frac23}\,h_k^{3/2}.
\]
A direct integration also gives
\[
\|\PL(D_k)-\PL(D_k')\|_{L^2(\R)}
=\sqrt{h_k}
=\sqrt{\frac12+\frac{k}{\sqrt2}}.
\]
Thus the pairwise landscape distance grows like $k^{1/2}$, whereas the deletion
distance grows like $k^{3/2}$. In \Cref{fig:vs_PL}, the orange deletion curve
uses $w=(1,1.00001)$ as a near-diagonal proxy for the empty diagram; its
landscape contribution is negligible.

These computations make explicit the main geometric bias of persistence landscapes: they amplify
variability at high persistence. In particular, moving a point in a direction that increases
persistence can have a much larger effect on the landscape norm than one would expect from the
underlying transport displacement alone.

This behavior is consistent with the general instability results of \cite{skraba2020wasserstein}:
for every \(q\in[1,\infty)\), the persistence-landscape map from persistence diagrams endowed with
\(W_p\) to \(L^q\) is not H\"older continuous. Their counterexample uses two one-point diagrams,
\[
D=\{(0,a)\},
\qquad
D'=\{(0,a-r)\},
\]
with \(r\ll a\), for which \(W_p(D,D')=r\) while
\[
\|\PL(D)-\PL(D')\|_{L^q}
\gtrsim
r\,(a/2-r)^{1/q}.
\]
This cannot be bounded uniformly by \(Cr^\alpha\), and formalizes the same high-persistence
amplification seen in the elementary example above.

\subsection{Persistence Images}

Persistence images deform the \(\POT_1\) geometry through two distinct mechanisms:
kernel-induced saturation and persistence-dependent reweighting. The first effect is already
visible at the level of a single Gaussian atom, before any additional weighting is introduced.

Strictly speaking, persistence images are defined by first forming a persistence
surface on the birth-persistence plane and then integrating that surface over the boxes of a chosen grid,
typically on a relevant subdomain. Thus the exact construction involves both truncation and
discretization. Since our goal here is only to isolate the underlying geometric mechanism, we
work instead on the whole space \(\mathbb R^2\), ignoring boundary and discretization effects.
This keeps the computations explicit while still capturing the relevant qualitative behavior,
especially when the chosen window is large enough that the Gaussian tails outside it are
negligible. See \citet{adams2017persistence} for the original persistence-surface and
persistence-image construction.  
Let \(g_p^\sigma\) denote the isotropic Gaussian centered at \(p\) with bandwidth \(\sigma\),
\[
g_p^\sigma(z)
=
\frac{1}{2\pi\sigma^2}\exp\!\left(-\frac{\|z-p\|^2}{2\sigma^2}\right),
\qquad z\in\mathbb{R}^2.
\]
For arbitrary \(p,q\in X\), a direct computation gives
\[
\|g_p^\sigma-g_q^\sigma\|_{L^2(\mathbb{R}^2)}^2
=
\|g_p^\sigma\|_{L^2}^2+\|g_q^\sigma\|_{L^2}^2
-2\langle g_p^\sigma,g_q^\sigma\rangle,
\]
with
\[
\|g_p^\sigma\|_{L^2(\mathbb{R}^2)}^2
=
\frac{1}{4\pi\sigma^2},
\qquad
\langle g_p^\sigma,g_q^\sigma\rangle
=
\frac{1}{4\pi\sigma^2}\exp\!\left(-\frac{\|p-q\|^2}{4\sigma^2}\right).
\]
Hence
\[
\|g_p^\sigma-g_q^\sigma\|_{L^2(\mathbb{R}^2)}^2
=
\frac{1}{2\pi\sigma^2}
\left(1-\exp\!\left(-\frac{\|p-q\|^2}{4\sigma^2}\right)\right),
\]
or equivalently
\[
\|g_p^\sigma-g_q^\sigma\|_{L^2(\mathbb{R}^2)}
=
\frac{1}{\sqrt{2\pi}\sigma}
\left(1-\exp\!\left(-\frac{\|p-q\|^2}{4\sigma^2}\right)\right)^{1/2}.
\]
Thus the image-space distance is an increasing function of \(\|p-q\|\), but it does not grow
indefinitely: once the overlap between the two Gaussians becomes negligible, it saturates at
\(1/(\sqrt{2\pi}\sigma)\).

To relate this to \(\POT_1\), consider now a moving one-point diagram
\[
D_k:=\{p_k\},
\qquad
p_k:=p+k\,v,
\qquad
D':=\{q\},
\qquad k\ge 0,
\]
where \(v\) is a unit vector compatible with the upper-half-plane constraint. In
\Cref{fig:vs_PI} we take
\[
p=(100,201),
\qquad
q=(100,200),
\qquad
v=\frac{1}{\sqrt{2}}(1,1).
\]
The calculation above deliberately concerns unweighted continuous Gaussian atoms, in order to
isolate kernel-induced saturation independently of weighting, truncation, and discretization. The
same qualitative conclusion holds for weights that are constant, or approach a common constant,
away from the diagonal; if that constant is \(1\), the ceiling above is recovered.

Figure~\ref{fig:vs_PI} uses discretized, persistence-weighted images. Since \(p_k\) moves parallel
to the diagonal, its persistence weight is constant in \(k\), so weighting changes the numerical
value of the plateau but not the saturation mechanism: as the overlap with the fixed atom vanishes,
the persistence-image distance approaches a \(k\)-independent level, while the corresponding
\(\POT_1\) distance continues to grow until deletion becomes preferable.

An even stronger distortion appears when one compares a point directly with the zero diagram.
Without persistence-based weighting, deletion to zero is represented by the Gaussian mass
\(\|g_{p_k}^\sigma\|_{L^2(\mathbb{R}^2)}=1/(2\sqrt{\pi}\sigma)\), up to the
truncation and discretization effects of the actual image representation. In
other words, once the Gaussian is sufficiently far from the diagonal, its deletion cost is
essentially set by the kernel normalization alone, rather than by the persistence of the
point. This can severely undershoot the true transport cost: high-persistence points are
deleted at almost the same image cost as lower-persistence ones.

A persistence-dependent weighting can partially compensate for this effect, but only by
introducing a second deformation of the geometry. If one weights each atom by a function of
persistence, then high-persistence points receive larger amplitudes, so deletion costs can be
made to grow more in line with \(\POT_1\). But this comes at the price of an explicit
persistence bias, analogous in spirit to the high-persistence emphasis observed for persistence
landscapes. On the other hand, a weight that stays essentially flat away from the diagonal
avoids that bias, but then the deletion cost remains tied to the Gaussian mass and can
dramatically underestimate transport to the diagonal.

Thus the deformation induced by persistence images is not just the saturation of pairwise
distances. More fundamentally, the representation struggles to encode deletion costs at the
correct scale. Without persistence-based weighting, deletion can be severely underestimated;
with such weighting, one recovers a more realistic scale only by imposing a strong
persistence bias.

\begin{figure}
\centering
	\begin{subfigure}[c]{0.45\textwidth}
		\centering
		\includegraphics[width = \textwidth]{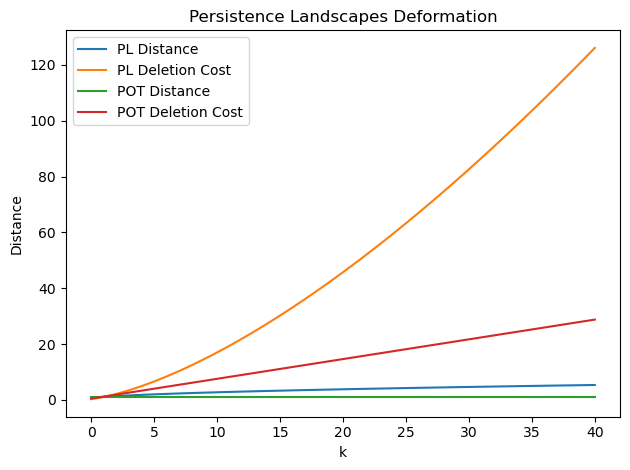}
		\captionsetup{singlelinecheck=off, margin={0.05cm, 0.05cm}}
		 \caption{\textbf{Persistence landscapes.} The blue pairwise $L^2$
landscape distance grows like $k^{1/2}$, while the orange near-diagonal
deletion proxy grows like $k^{3/2}$. The green curve is the constant pairwise
$\POT_1$ distance, while the red $\POT_1$ deletion proxy
grows linearly in $k$.}
		\label{fig:vs_PL}
	\end{subfigure}\hfill
	\begin{subfigure}[c]{0.45\textwidth}
		\centering
		\includegraphics[width = \textwidth]{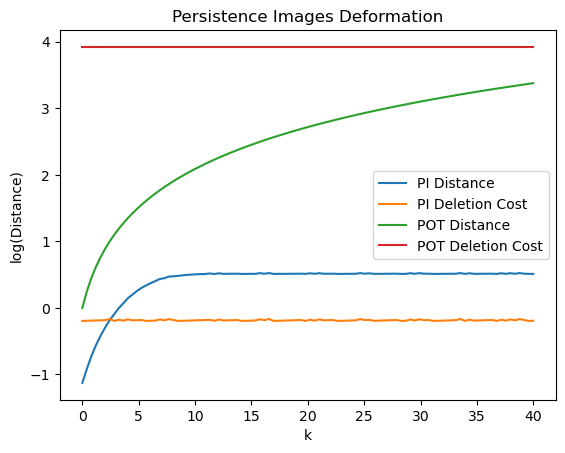}
		\captionsetup{singlelinecheck=off, margin={0.05cm, 0.05cm}}
		\caption{\textbf{Persistence images.} One-point toy example illustrating the deformation induced by Gaussian persistence images. As \(D_k=\{p_k\}\) moves in the diagonal direction, the \(L^2\) distance between the corresponding Gaussian atom and that of \(D'=\{q\}\) increases at first but then saturates at a kernel-dependent ceiling. Thus, for fixed bandwidth \(\sigma\), the persistence-image distance stabilizes much earlier than the corresponding \(\POT_1\) scale when the deletion cost is large. The vertical axis, representing the relevant distances, is displayed on a logarithmic scale.}
		\label{fig:vs_PI}
	\end{subfigure}
    \caption{
    Two elementary mechanisms by which standard topological summaries deform the \(\POT_1\)
geometry. Persistence landscapes amplify variability at high persistence, while Gaussian
persistence images induce a kernel-dependent saturation of pairwise distances. In the latter case,
for fixed bandwidth \(\sigma\), the discrepancy with \(\POT_1\) becomes more pronounced at high
persistence, since the persistence-image distance saturates at a fixed kernel scale whereas the
\(\POT_1\) flattening occurs only at the much larger deletion scale determined by the sum of the
persistences.}
    \label{fig:vs_PLPI}
\end{figure}

\subsection{Take-home message}

These examples illustrate a general point: all topological summaries deform the \(\POT_1\)
geometry, and each therefore introduces its own geometric bias into downstream analyses. From this
viewpoint, having a broader repertoire of summaries is valuable, especially when some of them, as
in the case of persistence spheres, come with explicit guarantees on their relation to the original
diagram geometry. At the same time, understanding the specific deformation mechanism of each summary
helps interpret empirical results, since part of the observed variability may come not only from the
data, but also from the way the chosen representation emphasizes, suppresses, or reshapes different
types of variation.

\section*{The Use of Large Language Models}

Large language models were occasionally used to refine the exposition and to
provide an additional check on selected calculations. The author takes full
responsibility for the content of the manuscript.

\section*{Code}

At the time of submission, the code associated with this work is not yet ready to be released in the form of a public GitHub repository or a Python package. However, the current implementation is available upon request to the author.

\acks{We thank Nicolas Chenavier, Christophe Biscio, and Nicola Rares Franco for the helpful discussions.}


\newpage

\vskip 0.2in
\bibliography{references}

\end{document}